\documentclass{article}

\usepackage[final,nonatbib]{neurips_glfrontiers_2023}

\usepackage{caption}
\usepackage{subcaption}
\usepackage{tikz}
\usepackage{amsmath}
\usepackage{amsmath}
\usepackage{amsthm}
\usepackage{amsfonts}
\usepackage{amssymb}
\usepackage{graphicx}
\usepackage{mathrsfs}
\usepackage{float}
\usepackage{threeparttable}
\usepackage{booktabs}
\usepackage{authblk}
\usepackage[bookmarksnumbered,
hidelinks,
colorlinks = true,
linkcolor = blue,
urlcolor  = blue,
citecolor = blue,
anchorcolor = blue]{hyperref}
\usepackage{tikz}
\usepackage{tikz-network}
\usepackage{pgfplots}
\usepackage{enumitem}
\usepackage{titlesec}
\usepackage{algorithm}
\usepackage{xspace}
\usepackage[noend]{algpseudocode}
\setlist[description]{leftmargin=*}
\titlespacing*{\section}{0pt}{0.3\baselineskip}{0.1\baselineskip}
\titlespacing*{\subsection}{0pt}{0.3\baselineskip}{0.1\baselineskip}

\newtheorem{lemma}{Lemma}
\theoremstyle{definition}
\newtheorem{definition}{Definition}
\theoremstyle{theorem}
\usetikzlibrary{decorations.pathreplacing,calc,shapes,positioning,arrows,arrows}
\usepackage{pgfplots}
\pgfplotsset{%
	,compat=1.12
	,every axis x label/.style={at={(current axis.right of origin)},anchor=north west}
	,every axis y label/.style={at={(current axis.above origin)},anchor=north east}
    ,every axis/.append style={legend image post style={xscale=0.5}}
}

\usetikzlibrary{datavisualization,shapes.geometric,arrows.meta,decorations.markings, fit}
\usetikzlibrary{datavisualization.formats.functions}
\definecolor{scarlet}{rgb}{1.0, 0.13, 0.0}
\definecolor{brightmaroon}{rgb}{0.76, 0.13, 0.28}
\definecolor{mediumturquoise}{rgb}{0.28, 0.82, 0.8}
\definecolor{fandango}{rgb}{0.71, 0.2, 0.54}
\definecolor{antiquewhite}{rgb}{0.98, 0.92, 0.84}
\definecolor{babyblue}{rgb}{0.54, 0.81, 0.94}
\definecolor{brilliantlavender}{rgb}{0.96, 0.73, 1.0}
\definecolor{bronze}{rgb}{0.8, 0.5, 0.2}
\definecolor{cornsilk}{rgb}{1.0, 0.97, 0.86}
\definecolor{lavenderpink}{rgb}{0.98, 0.68, 0.82}
\definecolor{sandybrown}{rgb}{0.96, 0.64, 0.38}
\definecolor{celadon}{rgb}{0.67, 0.88, 0.69}

\newcommand{\renyi}{R\'{e}nyi }
\newcommand{\pp}[1]{\mathbb{P}\left[#1\right]}
\newcommand{\expectation}[2]{\mathbb{E}_{#1}\left[#2\right]}
\newcommand{\dalpha}[2]{D_{\alpha}\left(#1 || #2\right)}

\newcommand{\opnorm}[1]{\|#1\|_{\text{op}}}
\newcommand{\norm}[1]{\left\|#1\right\|}

\newcommand{\se}{\mathcal{S}}
\newcommand{\mind}{D_{\text{min}}}

\newcommand{\ogb}{ogbn-products\xspace}
\newcommand{\vesper}{VESPER\xspace}

\newtheorem{theorem}{Theorem}[section]

\theoremstyle{definition}

\theoremstyle{remark}
\theoremstyle{plain}

\newtheorem{proposition}{Proposition}

\newcommand{\result}[2]{${#1}_{\pm#2}$}

\begin{document}
    \title{Privacy-preserving design of graph neural networks\\with applications to vertical federated learning}
    \author[$\dagger$]{Ruofan Wu}
    \author[$\dagger$]{Mingyang Zhang}
    \author[$\S$]{Lingjuan Lyu}
    \author[$\dagger$]{Xiaolong Xu}
    \author[$\dagger$]{\authorcr Xiuquan Hao}
    \author[$\dagger$]{Xinyi Fu}
    \author[$\dagger$]{Tengfei Liu}
    \author[$\dagger$]{Tianyi Zhang}
    \author[$\dagger$]{Weiqiang Wang}
    \affil[$\dagger$]{Ant Group}
    \affil[$\S$]{Sony AI}
    \affil[ ]{\footnotesize{\texttt{\{ruofan.wrf, zhangmingyang.zmy, yiyin.xxl, haoxiuquan.hxq, fxy122992, aaron.ltf, zty113091, weiqiang.wwq\}@antgroup.com}}}
    \affil[ ]{\footnotesize{\texttt{lingjuan.lv@sony.com}}}
    \maketitle
    
    \begin{abstract}
        The paradigm of vertical federated learning (VFL), where institutions collaboratively train machine learning models via combining each other's local feature or label information, has achieved great success in applications to financial risk management (FRM). The surging developments of graph representation learning (GRL) have opened up new opportunities for FRM applications under FL via efficiently utilizing the graph-structured data generated from underlying transaction networks. Meanwhile, transaction information is often considered highly sensitive. To prevent data leakage during training, it is critical to develop FL protocols with \emph{formal privacy guarantees}. In this paper, we present an end-to-end GRL framework in the VFL setting called \vesper, which is built upon a general privatization scheme termed \emph{perturbed message passing (PMP)} that allows the privatization of many popular graph neural architectures. 
        Based on PMP, we discuss the strengths and weaknesses of specific design choices of concrete graph neural architectures and provide solutions and improvements for both dense and sparse graphs.
        Extensive empirical evaluations over both public datasets and an industry dataset demonstrate that \vesper is capable of training high-performance GNN models over both sparse and dense graphs under reasonable privacy budgets.
    \end{abstract}
    
    \section{Introduction}\label{sec: intro}
    In recent years, there has been an increasing interest in adopting modern machine learning paradigms to the area of financial risk management (FRM) \cite{akib2020ML4fin}. The most crucial task in operational risk scenarios like fraud detection is identifying risky identities based on the behavioral data collected from the operating financial platform \cite{chen2018machine, kute2021aml}. For institutions like commercial banks and online payment platforms, the most important source of behavior information is the \emph{transaction records} between users,
    making \emph{transaction networks} (with users as nodes and transactions as edges) a direct and appropriate data model.
    To exploit the potential of transaction networks in a machine learning context, recent approaches \cite{liu2021intention, wu2022grande} have been exploring the adoption of graph representation learning (GRL) \cite{hamilton2020graph} as a principled way of incorporating structural information contained in transaction networks into the learning process. The family of graph neural networks in the message passing form \cite{pmlr-v70-gilmer17a, xu2018powerful} offers a powerful yet scalable solution to GRL, and has become the prevailing practice in industry-scale graph learning \cite{ying2018graph}.\par\noindent
    Despite its convincing performance, high-quality network data are not always available for financial institutions. F
    It is, therefore, of great interest for institutions to learn GRL models \emph{collaboratively} while being coherent to regulatory strictures at the same time. The technique of federated learning (FL) \cite{kairouz2019foundations, yang2019federated} provides a recipe for such scenarios, with participating institutions (hereafter abbreviated as \emph{parties}) exchanging intermediate results instead of raw data. Depending on the specific form of collaboration, FL protocols are generally divided into horizontal federated learning (HFL), where participants aggregate their locally trained models to obtain a strong global model, and vertical federated learning (VFL) where participants are able to align the identifiers of modeling entities and train a model that efficiently combines feature or label information that are distributed among different parties. VFL is particularly useful when training a (supervised) model is not possible based on information of a single party, i.e., each party holds only feature or label data, and has attracted significant attention in applications to FRM \cite{liu2022vertical}. While ordinary FL paradigms avoid the transmission of local raw data, they typically lack a formal guarantee of privacy \cite[Chapter 4]{kairouz2019foundations}. Moreover, recent studies have reported successful attacks targeting individual privacy against FL protocols \cite{zhu2019deep,yin2021see, jin2021cafe, duddu2020quantifying, dang2021revealing}.
    As transaction records are widely considered extremely sensitive personal information, it is thus critical to establish FL applications in FRM with rigorous privacy guarantees.\par\noindent
    Differential privacy (DP) \cite{dwork2014algorithmic} is the state-of-the-art approach to address information disclosure that injects algorithm-specific random noise to fuse the participation of any individual. The adoption of DP as the privacy model for FL is now under active development, with most of the applications appearing in HFL over independently identically distributed (i.i.d.) data through the lens of optimization \cite{kairouz2019foundations}. However, discussions on applying DP over VFL remain nascent \cite{chen2020vafl, zhou2020vertically, sajadmanesh2022gap}. The situation becomes even more complicated in VFL over graph-structured data, since the right notions of (differential) privacy on graphs are semantically different from that in the i.i.d. case \cite{nissim2007smooth, kasiviswanathan2013analyzing}. So far, as we have noticed, the only work that provides meaningful DP guarantee under VFL over graphs is the GAP model \cite{sajadmanesh2022gap}, which requires three stages of training. Meanwhile, a notable aspect of GRL is that the structure of the underlying graph, i.e., whether the graph is dense or sparse, might have a significant influence on the performance of the graph neural model especially when the aggregation process involves noisy perturbations. This phenomenon was overlooked in previous studies. \par\noindent
    In this paper, we discuss private FL over graph-structured data under the task of node classification in the vertical setup with edge DP \cite{nissim2007smooth} chosen as the privacy model. We first develop a general privatization scheme termed \emph{perturbed message passing (PMP)} that produces message-passing procedures over graphs that are guaranteed to satisfy edge DP constraints. 
    Next, we discuss the influence of the underlying graph's degree profiles on the utility of specific design choices of PMP, using two representative graph aggregation schemes, namely GIN \cite{xu2018powerful} and GCN \cite{kipf2016semi}, and develop further improvements of PMP that better handles sparse graphs under the GCN aggregation scheme.
    Finally, we integrate the developments of PMP and its variants into a VFL architecture called \vesper based on the SplitNN framework \cite{gupta2018distributed}, and conducted extensive empirical evaluations over both public and industrial datasets covering dense and sparse graphs. We summarize our contributions as follows:
    \begin{itemize}[leftmargin=*]
        \item We propose PMP, a general framework for designing differentially private message-passing procedures. PMP enables the privatization of many popular graph neural network architectures. The privacy guarantee of PMP is formally analyzed with new privacy amplification results under uniform neighborhood sampling. 
        \item We discuss two representative design choices under the PMP framework, GIN and GCN, and discover the fact that the utility of the privatized GNN model may be affected by the \emph{degree profile} of the input graph. To better accommodate varying graph structures, we develop the truncated message passing framework under the base model of GCN through properly tuning the hyperparameter that reduces noise scale at the cost of learning less structural information, which is beneficial when the input graph is \emph{sparse}.
        \item We derive an end-to-end VFL learning framework operating over graph-structured data called \vesper, which is efficient in computation and communication. A thorough experimental study demonstrates that \vesper achieves better privacy-utility trade-off over previously proposed models and is capable of training high-performance GNN models over both sparse and dense graphs under reasonable privacy budgets. 
    \end{itemize}
    \begin{figure*}
        \centering
        \resizebox{0.65\textwidth}{!}{
        \tikzset{%
    VertArrow/.style = {%
        double equal sign distance,
        scarlet,
        -{Stealth[brightmaroon,scale=1.3,inset=2pt, angle=30:10pt]},
        semithick
    },
    HoriArrow/.style = {%
        -{Stealth[scale=1.5,inset=1pt, angle=20:7pt]},
        semithick
    },
    DArrow/.style = {%
        #1, 
        dashed,
        -{Stealth[scale=1.5,inset=1pt, angle=30:4pt]},
        semithick
    },
    Box/.style 2 args = {
        rectangle,
        draw = black,
        very thick,
        fill = #1!#2, 
        align=center,
        minimum width=1cm,
        minimum height = 0.5cm
    },	
    Frame/.style 2 args = {
        rectangle,
        draw = 	#1,
        fill = #1!#2,
        minimum width=18cm,
        minimum height = 8cm
    },
    MiniFrame/.style 2 args = {
        rectangle,
        draw = 	#1,
        fill = #1!#2,
        minimum width=6cm,
        minimum height = 3cm
    },
    WideBar/.style 2 args = {
        rectangle,
        draw = 	#1,
        fill = #1!#2,
        minimum width=15cm,
        minimum height = 1cm
    },
    MiniBar/.style 2 args = {
        rectangle,
        draw = 	#1,
        fill = #1!#2,
        minimum width=3cm,
        minimum height = 0.5cm
    },
}
\begin{tikzpicture}
    \begin{pgfonlayer}{background}
        \node [draw=blue!50, fill=blue!10, fit={(-7, 3) (0, 3) (0, 0) (-7, 0)}, inner sep=5.75pt] (partyA) {};
    \end{pgfonlayer}
    \begin{pgfonlayer}{background}
        \node [draw=fandango!50, fill=fandango!10, fit={(7, 3) (2.5, 3) (2.5, 0) (7, 0)}, inner sep=5.75pt] (partyB) {};
    \end{pgfonlayer}
    \node[] at ($(partyA.north) - (0, 4mm)$){\color{blue}\texttt{Party A}};
    \node[] at ($(partyB.north) - (0, 4mm)$){\color{fandango}\texttt{Party B}};
    \Vertex[color=antiquewhite, x=-5, y=0.5, size=.1]{n0},
    \Vertex[color=antiquewhite, x=-5.5, y=0.8, size=.1]{n1},
    \Vertex[color=antiquewhite, x=-4.5, y=0.8, size=.1]{n2},
    \Vertex[color=antiquewhite, x=-5.25, y=1.1, size=.1]{n3},
    \Vertex[color=antiquewhite, x=-4.75, y=1.1, size=.1]{n4},
    \Edge[](n1)(n0),
    \Edge[](n2)(n0),
    \Edge[](n3)(n0),
    \Edge[](n3)(n2),
    \Edge[](n4)(n2),
    \node[Box = {antiquewhite}{50}] at (-6, 2)(nodefeature){\small\texttt{node feature}};
    \node[Box = {white}{10}] at (-3, 2)(encoder){\small\texttt{encoder}};
    \node[Box = {white}{10}] at (-0.5, 2)(pre){\small\texttt{pre}};
    \node[Box = {white}{10}] at (4, 2)(decoder){\small\texttt{decoder}};
    \node[Box = {white}{10}] at (6, 2)(loss){\small\texttt{loss}};
    \node[Box = {antiquewhite}{50}] at (6, 1)(label){\small\texttt{label}};
    \node[Box = {white}{10}] at (-3, 0.8)(sampler){\small\texttt{sampler}};
    \node[] at ($(partyA.east) + (1, 0.8)$){\texttt{embeddings}};
    \node[] at ($(partyA.east) + (1, 0.2)$){\texttt{gradients}};
    \node[] at ($(n0.south) - (0, 0.2)$){\texttt{graph} $G$};
    \draw[HoriArrow](nodefeature.east)--(encoder.west);
    \draw[HoriArrow]($(encoder.east) + (0, 1mm)$)--($(pre.west) + (0, 1mm)$);
    \draw[DArrow={black}]($(pre.west) - (0, 1mm)$)--($(encoder.east) - (0, 1mm)$);
    \draw[HoriArrow]($(decoder.east) + (0, 1mm)$)--($(loss.west) + (0, 1mm)$);
    \draw[DArrow={black}]($(loss.west) - (0, 1mm)$)--($(decoder.east) - (0, 1mm)$);
    \draw[HoriArrow](label.north)--(loss.south);
    \draw[HoriArrow]($(pre.east) + (0, 1mm)$)--($(decoder.west) + (0, 1mm)$);
    \draw[DArrow={black}]($(decoder.west) - (0, 1mm)$)--($(pre.east) - (0, 1mm)$);
    \draw[HoriArrow](-4.25, 0.8)--(sampler.west);
    \draw[-](sampler.east)--(-0.5, 0.8);
    \draw[HoriArrow](-0.5, 0.8)--(pre.south);
\end{tikzpicture}
        }
        \caption{A concise pictorial description of the \vesper framework. We use solid arrows to depict the dataflow of forward computations and use dashed arrows to depict the dataflow of backward computations. }
        \label{fig: vesper}
    \end{figure*}
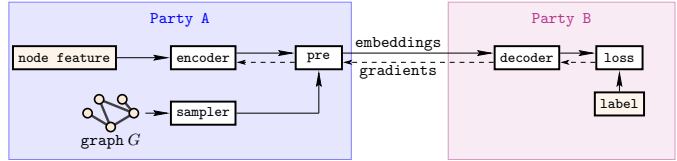
    \section{Methodology} 
    \subsection{Priliminaries}\label{sec: setup}
    We focus on the node classification task over a static, undirected graph $ G = (V, E) $ with node size $N = |V|$, node feature $X = \{x_v\}_{v \in V}$ and node labels $Y = \{y_v\}_{v \in V_T}$ where $V_T \subseteq V$ is the set of training nodes with $N_T = \left|V_T\right|$. Throughout this article, we will assume the graph of interest to be degree bounded, i.e., 
    \begin{align}
        \max_{G} \max_{v \in G} d_v \le D
    \end{align}
    for some $D > 1$. In this paper, we will be interested in the setup where the graph data $G$ and label information are distributed over two distinct parties. Specifically, suppose there are two parties, A (Alice) and B (Bob), where A holds the graph data $G$ as well as the node feature $X$ and B holds the label collection $Y$, both indexed by node identifiers that are known to both sides (i.e., $V_T$ is known to both party A and party B). We consider a representative federated learning paradigm that A and B collaboratively train a graph representation learning model via utilizing the panoply of graph neural networks \cite{pmlr-v70-gilmer17a}, which could be regarded as a special case of vertical federated learning (VFL) \cite{yang2019federated}. Under VFL protocols, party A and party B iteratively exchange intermediate outputs depending on the specific training algorithm chosen. A main concern in VFL \cite[Chapter 4]{kairouz2019foundations} is, therefore, whether the exchanging process satisfies formal \emph{privacy} guarantees. Before elaborating on privacy protection issues, we first state the threat model in our context.\par\noindent
    \textbf{Threat model} We adopt the following threat model in this paper: In the training stage, label party B is curious about the adjacency information (i.e., the existence of some edges) in the data party A. The data party A is assumed to be benign, with both parties strictly obeying the chosen VFL protocol. 
    \footnote{The assumption of a harmless party A might be relaxed to a curious onlooker that tries to infer party B's label information. We discuss related extensions in section \ref{sec: extensions}.}
    In other words, the goal of privacy protection is to prevent the \emph{semi-honest} adversary (party B) from inferring the edge membership that is only known to party A.\par\noindent
    Differential privacy \cite{dwork2014algorithmic} is now the \emph{de facto} choice of privacy protection paradigm against membership inference adversaries. As an appropriate solution concept in the current setup, we introduce the edge-level differential privacy model (hereafter abbreviated as Edge DP). 
    \begin{definition}[Edge-level differential privacy(Edge DP)]\label{def: edgeDP}
        For a (randomized) graph-input mechanism $\mathcal{M}$ that maps graphs to some output space $\mathcal{S}$ and two non-negative numbers $\epsilon$ and $\delta$, the mechanism is $(\epsilon, \delta)$-Edge DP if for any subset $S$ (or more rigorously defined as Borel measurable sets) of the output space, the following holds uniformly for any two possible adjacent graphs $(G, G^\prime)$:
        \begin{align}
            \mathbb{P}[\mathcal{M}(G) \in S] \le e^\epsilon \mathbb{P}[\mathcal{M}(G^\prime) \in S] + \delta,
        \end{align}
        where we define two graphs $G$ and $G^\prime$ as being adjacent if $G$ could be edited into $G^\prime$ via adding or removing a single edge.
    \end{definition}
    Regarding the capability of the adversary adopted in this paper, a VFL protocol satisfying Edge DP with a reasonable $\epsilon$ level implies that based on all the exchanged intermediate outputs between party A and party B, any membership inference algorithm may not be able to make any sophisticated guess about the existence of some specific edge in a probabilistic sense, thereby offering strong privacy protection. Most contemporary differentially private machine learning algorithms involve sequentially applying DP procedures to intermediate learning steps \cite{abadi2016deep}, with the privacy level of the entire training procedure obtained via composition theorems \cite{dwork2014algorithmic, kairouz2015composition}. In this paper, we choose the composition framework of analytical moment accountant (AMA) \cite{wang2019subsampled} that exploits the idea of \renyi DP \cite{mironov2017renyi}, which we introduce below in our graph learning context:
    \begin{definition}[Edge-level \renyi-differential privacy(Edge RDP)]\label{def: edgeRDP}
        Sharing notations with definition \ref{def: edgeDP}, the mechanism $\mathcal{M}$ is $(\alpha, \epsilon(\alpha))$-\renyi differentially private with some $\alpha > 1$ and $\epsilon(\alpha) \ge 0$, if for any two possible adjacent graphs $(G, G^\prime)$, the $\alpha$-\renyi divergence of the induced probability distribution of random variables $\mathcal{M}(G)$ and $\mathcal{M}(G^\prime)$ is bounded by $\epsilon(\alpha)$:
        \begin{align}\label{eqn: RDP}
            \dalpha{\mathcal{M}(G)}{\mathcal{M}(G^\prime)} \le \epsilon(\alpha),
        \end{align}
        with the definition of $\alpha$-\renyi divergence $\dalpha{\cdot}{\cdot}$ presented in appendix \ref{sec: rdp}.
    \end{definition}
    To develop privacy-preserving learning algorithms under the AMA framework, we first design mechanisms that satisfy RDP guarantee in each step, then use standard composition results of RDP \cite{mironov2017renyi} to obtain the privacy level of the learning procedure. Finally, we apply the conversion rule in \cite{balle2020hypothesis} to convert it back to $(\epsilon, \delta)$-DP for reporting. \\
    \textbf{Message passing GNNs with stochastic training} The backbone of our privacy-preserving training framework is the graph neural network model in the message passing form \cite{pmlr-v70-gilmer17a}. We define the GNN of interest to be a map from the space of graphs to a node embedding matrix with embedding dimension $d$: $f: \mathcal{G} \mapsto \mathbb{R}^{N \times d}$, or $H := \{h_v\}_{v \in V} = f(G)$. For an $L$-layer GNN, let $ h_v^{(0)} = g(x_v) $ be the input encoding of node $v$, which could be either $x_v$ or some encoding based on $x_v$. We assume the following recursive update rule for $1 \le l \le L$ and $v \in V$:
    \begin{align}\label{eqn: gnn}
        h^{(l)}_v = \sigma\left(\widetilde{h}^{(l)}_v\right), \quad \widetilde{h}^{(l)}_v = \omega_v W^{(l)}_1 h_v^{(l-1)} + \sum_{u \in N(v)}\beta_{uv}W^{(l)}_2 h^{(l-1)}_u,
    \end{align}
    with $\boldsymbol\omega := \{\omega_v\}_{v \in V} \in \mathbb{R}^N$ and $\boldsymbol\beta := \{\beta_{uv}\}_{u, v \in V\times V} \in \mathbb{R}^{N \times N}$ be model-dependent coefficients, $\sigma$ a parameter-free nonlinear function, and  $\mathbf{W} = (W_1^{(1)}, \ldots, W_1^{(L)}, W_2^{(1)}, \ldots, W_2^{(L)})$ be the collection of learnable parameters. For any matrix $W$, we denote $\opnorm{W}$ as the operator norm of the matrix (i.e., its largest singular value). In this paper, we assess two representative instantiations of the protocol \eqref{eqn: gnn} which are the GIN model \cite{xu2018powerful} with with $\omega_v \equiv \beta_{uv} \equiv 1, \forall u, v \in V$ and the GCN model \cite{kipf2016semi} with $\omega_v = \frac{1}{d_v + 1} $ and $\beta_{uv} = \frac{1}{\sqrt{d_u + 1}\sqrt{d_v + 1}}$. For simplicity we additionally let the nonlinearity be the ReLU function and set $W_1^{(l)} = W_2^{(l)} = W^{(l)}, 1\le l \le L$. \par\noindent
    Applying message passing updates \eqref{eqn: gnn} may become computationally prohibitive for large input graphs, which are frequently encountered in industrial scenarios. To enable scalable GRL, the prevailing practice is to use graph sampling methods \cite{hamilton2017inductive} and adopt \textbf{stochastic training of graph neural networks}. In this paper, we investigate the simple and effective sampling scheme of uniform neighborhood sampling \cite{hamilton2017inductive, daigavane2021node}, with the maximum number of neighbors sampled in each layer to be the maximum degree $D$. Asides from their computational benefits, it has been observed \cite{abadi2016deep, mcmahan2017learning} that stochastic training with a low sampling ratio over large datasets is crucial to training high-utility differentially private machine learning models with reasonably small privacy budgets, which has also been recently verified in the case of differentially private graph learning \cite{daigavane2021node, sajadmanesh2022gap}. 
    \subsection{Perturbed message passing}\label{sec: priv_gnn}
    A notable fact about the message-passing protocol \eqref{eqn: gnn} is that it uses the aggregation strategy of \emph{weighted summation}, thereby allowing standard additive perturbation mechanisms like the Laplace mechanism or Gaussian mechanism that are prevailing in the design of differentially private algorithms \cite{dwork2014algorithmic}. Motivated by this fact, we propose a straightforward solution to privatize message-passing GNNs in a \emph{layer-wise} fashion named \emph{perturbed message passing (PMP)}, which adds layer-wide Gaussian noise with an additional normalization step that controls sensitivity. We present the pseudo-code of PMP with neighborhood sampling in algorithm \ref{alg: priv_gnn_sample}.
    \begin{algorithm}
        \caption{PMP with neighborhood sampling}
        \label{alg: priv_gnn_sample}
        \begin{algorithmic}[1]
            \Require Graph $G = (V, E)$, input encodings $ \{h^{(0)}_v\}_{v \in V}$, number of message passing rounds $L$, GNN spec $(\boldsymbol{
            \omega}, \boldsymbol{\beta}, \sigma)$, noise scale $\theta$, GNN parameter $\mathbf{W}$, batch size $B$, maximum degree $D$.
            \State Sample a random batch of root nodes $v_1, \ldots, v_B$.
            \State Apply an $L$-layer neighborhood sampler with each layer sampling at most $D$ nodes with roots $v_1, \ldots, v_B$, obtaining a batch of $B$ subgraphs $(G^{(L)}_{v_1}, \ldots, G^{(L)}_{v_B})$.
            \State Combine $(G^{(L)}_{v_1}, \ldots, G^{(L)}_{v_B})$ into a subgraph $G^{(L)}_B$. Additionally, overload the notation $N(v)$ for the neighborhood of node $v$ with respect to $G^{(L)}_{v_B}$.
            \State Set $h^{(0)}_v = \frac{h^{(0)}_v}{\left\|h^{(0)}_v\right\|_2}$ for $\forall v \in G^{(L)}_{v_B})$
            \For{$l \in \{1, \ldots, L\}$}
            \For{$v \in G^{(L)}_{v_B}$}
            \State Compute the linear update
            $\widetilde{h}^{(l)}_v = \omega_v W^{(l)}_1 h_v^{(l-1)} + \sum_{u \in N(v)}\beta_{uv}W^{(l)}_2 h^{(l-1)}_u$.
            \State Do additive perturbation, $h^{(l)}_v = \sigma(\widetilde{h}^{(l)}_v + N(0, \theta^2))$ 
            \State Normalize $h^{(l}_v = \frac{h^{(l)}_v}{\left\|h^{(l}_v\right\|_2}$
            \EndFor
            \EndFor
            \Return A list of all layers' embedding matrices $\mathbf{H}_L = (H^{(1)}, \ldots, H^{(L)})$, with $H^{(l)} = \{h^{(l)}_v\}_{v \in G^{(L)}_{v_B}}, 1 \le l \le L$.
        \end{algorithmic}
    \end{algorithm}
    Next we discuss the privacy guarantee of algorithm \ref{alg: priv_gnn_sample}. To state our main result, we first define the right notion of sensitivity in our context:
    \begin{definition}[Edge sensitivity]\label{def: edge_sensitivity}
        Denote $G^\prime$ as the adjacent graph via removing the edge $(u^*, v^*)$ from $G$, and let $\widetilde{h}_v$ and $\widetilde{h}^\prime_v$ be the outputs of node $v$ generated via some $1$-layer GNN protocol under graph $G$ and $G^\prime$ without nonlinearity, then we define the ($\ell_2$-) \emph{edge sensitivity} as:
        \begin{align}
            \se = \max_{G, G^\prime}\sqrt{\sum_{v \in V}\|\widetilde{h}_v - \widetilde{h}^\prime_v\|_2^2}.
        \end{align}
    \end{definition}
    The following theorem quantifies the privacy guarantee of algorithm \ref{alg: priv_gnn_sample}:
    \begin{theorem}[RDP guarantee]\label{thm: sampled}
        Let $\mathbf{H}_L$ be the released outputs with input a minitach of $B$ subgraphs produced by uniform neighborhood sampling for $L$ layers with a maximum number of $D$ neighbors sampled in each layer. Define
        $\epsilon(\alpha) := \frac{\alpha \sum_{l=1}^L\mathcal{S}_l^2}{2\theta^2}$,
        then $\mathbf{H}_L$ is $(\alpha, \epsilon_{\gamma}(\alpha)$-RDP for any $\alpha > 1$, where $\gamma = 1 - \frac{\binom{N_T - \frac{2(D^L - 1)}{D - 1}}{B}}{\binom{N_T}{B}}$
        and
        \begin{align}\label{eqn: sampling_rdp}
            \begin{aligned}
                \epsilon_{\gamma}(\alpha) \le \frac{1}{\alpha - 1}\log & \left(1 + \gamma^2{\alpha \choose 2}\min\left(4\left(e^{\epsilon(2)} - 1\right), \epsilon(2)\min\left(2, \left(e^{\epsilon(\infty) - 1}\right)^2\right)\right)\right. \\
                & + \left.\sum_{j=3}^\infty \gamma^j {\alpha \choose j} e^{(j - 1)\epsilon(j)}\min\left(2, \left(e^{\epsilon(\infty) - 1}\right)^j\right)\right)
            \end{aligned}
        \end{align}
    \end{theorem}
    
    Theorem \ref{thm: sampled} provides a principled way of analyzing the privacy of privatized GNN models using algorithm \ref{alg: priv_gnn_sample}, which boils down to computing the edge sensitivity of the underlying message passing protocol. However, sensitivity computations are usually conducted in a \emph{worst-case} manner, resulting in unnecessarily large noise levels and significant utility loss. Therefore, it is valuable to explore the utility of concrete PMP models and their relationships with the underlying input graph. To begin our expositions, we analyze the GIN model in the following section. 
    \subsection{Analysis of GIN and the challenge of sparse graphs}
    We start with the following proposition:
    \begin{proposition}\label{prop: gin_es}
        Under the GIN model, the edge sensitivity is bounded from above by $\se^{\text{GIN}}_l \le \sqrt{2}\opnorm{W^{(l)}}$ for each $1\le l \le L$.
    \end{proposition}
    \textbf{Advantage of layer-wise perturbations} According to proposition \ref{prop: gin_es}, the edge sensitivity of GIN is independent of the input graph's maximum degree upper bound $D$, which is essentially a direct consequence of the fact that for a $1$ layer message passing procedure, adding or removing one edge would affect up to two nodes' output embeddings. As a consequence, the privacy cost scales linearly with the number of message-passing layers in the \renyi DP framework, thereby offering a better privacy-utility trade-off than algorithms that do the do the perturbation only in the final layer \cite{zhou2020vertically}, whose privacy cost may scale exponentially with $D$.\par\noindent
    \textbf{Effectiveness and challenges of summation pooling} It has been observed in previous works \cite{sajadmanesh2022gap} that aggregation perturbation with sum pooling works well on graphs with a large average degree. Intuitively, this phenomenon could be understood as keeping a high "signal-to-noise ratio (SNR)" during the aggregation process: For nodes with large degrees, the noise scale becomes relatively small with respect to the summation of incoming messages. Therefore if high-degree nodes are prevalent in the underlying graph, the utility loss during aggregation is reasonably controlled for most nodes. However, realistic graph data might not have large average degrees. For example, transaction networks in FRM scenarios are usually sparse, including many nodes with degrees smaller than $5$ or even being singular (i.e., of degree $0$). 
    Consequently, the SNR of sparse networks makes it harder for summation pooling to maintain decent utility, which will be further verified in section \ref{sec: experiments}.
    \subsection{Improvements of PMP in the GCN model}\label{sec: pmp_gcn}
    As discussed in the previous section, the degree profile of the input graph may affect the utility of PMP-privatized GNNs when the underlying aggregation follows the summation pooling scheme. It is therefore of interest to explore aggregation schemes that are more appropriate when the input graph is sparse. On first thought, we may expect aggregation schemes like mean pooling or GCN pooling to have smaller sensitivities. However, such sensitivity reduction does NOT hold in a worst-case analysis: Just think of nodes with degree $1$, then it is not hard to check that mean pooling or GCN pooling behaves similarly to summation pooling. The primary issue with worst-case analysis is that the resulting sensitivity is determined by extremely \emph{low-degree} nodes. Inspired by this phenomenon, we seek improvements by first deriving lower sensitivity with an extra requirement on a \emph{degree lower bound}, and then relax the requirement via introducing a modified protocol. We start with the following observation:
    
    \begin{proposition}\label{prop: gcn_es}
        Assume all the possible input graphs have a minimum degree larger or equal to $\mind$, or
        \begin{align}\label{eqn: min_degree}
            \min_{G} \min_{v \in G} d_v \ge \mind > 1.
        \end{align}
        Then for the GCN model, the edge sensitivity of the $l$-th layer $\se_l^{\text{GCN}}$ is bounded from above by a function $\eta_l(\mind)$, defined as:
        \begin{align}
            \begin{aligned}
                \eta_l(\mind)=\sqrt{2}\left(\dfrac{1 - 1/\mind}{2\mind} + \dfrac{1}{\mind (\mind + 1)} + \dfrac{1}{\mind + 1}\right) \opnorm{W^{(l)}}.
            \end{aligned}
        \end{align}
    \end{proposition}
    Proposition \ref{prop: gcn_es} implies that the edge sensitivity of the GCN model shrinks significantly if the underlying graph has a reasonably large minimum degree, which will result in a significantly reduced noise scale that improves utility. However, the minimum degree assumption \eqref{eqn: min_degree} is impractical since most of the realistic graph data have a large number of nodes with small degrees. To circumvent the impracticality of assumption \eqref{eqn: min_degree} while still being able to reduce the noise scale in the GCN model, we propose a modification to the basic message passing algorithm \ref{alg: priv_gnn_sample} called \emph{truncated message passing}.
    The idea of truncated message passing is to block all the incoming messages unless the receiver node's neighborhood is large than or equal to $\mind$, which is treated as a hyperparameter. For nodes with degrees lower than $\mind$, the output embedding is instead produced by an MLP with perturbation that does not involve any edge information. A detailed version is provided in algorithm \ref{alg: truncated_mp} in appendix \ref{sec: algo}. Consequently, it is straightforward to show that the differential privacy guarantee of the resulting algorithm operating on any graph matches the privacy level of perturbed GCN (produced by algorithm \ref{alg: priv_gnn_sample}) operating only on graphs with minimum degree assumption.\\
    \textbf{How to choose $\mind$?} To maintain the same privacy level under the truncated message passing algorithm, one may reduce the noise scale $\theta$ at the cost of raising the minimum degree hyperparameter $\mind$. On the one hand, reducing the noise scale significantly improves the utility of the message-passing procedure. On the other hand, raising $\mind$ might prevent a non-ignorable proportion of nodes from learning structural information. Therefore, properly adjusting $\mind$ may help achieve a better privacy-utility trade-off in the GCN model. In practice, one may choose $\mind$ based on prior knowledge about the degree distribution of the underlying graph or via inspecting a private release of its degree distribution, which could be done efficiently using the Laplace mechanism \cite{dwork2014algorithmic}. \par\noindent 
    \subsection{\vesper: an end-to-end learning framework}\label{sec: vesper}
    In previous sections, we have established the PMP framework for differentially private graph representation learning. Now under the vertically federated learning setup described in section \ref{sec: setup}, we propose an end-to-end architecture inspired by the SplitNN paradigm \cite{gupta2018distributed} based on the PMP framework, named \underline{\textbf{VE}}rtically private \underline{\textbf{S}}plit GNN with \underline{\textbf{PER}}turbed message passing({\textbf{\vesper}}). The \vesper architecture contains three main components: Encoder, Private representation extractor (PRE), and Decoder.\par\noindent
    \textbf{Encoder} The encoder module maps input node features into a $d$-dimensional representation vector, with an ad-hoc choice being an MLP. Note that for node features with additional structural patterns (i.e., sequence data), we may use a more tailored encoder architecture as long as it does not involve edge information. The encoder model is physically stored in party A.\\
    \textbf{Private representation extractor} The PRE module takes its input the node embeddings produced by the encoder and a batch of $B$ subgraphs produced by a neighborhood sampler. The output representation of PRE is computed using some specific type of PMP mechanism such as PMP-GIN or PMP-GCN. The PRE module is physically stored in party A. The output of PRE is a tensor of shape $B \times d \times L$, with $d$ and $L$ being the dimension of graph representation and the number of message passing layers respectively. The outputs will be transmitted from party A to party B.\\
    \textbf{Decoder} The decoder module is physically stored in party B, which decodes the received node embeddings produced by PRE into the final prediction of \vesper with its structure depending on the downstream tasks (i.e., classification, regression, ranking, etc.). We test two types of decoder architectures in our implementation of \vesper. The first one proceeds via concatenating the node embeddings of all layers followed by an MLP, which we call the CONCAT decoder. The second one treats the node embeddings as a sequence of $L$ node embeddings and uses a GRU network to combine them, similar to the idea used in GNN architectures like GaAN \cite{li2015gated} and GeniePath \cite{liu2019geniepath} which we term the GRU decoder. \\
    The VFL training protocol closely resembles the SplitNN protocol \cite{gupta2018distributed}, where in each step, forward computation results (i.e., the outputs of the PRE module) are transmitted from party A to party B. After party B finishes the forward computation using the decoded outputs and label information, party B first update its local decoder module via back-propagation, and then sends (partial) gradients that are intermediate results of the backward computation to party A for updating party A's local parameters (i.e., parameters of the encoder module and PRE module). A pictorial illustration of the \vesper architecture is presented in figure \ref{fig: vesper}. We will discuss some practical issues in implementing \vesper in appendix \ref{sec: impl}. \par\noindent

    \section{Experiments}\label{sec: experiments}
    In this section we present empirical evaluations of the \vesper framework via investigate its privacy-utility trade-off and resistance to empirical membership inference attacks. Due to limited space, a complete report will be postponed to appendix \ref{sec: exp_full}.
    \subsection{Datasets}
    We use three large-scale graph datasets, with their summary statistics listed in table \ref{tab: dataset_summary}. Specifically, we use two public datasets \ogb and Reddit, with their detailed descriptions postponed to appendix \ref{sec: dataset}. We additionally used an industrial dataset called \textbf{the Finance dataset} which is generated from transaction records collected from one of the world's leading online payment systems. The underlying graph is generated by treating users as nodes, and two nodes are connected if at least one transaction occurred between corresponding users within a predefined time period. The business goal is to identify risky users which is cast into an algorithmic problem of node classification with a binary label. The node features are obtained via statistical summaries of corresponding users' behavior on the platform during a specific time period. The training and testing datasets are constructed under two distinct time windows with no overlap.\par\noindent
    \textbf{A differentially private analysis of degree profiles} While all three datasets are large in scale (i.e., with the number of nodes exceeding $100,000$), they differ significantly in their degree distributions. For a better illustration, we conduct a differentially private analysis of degree distribution (with $(0.1, 0)$-differential privacy) detailed in appendix \ref{sec: histogram}. According to the analysis, we find that both the \ogb and the Reddit contain a large portion of high-degree nodes (as illustrated by the spiking bar at the $\ge 50$ category), while the Finance dataset exhibits a concentration on the lower-degree nodes. As discussed in section \ref{sec: priv_gnn}, it is expected that the Finance dataset is more challenging for (private) message passing under sum pooling. 
    \subsection{Baselines}
    We compare the proposed \vesper framework with three types of baselines, with each one being able to implement in the vertically federated setting. \textbf{MLP without edge information} we use MLP over node features directly is the most trivial solution to the learning task as it totally ignores edge information. \textbf{Non-private GNN counterparts} we compare with ordinary GCN and GIN models without privacy guarantees, or equivalently set the $\epsilon$ parameter in the \vesper framework to be infinity. \textbf{GNN models with privacy guarantees} we consider two alternative approaches to private GRL, namely the VFGNN model \cite{zhou2020vertically} and the GAP model \cite{sajadmanesh2022gap}. We found the privacy analysis in the corresponding papers to be somewhat incoherent with the privacy model in our paper and we conducted new analysis of their privacy properties, detailed in appendix \ref{sec: baselines}.
    
    \subsection{Experimental setup}\label{sec: exp_setup}
    Due to limited space, we postpone the description of our training configurations to appendix \ref{sec: exp_setup_full} and elaborate more on the \textbf{privacy configurations}: All the privacy reports are based on the $(\epsilon, \delta)$-differential privacy model, with $\delta$ being the reciprocal of the number of edges. 
    To adequately inspect the privacy-utility trade-off, we aim to evaluate all the models with differential privacy guarantees under the total privacy costs (privacy budgets) $\epsilon \in \{1, 2, 4, 8, 16, 32\}$, with the privacy costs accounted during the entire training period. 
    We treat the setting where $\epsilon \in \{1, 2\}$ as of \emph{high privacy}, $\epsilon \in \{4, 8\}$ as of \emph{moderate privacy}, and the rest as of \emph{low privacy}.
    For \vesper and VFGNN, we add spectral normalization to each GNN layer. For the privacy accountant, we base our implementation upon AMA implementation available in the \href{https://github.com/google/differential-privacy/tree/main/python}{\texttt{dp-accounting}} library and use an adjusted sampling probability according to theorem \ref{thm: sampled}. For each required privacy level, we compute the minimum scale of Gaussian noise via conducting a binary search over the adjusted AMA, with associating spectral norms of weight matrices fixed at one in all layers. \par\noindent
    \textbf{Evaluation metrics} We adopt classification accuracy (ACC) as the evaluation metric for the \ogb and Reddit datasets, and ROC-AUC score (AUC) as the evaluation metric for the Finance dataset.
    \begin{figure}
        \centering
        \begin{subfigure}{0.16\linewidth}
            \resizebox{\linewidth}{\linewidth}{
\begin{tikzpicture}

\definecolor{darkgray176}{RGB}{176,176,176}
\definecolor{gray}{RGB}{128,128,128}
\definecolor{green}{RGB}{0,128,0}
\definecolor{lightgray204}{RGB}{204,204,204}

\begin{axis}[
legend cell align={left},
legend style={
  fill opacity=0.8,
  draw opacity=1,
  text opacity=1,
  at={(0.52,0.45)},
  anchor=north west,
  draw=lightgray204
},
scale only axis,
width=2.25cm,
height=2cm,
tick align=outside,
tick pos=left,
title={\ogb},
tick label style={font=\tiny},
title style={yshift=-4pt},
xticklabel style={yshift=2pt},
yticklabel style={xshift=2pt},
legend style={nodes={scale=0.5, transform shape}},
every tick/.style={
black,
semithick,
},
x label style={at={(axis description cs:0.5,-0.1)},anchor=north,font=\tiny},
y label style={at={(axis description cs:-0.1,.5)},rotate=90,anchor=south,font=\tiny},
x grid style={darkgray176},
xlabel={\(\displaystyle \epsilon\)},
xmin=-0.3, xmax=6.3,
xtick={0,1,2,3,4,5,6},
xticklabels={1,2,4,8,16,32,\(\displaystyle \infty\)},
ylabel={ACC},
ymin=0.35, ymax=0.8,
]
\path [draw=gray, fill=gray, opacity=0.2]
(axis cs:0,0.3935)
--(axis cs:0,0.3307)
--(axis cs:1,0.5376)
--(axis cs:2,0.7057)
--(axis cs:3,0.7401)
--(axis cs:4,0.7487)
--(axis cs:5,0.749)
--(axis cs:6,0.759)
--(axis cs:6,0.7778)
--(axis cs:6,0.7778)
--(axis cs:5,0.7674)
--(axis cs:4,0.7585)
--(axis cs:3,0.7491)
--(axis cs:2,0.7197)
--(axis cs:1,0.5608)
--(axis cs:0,0.3935)
--cycle;

\path [draw=blue, fill=blue, opacity=0.2]
(axis cs:0,0.4825)
--(axis cs:0,0.4449)
--(axis cs:1,0.6004)
--(axis cs:2,0.672)
--(axis cs:3,0.6801)
--(axis cs:4,0.6958)
--(axis cs:5,0.7029)
--(axis cs:6,0.779)
--(axis cs:6,0.7862)
--(axis cs:6,0.7862)
--(axis cs:5,0.7095)
--(axis cs:4,0.7068)
--(axis cs:3,0.6935)
--(axis cs:2,0.68)
--(axis cs:1,0.6276)
--(axis cs:0,0.4825)
--cycle;

\path [draw=red, fill=red, opacity=0.2]
(axis cs:0,0.6098) -- (axis cs:0,0.6114) -- (axis cs:6,0.6114) -- (axis cs:6,0.6098);

\addplot [semithick, black, mark=*, mark size=1, mark options={solid}]
table {%
0 0.3621
1 0.5492
2 0.7127
3 0.7446
4 0.7536
5 0.7582
6 0.7684
};
\addlegendentry{GIN}
\addplot [semithick, blue, dashed, mark=square*, mark size=1, mark options={solid}]
table {%
0 0.4637
1 0.614
2 0.676
3 0.6868
4 0.7013
5 0.7062
6 0.7826
};
\addlegendentry{GCN}
\addplot [thick, red, mark=none]
table {%
0 0.6106
6 0.6106
};
\addlegendentry{MLP}
\end{axis}
\end{tikzpicture}}
            \caption{}
        \end{subfigure}
        \begin{subfigure}{0.16\linewidth}
            \resizebox{\linewidth}{\linewidth}{
\begin{tikzpicture}

\definecolor{darkgray176}{RGB}{176,176,176}
\definecolor{gray}{RGB}{128,128,128}
\definecolor{green}{RGB}{0,128,0}
\definecolor{lightgray204}{RGB}{204,204,204}

\begin{axis}[
legend cell align={left},
legend style={
  fill opacity=0.8,
  draw opacity=1,
  text opacity=1,
  at={(0.52,0.45)},
  anchor=north west,
  draw=lightgray204
},
scale only axis,
width=2.25cm,
height=2cm,
tick align=outside,
tick pos=left,
title={Reddit},
tick label style={font=\tiny},
title style={yshift=-4pt},
xticklabel style={yshift=2pt},
yticklabel style={xshift=2pt},
legend style={nodes={scale=0.5, transform shape}},
every tick/.style={
black,
semithick,
},
x label style={at={(axis description cs:0.5,-0.1)},anchor=north,font=\tiny},
y label style={at={(axis description cs:-0.1,.5)},rotate=90,anchor=south,font=\tiny},
x grid style={darkgray176},
xlabel={\(\displaystyle \epsilon\)},
xmin=-0.3, xmax=6.3,
xtick={0,1,2,3,4,5,6},
xticklabels={1,2,4,8,16,32,\(\displaystyle \infty\)},
ylabel={ACC},
ymin=0.2, ymax=1.0,
]
\path [draw=gray, fill=gray, opacity=0.2]
(axis cs:0,0.2204)
--(axis cs:0,0.1882)
--(axis cs:1,0.3653)
--(axis cs:2,0.8077)
--(axis cs:3,0.913)
--(axis cs:4,0.9314)
--(axis cs:5,0.9354)
--(axis cs:6,0.947)
--(axis cs:6,0.95)
--(axis cs:6,0.95)
--(axis cs:5,0.94)
--(axis cs:4,0.9354)
--(axis cs:3,0.9174)
--(axis cs:2,0.8391)
--(axis cs:1,0.4893)
--(axis cs:0,0.2204)
--cycle;

\path [draw=blue, fill=blue, opacity=0.2]
(axis cs:0,0.7512)
--(axis cs:0,0.7014)
--(axis cs:1,0.8591)
--(axis cs:2,0.8958)
--(axis cs:3,0.9111)
--(axis cs:4,0.9191)
--(axis cs:5,0.9243)
--(axis cs:6,0.9446)
--(axis cs:6,0.9466)
--(axis cs:6,0.9466)
--(axis cs:5,0.9271)
--(axis cs:4,0.9231)
--(axis cs:3,0.9145)
--(axis cs:2,0.9012)
--(axis cs:1,0.8695)
--(axis cs:0,0.7512)
--cycle;

\path [draw=red, fill=red, opacity=0.2]
(axis cs:0,0.7082) -- (axis cs:0,0.7132) -- (axis cs:6,0.7132) -- (axis cs:6,0.7082);

\addplot [semithick, black, mark=*, mark size=1, mark options={solid}]
table {%
0 0.2043
1 0.4273
2 0.8234
3 0.9152
4 0.9334
5 0.9377
6 0.9485
};
\addlegendentry{GIN}
\addplot [semithick, blue, dashed, mark=square*, mark size=1, mark options={solid}]
table {%
0 0.7263
1 0.8643
2 0.8985
3 0.9128
4 0.9211
5 0.9257
6 0.9456
};
\addlegendentry{GCN}
\addplot [thick, red, mark=none]
table {%
0 0.7107
6 0.7107
};
\addlegendentry{MLP}
\end{axis}
\end{tikzpicture}}
            \caption{}
        \end{subfigure}
        \begin{subfigure}{0.16\linewidth}
            \resizebox{\linewidth}{\linewidth}{
\begin{tikzpicture}

\definecolor{darkgray176}{RGB}{176,176,176}
\definecolor{gray}{RGB}{128,128,128}
\definecolor{green}{RGB}{0,128,0}
\definecolor{lightgray204}{RGB}{204,204,204}

\begin{axis}[
legend cell align={left},
legend style={
  fill opacity=0.8,
  draw opacity=1,
  text opacity=1,
  at={(0.52,0.45)},
  anchor=north west,
  draw=lightgray204
},
scale only axis,
width=2.25cm,
height=2cm,
tick align=outside,
tick pos=left,
title={Finance},
tick label style={font=\tiny},
title style={yshift=-4pt},
xticklabel style={yshift=2pt},
yticklabel style={xshift=2pt},
legend style={nodes={scale=0.5, transform shape}},
every tick/.style={
black,
semithick,
},
x label style={at={(axis description cs:0.5,-0.1)},anchor=north,font=\tiny},
y label style={at={(axis description cs:-0.1,.5)},rotate=90,anchor=south,font=\tiny},
x grid style={darkgray176},
xlabel={\(\displaystyle \epsilon\)},
xmin=-0.3, xmax=6.3,
xtick={0,1,2,3,4,5,6},
xticklabels={1,2,4,8,16,32,\(\displaystyle \infty\)},
ylabel={ACC},
ymin=0.5, ymax=0.85,
]
\path [draw=gray, fill=gray, opacity=0.2]
(axis cs:0,0.6895)
--(axis cs:0,0.5573)
--(axis cs:1,0.6693)
--(axis cs:2,0.6515)
--(axis cs:3,0.6631)
--(axis cs:4,0.6696)
--(axis cs:5,0.6773)
--(axis cs:6,0.7926)
--(axis cs:6,0.803)
--(axis cs:6,0.803)
--(axis cs:5,0.6955)
--(axis cs:4,0.692)
--(axis cs:3,0.6943)
--(axis cs:2,0.6891)
--(axis cs:1,0.6851)
--(axis cs:0,0.6895)
--cycle;

\path [draw=blue, fill=blue, opacity=0.2]
(axis cs:0,0.698)
--(axis cs:0,0.6938)
--(axis cs:1,0.715)
--(axis cs:2,0.7284)
--(axis cs:3,0.7315)
--(axis cs:4,0.7338)
--(axis cs:5,0.7427)
--(axis cs:6,0.7972)
--(axis cs:6,0.8054)
--(axis cs:6,0.8054)
--(axis cs:5,0.7497)
--(axis cs:4,0.7442)
--(axis cs:3,0.7477)
--(axis cs:2,0.743)
--(axis cs:1,0.731)
--(axis cs:0,0.698)
--cycle;

\path [draw=red, fill=red, opacity=0.2]
(axis cs:0,0.7113) -- (axis cs:0,0.7147) -- (axis cs:6,0.7147) -- (axis cs:6,0.7113);

\addplot [semithick, black, mark=*, mark size=1, mark options={solid}]
table {%
0 0.6234
1 0.6772
2 0.6703
3 0.6787
4 0.6808
5 0.6864
6 0.7978
};
\addlegendentry{GIN}
\addplot [semithick, blue, dashed, mark=square*, mark size=1, mark options={solid}]
table {%
0 0.6959
1 0.723
2 0.7357
3 0.7396
4 0.739
5 0.7462
6 0.8013
};
\addlegendentry{GCN}
\addplot [thick, red, mark=none]
table {%
0 0.7130
6 0.7130
};
\addlegendentry{MLP}
\end{axis}
\end{tikzpicture}}
            \caption{}
        \end{subfigure}
        \begin{subfigure}{0.16\linewidth}
            \resizebox{\linewidth}{\linewidth}{
\begin{tikzpicture}

\definecolor{darkgray176}{RGB}{176,176,176}
\definecolor{gray}{RGB}{128,128,128}
\definecolor{green}{RGB}{0,128,0}
\definecolor{lightgray204}{RGB}{204,204,204}

\begin{axis}[
legend cell align={left},
legend style={
  fill opacity=0.8,
  draw opacity=1,
  text opacity=1,
  at={(0.03,0.97)},
  anchor=north west,
  draw=lightgray204
},
scale only axis,
width=2.25cm,
height=2cm,
tick align=outside,
tick pos=left,
title={ogbn-products},
title style={yshift=-4pt},
tick label style={font=\tiny},
xticklabel style={yshift=2pt},
yticklabel style={xshift=2pt},
legend style={nodes={scale=0.5, transform shape}},
every tick/.style={
black,
semithick,
},
x label style={at={(axis description cs:0.5,-0.1)},anchor=north,font=\tiny},
y label style={at={(axis description cs:-0.1,.5)},rotate=90,anchor=south,font=\tiny},
x grid style={darkgray176},
xlabel={\(\displaystyle \epsilon\)},
xmin=-0.3, xmax=6.3,
xtick={0,1,2,3,4,5,6},
xticklabels={1,2,4,8,16,32,\(\displaystyle \infty\)},
ylabel={AUC},
ymin=0.431012779067799, ymax=0.802907937653953,
]
\path [draw=gray, fill=gray, opacity=0.2]
(axis cs:0,0.602176660111905)
--(axis cs:0,0.546547380682288)
--(axis cs:1,0.529904540801537)
--(axis cs:2,0.568640844218208)
--(axis cs:3,0.696724663442356)
--(axis cs:4,0.648009936126391)
--(axis cs:5,0.694683482192038)
--(axis cs:6,0.736415792783582)
--(axis cs:6,0.747821794081856)
--(axis cs:6,0.747821794081856)
--(axis cs:5,0.729590082701992)
--(axis cs:4,0.711278402018966)
--(axis cs:3,0.711543666945812)
--(axis cs:2,0.726565760382228)
--(axis cs:1,0.600088459391054)
--(axis cs:0,0.602176660111905)
--cycle;

\path [draw=green, fill=green, opacity=0.2]
(axis cs:0,0.552772369227595)
--(axis cs:0,0.499316778766979)
--(axis cs:1,0.489280708217722)
--(axis cs:2,0.567128732272145)
--(axis cs:3,0.446098922639897)
--(axis cs:4,0.549850633072463)
--(axis cs:5,0.662020794552003)
--(axis cs:6,0.731115861210483)
--(axis cs:6,0.738301930710474)
--(axis cs:6,0.738301930710474)
--(axis cs:5,0.667302460975767)
--(axis cs:4,0.634960972955443)
--(axis cs:3,0.588732108182758)
--(axis cs:2,0.625956354108116)
--(axis cs:1,0.562867629855046)
--(axis cs:0,0.552772369227595)
--cycle;

\addplot [semithick, black, mark=*, mark size=1, mark options={solid}]
table {%
0 0.574362020397096
1 0.564996500096295
2 0.647603302300218
3 0.704134165194084
4 0.679644169072679
5 0.712136782447015
6 0.742118793432719
};
\addlegendentry{GIN}
\addplot [semithick, green, dashed, mark=square*, mark size=1, mark options={solid}]
table {%
0 0.526044573997287
1 0.526074169036384
2 0.596542543190131
3 0.517415515411327
4 0.592405803013953
5 0.664661627763885
6 0.734708895960479
};
\addlegendentry{GCN}
\end{axis}

\end{tikzpicture}}
            \caption{}
        \end{subfigure}
        \begin{subfigure}{0.16\linewidth}
            \resizebox{\linewidth}{\linewidth}{
\begin{tikzpicture}

\definecolor{darkgray176}{RGB}{176,176,176}
\definecolor{gray}{RGB}{128,128,128}
\definecolor{green}{RGB}{0,128,0}
\definecolor{lightgray204}{RGB}{204,204,204}

\begin{axis}[
legend cell align={left},
legend style={
  fill opacity=0.8,
  draw opacity=1,
  text opacity=1,
  at={(0.03,0.97)},
  anchor=north west,
  draw=lightgray204
},
scale only axis,
width=2.25cm,
height=2cm,
tick align=outside,
tick pos=left,
title={Reddit},
title style={yshift=-4pt},
tick label style={font=\tiny},
xticklabel style={yshift=2pt},
yticklabel style={xshift=2pt},
legend style={nodes={scale=0.5, transform shape}},
every tick/.style={
black,
semithick,
},
x label style={at={(axis description cs:0.5,-0.1)},anchor=north,font=\tiny},
y label style={at={(axis description cs:-0.1,.5)},rotate=90,anchor=south,font=\tiny},
x grid style={darkgray176},
xlabel={\(\displaystyle \epsilon\)},
xmin=-0.3, xmax=6.3,
xtick={0,1,2,3,4,5,6},
xticklabels={1,2,4,8,16,32,\(\displaystyle \infty\)},
ylabel={AUC},
ymin=0.443367498325971, ymax=0.73515236932523,
]
\path [draw=gray, fill=gray, opacity=0.2]
(axis cs:0,0.500054802167273)
--(axis cs:0,0.499979067434759)
--(axis cs:1,0.487081764245714)
--(axis cs:2,0.491429487093707)
--(axis cs:3,0.477217057040684)
--(axis cs:4,0.581001404889022)
--(axis cs:5,0.630675341920215)
--(axis cs:6,0.717920330274454)
--(axis cs:6,0.721889420643446)
--(axis cs:6,0.721889420643446)
--(axis cs:5,0.674843008535527)
--(axis cs:4,0.611171206762723)
--(axis cs:3,0.54652932531443)
--(axis cs:2,0.519379303803085)
--(axis cs:1,0.533820380279406)
--(axis cs:0,0.500054802167273)
--cycle;

\path [draw=green, fill=green, opacity=0.2]
(axis cs:0,0.500497211681264)
--(axis cs:0,0.498622855484365)
--(axis cs:1,0.480031393230902)
--(axis cs:2,0.46046978823723)
--(axis cs:3,0.456630447007755)
--(axis cs:4,0.533704798268444)
--(axis cs:5,0.503816401633104)
--(axis cs:6,0.637916182764434)
--(axis cs:6,0.685824128299612)
--(axis cs:6,0.685824128299612)
--(axis cs:5,0.586101478204373)
--(axis cs:4,0.56041968614143)
--(axis cs:3,0.563076264034495)
--(axis cs:2,0.514048879835651)
--(axis cs:1,0.509900828848723)
--(axis cs:0,0.500497211681264)
--cycle;

\addplot [semithick, black, mark=*, mark size=1, mark options={solid}]
table {%
0 0.500016934801016
1 0.51045107226256
2 0.505404395448396
3 0.511873191177557
4 0.596086305825873
5 0.652759175227871
6 0.71990487545895
};
\addlegendentry{GIN}
\addplot [semithick, green, dashed, mark=square*, mark size=1, mark options={solid}]
table {%
0 0.499560033582815
1 0.494966111039813
2 0.487259334036441
3 0.509853355521125
4 0.547062242204937
5 0.544958939918739
6 0.661870155532023
};
\addlegendentry{GCN}
\end{axis}
\end{tikzpicture}}
            \caption{}
        \end{subfigure}
        \begin{subfigure}{0.16\linewidth}
            \resizebox{\linewidth}{\linewidth}{
\begin{tikzpicture}

\definecolor{darkgray176}{RGB}{176,176,176}
\definecolor{gray}{RGB}{128,128,128}
\definecolor{green}{RGB}{0,128,0}
\definecolor{lightgray204}{RGB}{204,204,204}

\begin{axis}[
legend cell align={left},
legend style={
  fill opacity=0.8,
  draw opacity=1,
  text opacity=1,
  at={(0.03,0.97)},
  anchor=north west,
  draw=lightgray204
},
scale only axis,
width=2.25cm,
height=2cm,
tick align=outside,
tick pos=left,
title={Finance},
tick label style={font=\tiny},
title style={yshift=-4pt},
xticklabel style={yshift=2pt},
yticklabel style={xshift=2pt},
legend style={nodes={scale=0.5, transform shape}},
every tick/.style={
black,
semithick,
},
x label style={at={(axis description cs:0.5,-0.1)},anchor=north,font=\tiny},
y label style={at={(axis description cs:-0.1,.5)},rotate=90,anchor=south,font=\tiny},
x grid style={darkgray176},
xlabel={\(\displaystyle \epsilon\)},
xmin=-0.3, xmax=6.3,
xtick={0,1,2,3,4,5,6},
xticklabels={1,2,4,8,16,32,\(\displaystyle \infty\)},
ylabel={AUC},
ymin=0.374827492589487, ymax=0.701017795046696,
]
\path [draw=gray, fill=gray, opacity=0.2]
(axis cs:0,0.512993688378993)
--(axis cs:0,0.43453785840271)
--(axis cs:1,0.445376487222329)
--(axis cs:2,0.462456322672348)
--(axis cs:3,0.431289839114708)
--(axis cs:4,0.45424108049143)
--(axis cs:5,0.464402734058819)
--(axis cs:6,0.537435839612462)
--(axis cs:6,0.686190963116823)
--(axis cs:6,0.686190963116823)
--(axis cs:5,0.487933310940003)
--(axis cs:4,0.489979829821448)
--(axis cs:3,0.470972627133641)
--(axis cs:2,0.514818555059631)
--(axis cs:1,0.525107788836534)
--(axis cs:0,0.512993688378993)
--cycle;

\path [draw=green, fill=green, opacity=0.2]
(axis cs:0,0.458626776897333)
--(axis cs:0,0.403290966681947)
--(axis cs:1,0.38965432451936)
--(axis cs:2,0.395600106938558)
--(axis cs:3,0.417302759501524)
--(axis cs:4,0.426797114248234)
--(axis cs:5,0.471714698863776)
--(axis cs:6,0.495997459990456)
--(axis cs:6,0.531098581674719)
--(axis cs:6,0.531098581674719)
--(axis cs:5,0.547566118085103)
--(axis cs:4,0.557744845109438)
--(axis cs:3,0.54701172782659)
--(axis cs:2,0.490060128933953)
--(axis cs:1,0.474112207967592)
--(axis cs:0,0.458626776897333)
--cycle;

\addplot [semithick, black, mark=*, mark size=1, mark options={solid}]
table {%
0 0.473765773390851
1 0.485242138029432
2 0.488637438865989
3 0.451131233124175
4 0.472110455156439
5 0.476168022499411
6 0.611813401364642
};
\addlegendentry{GIN}
\addplot [semithick, green, dashed, mark=square*, mark size=1, mark options={solid}]
table {%
0 0.43095887178964
1 0.431883266243476
2 0.442830117936255
3 0.482157243664057
4 0.492270979678836
5 0.509640408474439
6 0.513548020832588
};
\addlegendentry{GCN}
\end{axis}
\end{tikzpicture}}
            \caption{}
        \end{subfigure}
        \caption{(a)-(c): Evaluation of privacy-utility trade-off regarding the \vesper framework, with $\text{mean}\pm\text{std}$ plotted according to $10$ trials. The result of MLP is plotted as a reference line. Results below this line are practically problematic as it fails to exploit the graph information. (d)-(f): AUC ($\text{mean}\pm\text{std}$ over $10$ trials) of membership inference attacks.}
        \label{fig: privacy_utility_trade_off}
    \end{figure}
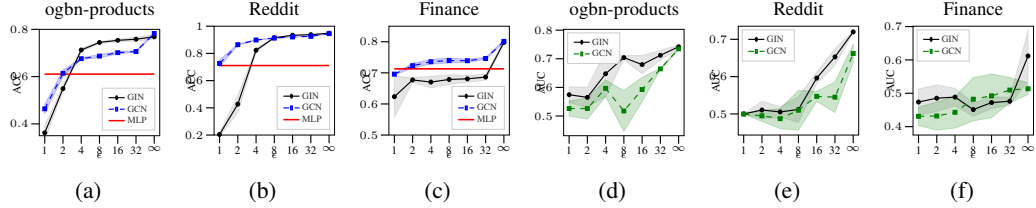

    \begin{table*}[]
        \centering
        \caption{Experimental results over three benchmark datasets using both non-private and private approaches, reported with format \result{\text{\texttt{mean}}}{\text{\texttt{std}}}, with \texttt{mean} and \texttt{std} (abbreviation for standard deviation) computed under $10$ trials for each setting.}
        \resizebox{\textwidth}{!}{%
            \begin{tabular}{l  c  c  c  c | c  c  c  c | c  c  c  c }
    \toprule
    \multicolumn{13}{c}{Non-private approaches} \\
    \midrule
    Model & \multicolumn{4}{c}{\ogb} & \multicolumn{4}{c}{Reddit} & \multicolumn{4}{c}{Finance} \\
    \midrule
    MLP & \multicolumn{4}{c}{\result{61.06}{0.08}} & \multicolumn{4}{c}{\result{71.07}{0.25}} & \multicolumn{4}{c}{\result{71.30}{0.17}}\\
    GIN & \multicolumn{4}{c}{\result{76.84}{0.94}} & \multicolumn{4}{c}{\result{94.85}{0.15}} & \multicolumn{4}{c}{\result{79.78}{0.52}}\\
    GCN & \multicolumn{4}{c}{\result{78.26}{0.36}} & \multicolumn{4}{c}{\result{94.56}{0.10}} & \multicolumn{4}{c}{\result{80.13}{0.41}}\\
    \midrule
    \multicolumn{13}{c}{Private approaches} \\
    \midrule
    Model & \multicolumn{4}{c}{\ogb} & \multicolumn{4}{c}{Reddit} & \multicolumn{4}{c}{Finance} \\
    \midrule
     & $\epsilon = 4$ & $\epsilon = 8$ & $\epsilon = 16$ & $\epsilon = 32$ & $\epsilon = 4$ & $\epsilon = 8$ & $\epsilon = 16$ & $\epsilon = 32$ & $\epsilon = 4$ & $\epsilon = 8$ & $\epsilon = 16$ & $\epsilon = 32$ \\
     \midrule

    VFGNN & \result{26.94}{0.00}& \result{40.96}{1.74}& \result{56.27}{1.35}& \result{69.57}{0.26}& \result{19.40}{2.56}& \result{31.20}{1.75}& \result{43.67}{1.18}& \result{86.72}{0.33}& \result{53.87}{6.86}& \result{56.17}{5.31}& \result{54.62}{4.09}& \result{52.12}{4.78} \\

    GAP & \result{59.20}{2.13}& \result{65.02}{0.82}& \result{66.20}{2.34}& \result{67.14}{0.34}& \result{76.84}{1.71}& \result{86.59}{0.48}& \result{88.57}{1.69}& \result{89.65}{0.30}& \result{52.02}{9.39}& \result{48.66}{7.14}& \result{59.04}{7.95}& \result{67.54}{4.41} \\
    \midrule
    
    \textbf{\vesper(GIN)} & \result{71.27}{0.70} & \result{74.46}{0.45} & \result{75.36}{0.49} & \result{75.82}{0.92} & \result{82.34}{1.57} & \result{91.52}{0.22} & \result{93.34}{0.20} & \result{93.77}{0.23} & \result{67.03}{1.88} & \result{67.87}{1.56} & \result{68.08}{1.12} & \result{68.64}{0.91} \\
    \textbf{\vesper(GCN)} & \result{67.60}{0.40} & \result{68.68}{0.67} & \result{70.13}{0.55}& \result{70.62}{0.33} & \result{89.85}{0.27} & \result{91.28}{0.17} & \result{92.11}{0.20} & \result{92.57}{0.14} & \result{73.57}{0.73} & \result{73.96}{0.81} & \result{73.90}{0.52} & \result{74.62}{0.35} \\
    \bottomrule
\end{tabular}
        }
        \label{tab: performance}
    \end{table*}
    \subsection{Performance and privacy-utility trade-off}
    According to our empirical experience, obtaining reasonable performance in the \emph{high privacy} regime is difficult, especially for baseline algorithms. Therefore, we report two sets of results: Firstly, we thoroughly investigate the privacy-utility trade-off regarding the proposed \vesper framework under both GIN and GCN aggregation schemes and plot the results in figure \ref{fig: privacy_utility_trade_off}. Secondly, we report comparisons of \vesper against private and non-private baselines with only moderate to large privacy budgets and summarize the results in table \ref{tab: performance}. The results demonstrate that the proposed \vesper framework exhibits competitive privacy-utility trade-off under both GIN and GCN aggregators. Moreover, a comparison of GIN and GCN aggregator suggests that summation pooling excels when the underlying graph is dense (i.e., \ogb and Reddit), while introducing the truncated message passing mechanism helps achieving better results over sparse graphs (i.e., Finance). Finally, \vesper demonstrates a better privacy-utility trade-off compared to other private GNN baselines. 
    
    \subsection{Protection against membership inference attacks}\label{sec: mia}
    We launch a membership inference attack (MIA) \cite{olatunji2021membership} to empirically investigate the resilience of \vesper against practical privacy risks that targets the membership of nodes instead of edges, which is regarded as a stronger attack than edge MIA. We provide a detailed description of the attack setup in appendix \ref{sec: mia_full}. The attack is conducted over trained models under privacy budgets $\epsilon\in\{1, 2, 4, 8, 16, 32, \infty\}$, where $\epsilon=\infty$ indicated no privacy protection is adopted. We use ROC-AUC (AUC) to evaluate the attack performance. We report the attack performances in Figure~\ref{fig: mia_attack_auc}. From the results, we observe that when privacy protection is disabled ($\epsilon=\infty$), the attacks show non-negligible effectiveness, especially on obgn-products and Reddit datasets. Generally, with the privacy budget getting smaller (privacy getting stronger), the attack performances sharply decline. With an appropriate privacy budget, the attacks on all three datasets are successfully defended with AUC reduced to around 0.5 (random guess baseline). \par\noindent
    \textbf{Additional experiments} We will report a series of ablation studies that assess the effect of maximum degree $D$, minimum degree $\mind$ for PMP-GCN and batch size in appendix \ref{sec: full_ablation}.

    \section{Related Works}
    \subsection{Graph representation learning in the federated setting}
    The majority of GRL research in the federated setting is based on the horizontal setup, with each party holding its own local graph data \cite{wu2021fedgnn, he2021fedgraphnn, ramezani2021learn}. 
    The adoption of VFL paradigms to GRL is relatively few, VFGNN \cite{zhou2020vertically} uses additive secret sharing to combine feature information held by different parties, followed by a straightforward adaptation of the SplitNN framework \cite{gupta2018distributed} with the underlying neural model being graph neural networks. In \cite{cheung2021fedsgc, wu2021linkteller}, the authors discussed VFL setups where node features and graph topology belong to different parties. We refer to the recent survey \cite{liu2022federated} for a more detailed overview. 
    \subsection{Graph representation learning with differential privacy guarantees}
    The most straightforward way to integrate DP techniques into GRL is via adopting private optimization algorithms like DP-SGD\cite{abadi2016deep}.
    However, meaningful notions of differential privacy over graph data (i.e., the edge model \cite{nissim2007smooth} and node model \cite{kasiviswanathan2013analyzing}) are semantically different from that of i.i.d. data, and require refined privacy analysis which is sometimes ignored in the privacy analysis in previous works \cite{zhou2020vertically, wu2021fedgnn, olatunji2021releasing}. In \cite{daigavane2021node}, the authors analyzed the DP-SGD algorithm in the node DP model. 
    The GAP model \cite{sajadmanesh2022gap} proposed a three-stage training procedure and analyzed its privacy guarantee in both edge DP and node DP models. However, we noticed that the privacy analysis in \cite{sajadmanesh2022gap} did not properly address the effect of sampling, resulting in an overly optimistic performance. Considering only edge DP, randomized response (RR) \cite{wu2021linkteller} that flips each entry of the underlying graph's adjacent matrix guarantees privacy (in a stronger \emph{local} sense), but makes reasonable privacy-utility trade-off extremely hard to obtain in practice. 

    \section{Conclusion and discussions}
    We present the \vesper framework as a differentially private solution to node classification in the VFL setup using graph representation learning techniques. The core algorithmic component of \vesper is the PMP scheme that allows efficient learning under both dense and sparse graph data. We demonstrate the practicality and effectiveness of the proposed framework by establishing theoretical DP guarantees as well as investigating its ability for privacy protection and privacy-utility trade-off empirically. We will discuss possible extensions and future directions of the \vesper framework in appendix \ref{sec: extensions}.
    
    \bibliographystyle{abbrv}
    \bibliography{privacy,private_gnn,attack}

    \appendix
    \section{Some standard tools for \renyi differential privacy}\label{sec: rdp}
    \textbf{\renyi divergence} the \renyi divergence between distributions of random variables $X$ and $Y$ given by
        \begin{align}
            \dalpha{X}{Y} = \dfrac{1}{\alpha - 1}\log \expectation{y \sim \mathbb{P}_Y}{\left(\dfrac{d\mathbb{P}_X}{d\mathbb{P}_Y}(y)\right)^{\alpha}}.
        \end{align}
        Here we use $\frac{d\mathbb{P}_X}{d\mathbb{P}_Y}(\cdot)$ to denote the density ratio between $X$ and $Y$ (or more formally the Radon-Nikodym derivative of the induced probability measure $\mathbb{P}_X$ with respect to $\mathbb{P}_Y$).
    We state here a couple of useful results in implementing and proving algorithms with \renyi differential privacy. The results will be stated under the context of graph algorithms in the edge DP model. The first lemma is the composition theorem of RDP:
    \begin{lemma}[Composition of \renyi DP \cite{mironov2017renyi}]
        Let $\mathcal{M}_1$ be a graph-input mechanism that satisfies $(\alpha, \epsilon_1)$-RDP, and $\mathcal{M}_2$ be a graph-input mechanism that is allowed to further depend on the output of $\mathcal{M}_1$ satisfying $(\alpha, \epsilon_2)$-RDP, then the composed mechanism $(\mathcal{M}_1 \circ \mathcal{M}_2)(G) = \mathcal{M}_2(\mathcal{M}_1(G), G)$ satisfies $(\alpha, \epsilon_1 + \epsilon_2)$-RDP.
    \end{lemma}
    The second lemma is the conversion rule of RDP to the approximate $(\epsilon, \delta)$-DP:
    \begin{lemma}[Conversion of RDP to $(\epsilon, \delta)$-DP, \cite{balle2020hypothesis}]
        Let mechanism $\mathcal{M}$ satisfy $(\alpha, \epsilon)$-RDP, then it is $(\epsilon^\prime, \delta)$-DP for
        \begin{align}\label{eqn: rdp_conversion}
            \epsilon^\prime = \epsilon - \dfrac{\log(\delta \alpha)}{\alpha - 1} + \log(1 - \frac{1}{\alpha})
        \end{align}
        with any $\delta > 0$.
    \end{lemma}
    \section{Missing proofs}\label{sec: proofs}
    \begin{proof}[Proof of theorem \ref{thm: sampled}]
        The proof contains two steps: In the first step, we prove that without neighborhood sampling, the algorithm is $(\alpha, \epsilon(\alpha))$-RDP. Then in the second step, we construct an algorithm that is \emph{less or equally private} than the procedure \ref{alg: priv_gnn_sample} and could be directly analyzed by \cite[Theorem 9]{wang2019subsampled} such that the privacy guarantee of the algorithm is the one stated in the theorem. \\
        \textbf{Step $1$}: We ignore neighborhood sampling and consider the first layer. By \cite[proposition 3]{mironov2017renyi}, the collection of perturbed embeddings $\{\check{h}^{(1)}_v\}_{v \in V}$, with $\check{h}_v^{(1)} = \widetilde{h}_v^{(1)} + N(0, \theta^2 I_d)$, 
        is $\left(\alpha, \dfrac{\alpha \se_1^2}{2\theta^2}\right)$-\renyi differentially private for any $\alpha > 1$. Since the nonlinear transform does not involve edge information and is therefore treated as a post-processing mechanism \cite{dwork2014algorithmic}, it follows that the collection of transformed embeddings $\{h^{(1)}_v\}_{v \in V}$, with $h^{(1)}_v = \sigma\left(\check{h}_v^{(1)}\right)$, is also $\left(\alpha, \dfrac{\alpha \se_1^2}{2\theta^2}\right)$-\renyi differentially private for any $\alpha > 1$. Now we view the operation in a single layer as a base mechanism, an $L$-layer perturbed message passing procedure could thus be viewed as composing the base mechanism for $L$ times. Then it follows by the composition theorem of \renyi differential privacy \cite[Proposition 1]{mironov2017renyi} that the non-sampling version is $(\alpha, \epsilon(\alpha))$-RDP.\\
        \textbf{Step $2$}: First we introduce some additional notations: Denote $G^{(L)}_v$ as the $L$-layer rooted subgraph with root node $v \in V$ produced by a neighborhood sampler. Then each training batch consists of $B$ randomly chosen subgraphs $(G^{(L)}_{v_1}, \ldots, G^{(L)}_{v_B})$ with root nodes $(v_1, \ldots, v_B)$, further denote $G^{(L)}_B$ as the graph generated via combining $(G^{(L)}_{v_1}, \ldots, G^{(L)}_{v_B})$ with node set $V^{(L)}_B$ and edge set $E^{(L)}_B$. Let $N_e$ be the maximum number of possible subgraphs that any specific edge might affect after an $L$-layer message passing procedure, then we may bound the probability of the event that any specific edge $e \in E$ is contained in $G^{(L)}_B$
        \begin{align}\label{eqn: edge_prob}
            \max_{e \in E}\pp{e \in E^{(L)}_B} \le 1 - \dfrac{\binom{N_T - N_e}{B}}{\binom{N_T}{B}}
        \end{align}
        Since the maximum degree is bounded from above by $D$, we further bound the above probability by bounding $N_e$
        \begin{align*}
            N_e \le 2 \sum_{l=0}^{L-1} D^l = \dfrac{2(D^L - 1)}{D - 1},
        \end{align*}
        yielding
        \begin{align*}
            \max_{e \in E}\pp{e \in E^{(L)}_B} \le \gamma := 1 - \dfrac{\binom{N_T - \frac{2(D^L - 1)}{D - 1}}{B}}{\binom{N_T}{B}}
        \end{align*}
        Next, we construct an algorithm $\mathcal{A}$ as follows: For a batch of size $B$, the algorithm first randomly samples $B$ nodes, then independently samples $\lfloor\gamma |E| \rfloor$ edges to form a subgraph $G^{\mathcal{A}}_B$. Then it returns the result via running a non-sampled version of algorithm \ref{alg: priv_gnn_sample} over $G^{\mathcal{A}}_B$. Here note the fact that for any edge $e$, the probability that $e$ is contained in $G^{(L)}_B$ is no greater than the probability that it is contained in $G^{\mathcal{A}}_B$. Therefore, algorithm $\mathcal{A}$ is less or equally private than the procedure \ref{alg: priv_gnn_sample}.\par\noindent
        Since the privacy guarantee of algorithm $\mathcal{A}$ can be directly analyzed by \cite[Theorem 9]{wang2019subsampled}, yielding a \renyi differential privacy guarantee of $(\alpha, \epsilon_\gamma(\alpha))$ with $\epsilon_\gamma(\alpha)$ defined in \eqref{eqn: sampling_rdp}. The result of the theorem follows.
    \end{proof}
    In the proofs of propotision \ref{prop: gin_es} and \ref{prop: gcn_es}, \textbf{we will prove for an arbitrary weight matrix $W$} and the result trivially applies to the weight matrices in each layer of the message passing procedure.
    \begin{proof}[Proof of proposition \ref{prop: gin_es}]
        We inherit the notation from definition \ref{def: edge_sensitivity} that $G^\prime$ is the adjacent graph via removing the edge $(u^*, v^*)$ from $G$. Write the summation pooling update rule as
        \begin{align}
            \widetilde{h}_v \leftarrow \sum_{u \in N(v)} W h_u,\quad \forall v \in V
        \end{align}
        Note that the only two node embeddings that get affected by the removal is $h_{v^*}$ and $h_{u^*}$. For node $v^*$, it follows that
        \begin{align}
            \norm{\widetilde{h}_{v^*} - \widetilde{h}^\prime_{v^*}} = \norm{W h_{u^*}} \le \opnorm{W}\norm{h_{u^*}} = \opnorm{W}.
        \end{align}
        Where the last equality follows since the input representations are $\ell_2$-normalized. The same argument leads to 
        \begin{align}
            \norm{\widetilde{h}_{u^*} - \widetilde{h}^\prime_{u^*}} \le \opnorm{W}.
        \end{align}
        Then we arrive at
        \begin{align}\label{eqn: edge_sym}
            \begin{aligned}
                &\sqrt{\sum_{v \in V}\|\widetilde{h}_v - \widetilde{h}^\prime_v\|_2^2} \\
                =& \sqrt{\norm{\widetilde{h}_{u^*} - \widetilde{h}^\prime_{u^*}}^2 + \norm{\widetilde{h}_{v^*} - \widetilde{h}^\prime_{v^*}}^2} \\
                \le& \sqrt{2}\opnorm{W}.
            \end{aligned}
        \end{align}
    \end{proof}
    \begin{proof}[Proof of proposition \ref{prop: gcn_es}]
        Recall the update rule of GCN \cite{kipf2016semi}
        \begin{align}\label{eqn: gcn_kw}
            \widetilde{h}_v \leftarrow \dfrac{W h_v}{d_v + 1} + \sum_{u \in N(v)}\frac{W h_u}{\sqrt{d_v + 1}\sqrt{d_u + 1}}
        \end{align}
        Following similar arguments in the proof of proposition \ref{prop: gin_es}, we first bound the difference between $\widetilde{h}_{v^*}$ and $\widetilde{h}^\prime_{v^*}$ in $\ell_2$ norm. First we inspect
        \begin{align}
            \begin{aligned}
                &\widetilde{h}_{v^*} - \widetilde{h}^\prime_{v^*} \\
                =& - \dfrac{W h_{v^*}}{d_{v^*}(d_{v^*} + 1)} + \dfrac{W h_{u^*}}{\sqrt{d_{u^*} + 1}\sqrt{d_{v^*} + 1}} \\
                &+ \sum_{u \in N({v^*})\setminus \{u^*\}}\left(\dfrac{W h_u}{\sqrt{d_v + 1}\sqrt{d_u + 1}} - \dfrac{W h_u}{\sqrt{d_v}\sqrt{d_u + 1}}\right) \\
                :=& \mathcal{T}_1 + \mathcal{T}_2 + \mathcal{T}_3
            \end{aligned}
        \end{align}
        Bounding $\mathcal{T}_1$ and $\mathcal{T}_2$ are straightforward
        \begin{align}
            &\mathcal{T}_1 \le \dfrac{\opnorm{W}}{\mind(\mind + 1)} \label{eqn: t1}\\
            &\mathcal{T}_2 \le \dfrac{\opnorm{W}}{\mind + 1}\label{eqn: t2}
        \end{align}
        Where we use the minimum degree assumption \eqref{eqn: min_degree}. To bound $\mathcal{T}_3$, we use the inequality
        \begin{align}\label{eqn: useful_ineq}
            \forall x > 0, \quad \frac{1}{x} - \frac{1}{x+1} \le \frac{1}{2x^{3/2}}
        \end{align}
        Now we proceed as follows:
        \begin{align}
            &\norm{\mathcal{T}_3} \\
            \le& \sum_{u \in N({v^*})\setminus \{u^*\}} \norm{
                \dfrac{W h_u}{\sqrt{d_v + 1}\sqrt{d_u + 1}} - \dfrac{W h_u}{\sqrt{d_v}\sqrt{d_u + 1}}
            }\\
            \le& \sum_{u \in N({v^*})\setminus \{u^*\}} \dfrac{\opnorm{W}}{\sqrt{d_u + 1}}\left(\frac{1}{\sqrt{d_v}} - \frac{1}{\sqrt{d_v + 1}}\right)\\
            \le& \sum_{u \in N({v^*})\setminus \{u^*\}} \dfrac{\opnorm{W}}{2 \sqrt{d_u + 1} d_v^{3/2}} \tag*{By inequality \eqref{eqn: useful_ineq}}\\
            \le& \sum_{u \in N({v^*})\setminus \{u^*\}} \dfrac{\opnorm{W}}{2 \sqrt{\mind + 1} d_v^{3/2}}
            \tag*{By assumption \eqref{eqn: min_degree}} \\
            =& \dfrac{\opnorm{W} (d_v - 1)}{2\sqrt{\mind + 1}d_v^{3/2}} \label{eqn: t3_tmp}
        \end{align}
        To further bound \eqref{eqn: t3_tmp}, observe that the function
        \begin{align}
            f(x) = \dfrac{x - 1}{x^{3/2}},\qquad x > 1
        \end{align}
        attains its maximum at $x = 3$, and becomes monotonically decreasing as $x \ge 3$. Since $d_v \ge \mind$, it suffices to check the case for $\mind = 2$ and $\mind \ge 3$ separately. For $\mind = 2$, we have
        \begin{align}\label{eqn: dmin_2}
            \norm{\mathcal{T}_3} \le \dfrac{\opnorm{W} (3 - 1)}{2\sqrt{2 + 1} 3^{3/2}} < \dfrac{\opnorm{W} (2 - 1)}{2\sqrt{2} 2^{3/2}} = \dfrac{\opnorm{W} (\mind - 1)}{2\sqrt{\mind} \mind^{3/2}}
        \end{align}
        For $\mind \ge 3$, we have
        \begin{align}\label{eqn: dmin_ge_3}
            \norm{\mathcal{T}_3} \le \dfrac{\opnorm{W} (\mind - 1)}{2\sqrt{\mind + 1}\mind^{3/2}} < \dfrac{\opnorm{W} (\mind - 1)}{2\sqrt{\mind}\mind^{3/2}}
        \end{align}
        Combining \eqref{eqn: dmin_2} and \eqref{eqn: dmin_ge_3} we get
        \begin{align}
            \norm{\mathcal{T}_3} \le \dfrac{\opnorm{W}(1 - 1/\mind)}{2\mind}.\label{eqn: t3}
        \end{align}
        Finally, combine \eqref{eqn: t1}, \eqref{eqn: t2} and \eqref{eqn: t3}, and then use the argument in \eqref{eqn: edge_sym} yield the result.
    \end{proof}
    \section{A complete report of empirical evaluations}\label{sec: exp_full}
    \subsection{Datasets}\label{sec: dataset}
    \textbf{The \ogb dataset} is an undirected and unweighted graph that represents an Amazon product co-purchasing network \cite{hu2020open}. Nodes represent products sold on Amazon, and edges between two products indicate that the products are purchased together. The node features are generated as dimensionality-reduced bag-of-words of the product descriptions. The learning task is to predict the category of a product in a multi-class classification setup with $47$ classes. We took the dataset and train/validation/test splitting from the official implementation available in the \href{https://github.com/snap-stanford/ogb}{\texttt{ogb} library}.\par\noindent
    \textbf{The Reddit dataset} is a graph dataset from Reddit posts made in September, $2014$. Nodes represent Reddit posts; two posts are connected if the same user comments on both. The node features are generated by combining the word embeddings of the corresponding post's metadata, as described in \cite{hamilton2017inductive}. The learning task is to predict which community different Reddit posts belong to, with $41$ classes. We use the training/validation/testing splitting from \cite{hamilton2017inductive}.\par\noindent
    \textbf{The Finance dataset} This dataset is generated from transaction records collected from one of the world's leading online payment systems. The underlying graph is generated by treating users as nodes, and two nodes are connected if at least one transaction occurred between corresponding users within a predefined time period. The business goal is to identify risky users which is cast into an algorithmic problem of node classification with a binary label. The node features are obtained via statistical summaries of corresponding users' behavior on the platform during a specific time period. The training and testing datasets are constructed under two distinct time windows with no overlap. \par\noindent
    We list the summary statistics in table \ref{tab: dataset_summary}
    \begin{table}
        \centering
        \caption{Summary statistics of the evaluation datasets}
        \begin{tabular}{lccc}
            \toprule
            & \ogb & Reddit & Alipay \\
            \midrule
            \# Nodes & $2449029$ & $232965$ & $1132511$ \\
            \# Edges & $123718280$ & $114615892$ & $2447370$ \\
            \# Training nodes & $196615$ & $153431$ & $848963$ \\
            \# Node features & $100$ & $602$ & $155$\\
            \# Classes & $47$ & $41$ & $2$\\
            \bottomrule
        \end{tabular}
        \label{tab: dataset_summary}
    \end{table}

    \subsection{A differentially private analysis of degree distributions}\label{sec: histogram}
    While all three datasets are large in scale (i.e., with the number of nodes exceeding $100,000$), they differ significantly in their degree distributions. Specifically, the average node degree is much higher in the \ogb ($\approx 50$) and Reddit dataset ($\approx 490$) than that in the Finance dataset ($\approx 2.2$). Following literature in random graph theory \cite{chung2001diameter}, we might consider \ogb and Reddit as dense graphs (with average degree $\gg \log(N)$) and Finance as a sparse graph (with average degree $\ll \log(N)$). For a better illustration, we conduct a differentially private analysis of degree distribution (with $(0.1, 0)$-differential privacy). Since we are basically interested in graphs with bounded degrees (and enforcing the property using neighborhood sampling), during the computation of degree distributions, we group all nodes with degrees over $50$ to a single category (i.e., with degrees greater than or equal to $50$). As a result, the final histogram represents counts of nodes under degree $\{0, 1, \ldots, 49, \ge 50\}$. The analysis is based on the trivial fact that the addition and removal of any single edge would change the degree of $2$ nodes by exactly $1$, therefore, the $\ell_1$ sensitivity \cite{dwork2014algorithmic} of the degree distribution histogram query is exactly $2$. By standard Laplacian mechanism \cite{dwork2006calibrating, dwork2014algorithmic}, we add to each count an independent Laplacian noise with scale $\frac{2}{\epsilon}$ with $\epsilon = 0.1$. The resulting private histograms are shown in figure \ref{fig: histograms} with counts reported at a logarithmic scale.
    \begin{figure}[h]
        \centering
        \begin{subfigure}{0.32\linewidth}
            \resizebox{\linewidth}{\linewidth}{
\begin{tikzpicture}

\definecolor{darkgray176}{RGB}{176,176,176}
\definecolor{steelblue31119180}{RGB}{31,119,180}

\begin{axis}[
scale only axis,
width=4cm,
height=3cm,
tick align=outside,
tick pos=left,
title={\ogb},
tick label style={font=\tiny},
every tick/.style={
black,
semithick,
},
x label style={at={(axis description cs:0.5,-0.1)},anchor=north,font=\tiny},
y label style={at={(axis description cs:-0.1,.5)},rotate=90,anchor=south,font=\tiny},
x grid style={darkgray176},
xlabel={degree},
xmin=-2.94, xmax=52.94,
xtick style={color=black},
xtick={0, 10, 20, 30, 40, 50},
ytick={0, 1, 2, 3, 4, 5, 6},
y grid style={darkgray176},
ylabel={$\log_{10}(\text{counts})$},
ymin=0, ymax=6.2,
ytick style={color=black}
]
\draw[draw=none,fill=steelblue31119180,fill opacity=0.8] (axis cs:-0.4,0) rectangle (axis cs:0.4,4.68512988315264);
\draw[draw=none,fill=steelblue31119180,fill opacity=0.8] (axis cs:0.6,0) rectangle (axis cs:1.4,5.01284447239621);
\draw[draw=none,fill=steelblue31119180,fill opacity=0.8] (axis cs:1.6,0) rectangle (axis cs:2.4,4.93132419991988);
\draw[draw=none,fill=steelblue31119180,fill opacity=0.8] (axis cs:2.6,0) rectangle (axis cs:3.4,4.87879583686076);
\draw[draw=none,fill=steelblue31119180,fill opacity=0.8] (axis cs:3.6,0) rectangle (axis cs:4.4,4.86808413060804);
\draw[draw=none,fill=steelblue31119180,fill opacity=0.8] (axis cs:4.6,0) rectangle (axis cs:5.4,4.84773128905974);
\draw[draw=none,fill=steelblue31119180,fill opacity=0.8] (axis cs:5.6,0) rectangle (axis cs:6.4,4.81657083377104);
\draw[draw=none,fill=steelblue31119180,fill opacity=0.8] (axis cs:6.6,0) rectangle (axis cs:7.4,4.79395596400676);
\draw[draw=none,fill=steelblue31119180,fill opacity=0.8] (axis cs:7.6,0) rectangle (axis cs:8.4,4.7650999416694);
\draw[draw=none,fill=steelblue31119180,fill opacity=0.8] (axis cs:8.6,0) rectangle (axis cs:9.4,4.73676863215884);
\draw[draw=none,fill=steelblue31119180,fill opacity=0.8] (axis cs:9.6,0) rectangle (axis cs:10.4,4.69679063139526);
\draw[draw=none,fill=steelblue31119180,fill opacity=0.8] (axis cs:10.6,0) rectangle (axis cs:11.4,4.66483669625984);
\draw[draw=none,fill=steelblue31119180,fill opacity=0.8] (axis cs:11.6,0) rectangle (axis cs:12.4,4.63169452117734);
\draw[draw=none,fill=steelblue31119180,fill opacity=0.8] (axis cs:12.6,0) rectangle (axis cs:13.4,4.59988000496342);
\draw[draw=none,fill=steelblue31119180,fill opacity=0.8] (axis cs:13.6,0) rectangle (axis cs:14.4,4.57694823465041);
\draw[draw=none,fill=steelblue31119180,fill opacity=0.8] (axis cs:14.6,0) rectangle (axis cs:15.4,4.54576120449756);
\draw[draw=none,fill=steelblue31119180,fill opacity=0.8] (axis cs:15.6,0) rectangle (axis cs:16.4,4.52371680533396);
\draw[draw=none,fill=steelblue31119180,fill opacity=0.8] (axis cs:16.6,0) rectangle (axis cs:17.4,4.49928447167268);
\draw[draw=none,fill=steelblue31119180,fill opacity=0.8] (axis cs:17.6,0) rectangle (axis cs:18.4,4.47838926327768);
\draw[draw=none,fill=steelblue31119180,fill opacity=0.8] (axis cs:18.6,0) rectangle (axis cs:19.4,4.46451516501767);
\draw[draw=none,fill=steelblue31119180,fill opacity=0.8] (axis cs:19.6,0) rectangle (axis cs:20.4,4.43999104177094);
\draw[draw=none,fill=steelblue31119180,fill opacity=0.8] (axis cs:20.6,0) rectangle (axis cs:21.4,4.41952154949178);
\draw[draw=none,fill=steelblue31119180,fill opacity=0.8] (axis cs:21.6,0) rectangle (axis cs:22.4,4.40719910831528);
\draw[draw=none,fill=steelblue31119180,fill opacity=0.8] (axis cs:22.6,0) rectangle (axis cs:23.4,4.394219041963);
\draw[draw=none,fill=steelblue31119180,fill opacity=0.8] (axis cs:23.6,0) rectangle (axis cs:24.4,4.38018805921623);
\draw[draw=none,fill=steelblue31119180,fill opacity=0.8] (axis cs:24.6,0) rectangle (axis cs:25.4,4.36550144261016);
\draw[draw=none,fill=steelblue31119180,fill opacity=0.8] (axis cs:25.6,0) rectangle (axis cs:26.4,4.35385314736583);
\draw[draw=none,fill=steelblue31119180,fill opacity=0.8] (axis cs:26.6,0) rectangle (axis cs:27.4,4.33926692976464);
\draw[draw=none,fill=steelblue31119180,fill opacity=0.8] (axis cs:27.6,0) rectangle (axis cs:28.4,4.3377730164862);
\draw[draw=none,fill=steelblue31119180,fill opacity=0.8] (axis cs:28.6,0) rectangle (axis cs:29.4,4.31663531210145);
\draw[draw=none,fill=steelblue31119180,fill opacity=0.8] (axis cs:29.6,0) rectangle (axis cs:30.4,4.30811001107435);
\draw[draw=none,fill=steelblue31119180,fill opacity=0.8] (axis cs:30.6,0) rectangle (axis cs:31.4,4.30368666093213);
\draw[draw=none,fill=steelblue31119180,fill opacity=0.8] (axis cs:31.6,0) rectangle (axis cs:32.4,4.29945758890528);
\draw[draw=none,fill=steelblue31119180,fill opacity=0.8] (axis cs:32.6,0) rectangle (axis cs:33.4,4.2898278585207);
\draw[draw=none,fill=steelblue31119180,fill opacity=0.8] (axis cs:33.6,0) rectangle (axis cs:34.4,4.28129974957903);
\draw[draw=none,fill=steelblue31119180,fill opacity=0.8] (axis cs:34.6,0) rectangle (axis cs:35.4,4.27689172488416);
\draw[draw=none,fill=steelblue31119180,fill opacity=0.8] (axis cs:35.6,0) rectangle (axis cs:36.4,4.26377139879968);
\draw[draw=none,fill=steelblue31119180,fill opacity=0.8] (axis cs:36.6,0) rectangle (axis cs:37.4,4.25803159967308);
\draw[draw=none,fill=steelblue31119180,fill opacity=0.8] (axis cs:37.6,0) rectangle (axis cs:38.4,4.2557721128193);
\draw[draw=none,fill=steelblue31119180,fill opacity=0.8] (axis cs:38.6,0) rectangle (axis cs:39.4,4.2540334565414);
\draw[draw=none,fill=steelblue31119180,fill opacity=0.8] (axis cs:39.6,0) rectangle (axis cs:40.4,4.24916701175067);
\draw[draw=none,fill=steelblue31119180,fill opacity=0.8] (axis cs:40.6,0) rectangle (axis cs:41.4,4.24335357733532);
\draw[draw=none,fill=steelblue31119180,fill opacity=0.8] (axis cs:41.6,0) rectangle (axis cs:42.4,4.23610171388167);
\draw[draw=none,fill=steelblue31119180,fill opacity=0.8] (axis cs:42.6,0) rectangle (axis cs:43.4,4.2421367328053);
\draw[draw=none,fill=steelblue31119180,fill opacity=0.8] (axis cs:43.6,0) rectangle (axis cs:44.4,4.22252708064072);
\draw[draw=none,fill=steelblue31119180,fill opacity=0.8] (axis cs:44.6,0) rectangle (axis cs:45.4,4.22317701908957);
\draw[draw=none,fill=steelblue31119180,fill opacity=0.8] (axis cs:45.6,0) rectangle (axis cs:46.4,4.2173185892572);
\draw[draw=none,fill=steelblue31119180,fill opacity=0.8] (axis cs:46.6,0) rectangle (axis cs:47.4,4.21196690868068);
\draw[draw=none,fill=steelblue31119180,fill opacity=0.8] (axis cs:47.6,0) rectangle (axis cs:48.4,4.19895161799202);
\draw[draw=none,fill=steelblue31119180,fill opacity=0.8] (axis cs:48.6,0) rectangle (axis cs:49.4,4.1913872607817);
\draw[draw=none,fill=steelblue31119180,fill opacity=0.8] (axis cs:49.6,0) rectangle (axis cs:50.4,5.66693496407755);
\end{axis}

\end{tikzpicture}}
        \end{subfigure}
        \begin{subfigure}{0.32\linewidth}
            \resizebox{\linewidth}{\linewidth}{
\begin{tikzpicture}

\definecolor{darkgray176}{RGB}{176,176,176}
\definecolor{steelblue31119180}{RGB}{31,119,180}

\begin{axis}[
scale only axis,
width=4cm,
height=3cm,
tick align=outside,
tick pos=left,
title={Reddit},
tick label style={font=\tiny},
every tick/.style={
black,
semithick,
},
x label style={at={(axis description cs:0.5,-0.1)},anchor=north,font=\tiny},
y label style={at={(axis description cs:-0.1,.5)},rotate=90,anchor=south,font=\tiny},
x grid style={darkgray176},
xlabel={degree},
xmin=-2.94, xmax=52.94,
xtick style={color=black},
xtick={0, 10, 20, 30, 40, 50},
ytick={0, 1, 2, 3, 4, 5, 6},
y grid style={darkgray176},
ylabel={$\log_{10}(\text{counts})$},
ymin=0, ymax=6.2,
ytick style={color=black}
]
\draw[draw=none,fill=steelblue31119180,fill opacity=0.8] (axis cs:-0.4,0) rectangle (axis cs:0.4,1.11403446873716);
\draw[draw=none,fill=steelblue31119180,fill opacity=0.8] (axis cs:0.6,0) rectangle (axis cs:1.4,3.32076979433825);
\draw[draw=none,fill=steelblue31119180,fill opacity=0.8] (axis cs:1.6,0) rectangle (axis cs:2.4,3.26173919575811);
\draw[draw=none,fill=steelblue31119180,fill opacity=0.8] (axis cs:2.6,0) rectangle (axis cs:3.4,3.21165512821711);
\draw[draw=none,fill=steelblue31119180,fill opacity=0.8] (axis cs:3.6,0) rectangle (axis cs:4.4,3.17753728706394);
\draw[draw=none,fill=steelblue31119180,fill opacity=0.8] (axis cs:4.6,0) rectangle (axis cs:5.4,3.14364008630666);
\draw[draw=none,fill=steelblue31119180,fill opacity=0.8] (axis cs:5.6,0) rectangle (axis cs:6.4,3.10991678277271);
\draw[draw=none,fill=steelblue31119180,fill opacity=0.8] (axis cs:6.6,0) rectangle (axis cs:7.4,3.09968158279174);
\draw[draw=none,fill=steelblue31119180,fill opacity=0.8] (axis cs:7.6,0) rectangle (axis cs:8.4,3.08206791494449);
\draw[draw=none,fill=steelblue31119180,fill opacity=0.8] (axis cs:8.6,0) rectangle (axis cs:9.4,3.04688625424562);
\draw[draw=none,fill=steelblue31119180,fill opacity=0.8] (axis cs:9.6,0) rectangle (axis cs:10.4,3.01283837483749);
\draw[draw=none,fill=steelblue31119180,fill opacity=0.8] (axis cs:10.6,0) rectangle (axis cs:11.4,3.01786885585056);
\draw[draw=none,fill=steelblue31119180,fill opacity=0.8] (axis cs:11.6,0) rectangle (axis cs:12.4,3.01911742407034);
\draw[draw=none,fill=steelblue31119180,fill opacity=0.8] (axis cs:12.6,0) rectangle (axis cs:13.4,3.01994781313664);
\draw[draw=none,fill=steelblue31119180,fill opacity=0.8] (axis cs:13.6,0) rectangle (axis cs:14.4,3.00000118463624);
\draw[draw=none,fill=steelblue31119180,fill opacity=0.8] (axis cs:14.6,0) rectangle (axis cs:15.4,2.98587658110611);
\draw[draw=none,fill=steelblue31119180,fill opacity=0.8] (axis cs:15.6,0) rectangle (axis cs:16.4,2.98542769914628);
\draw[draw=none,fill=steelblue31119180,fill opacity=0.8] (axis cs:16.6,0) rectangle (axis cs:17.4,2.95230933180058);
\draw[draw=none,fill=steelblue31119180,fill opacity=0.8] (axis cs:17.6,0) rectangle (axis cs:18.4,2.93196750026732);
\draw[draw=none,fill=steelblue31119180,fill opacity=0.8] (axis cs:18.6,0) rectangle (axis cs:19.4,2.97954961646108);
\draw[draw=none,fill=steelblue31119180,fill opacity=0.8] (axis cs:19.6,0) rectangle (axis cs:20.4,2.91803176750458);
\draw[draw=none,fill=steelblue31119180,fill opacity=0.8] (axis cs:20.6,0) rectangle (axis cs:21.4,2.92685811088573);
\draw[draw=none,fill=steelblue31119180,fill opacity=0.8] (axis cs:21.6,0) rectangle (axis cs:22.4,2.92890908557488);
\draw[draw=none,fill=steelblue31119180,fill opacity=0.8] (axis cs:22.6,0) rectangle (axis cs:23.4,2.92428069634276);
\draw[draw=none,fill=steelblue31119180,fill opacity=0.8] (axis cs:23.6,0) rectangle (axis cs:24.4,2.90902231492123);
\draw[draw=none,fill=steelblue31119180,fill opacity=0.8] (axis cs:24.6,0) rectangle (axis cs:25.4,2.90848648139206);
\draw[draw=none,fill=steelblue31119180,fill opacity=0.8] (axis cs:25.6,0) rectangle (axis cs:26.4,2.9025482619621);
\draw[draw=none,fill=steelblue31119180,fill opacity=0.8] (axis cs:26.6,0) rectangle (axis cs:27.4,2.89762859082945);
\draw[draw=none,fill=steelblue31119180,fill opacity=0.8] (axis cs:27.6,0) rectangle (axis cs:28.4,2.91907951964826);
\draw[draw=none,fill=steelblue31119180,fill opacity=0.8] (axis cs:28.6,0) rectangle (axis cs:29.4,2.87390318160203);
\draw[draw=none,fill=steelblue31119180,fill opacity=0.8] (axis cs:29.6,0) rectangle (axis cs:30.4,2.89487116583545);
\draw[draw=none,fill=steelblue31119180,fill opacity=0.8] (axis cs:30.6,0) rectangle (axis cs:31.4,2.88536276252599);
\draw[draw=none,fill=steelblue31119180,fill opacity=0.8] (axis cs:31.6,0) rectangle (axis cs:32.4,2.88309490914604);
\draw[draw=none,fill=steelblue31119180,fill opacity=0.8] (axis cs:32.6,0) rectangle (axis cs:33.4,2.84942108937763);
\draw[draw=none,fill=steelblue31119180,fill opacity=0.8] (axis cs:33.6,0) rectangle (axis cs:34.4,2.84136117742248);
\draw[draw=none,fill=steelblue31119180,fill opacity=0.8] (axis cs:34.6,0) rectangle (axis cs:35.4,2.84695701013442);
\draw[draw=none,fill=steelblue31119180,fill opacity=0.8] (axis cs:35.6,0) rectangle (axis cs:36.4,2.84010780635781);
\draw[draw=none,fill=steelblue31119180,fill opacity=0.8] (axis cs:36.6,0) rectangle (axis cs:37.4,2.8591399357947);
\draw[draw=none,fill=steelblue31119180,fill opacity=0.8] (axis cs:37.6,0) rectangle (axis cs:38.4,2.84819079732554);
\draw[draw=none,fill=steelblue31119180,fill opacity=0.8] (axis cs:38.6,0) rectangle (axis cs:39.4,2.81224652212452);
\draw[draw=none,fill=steelblue31119180,fill opacity=0.8] (axis cs:39.6,0) rectangle (axis cs:40.4,2.87274041545617);
\draw[draw=none,fill=steelblue31119180,fill opacity=0.8] (axis cs:40.6,0) rectangle (axis cs:41.4,2.83442243814053);
\draw[draw=none,fill=steelblue31119180,fill opacity=0.8] (axis cs:41.6,0) rectangle (axis cs:42.4,2.81954573044401);
\draw[draw=none,fill=steelblue31119180,fill opacity=0.8] (axis cs:42.6,0) rectangle (axis cs:43.4,2.80414129204543);
\draw[draw=none,fill=steelblue31119180,fill opacity=0.8] (axis cs:43.6,0) rectangle (axis cs:44.4,2.81023435179595);
\draw[draw=none,fill=steelblue31119180,fill opacity=0.8] (axis cs:44.6,0) rectangle (axis cs:45.4,2.82994844836097);
\draw[draw=none,fill=steelblue31119180,fill opacity=0.8] (axis cs:45.6,0) rectangle (axis cs:46.4,2.84261094167312);
\draw[draw=none,fill=steelblue31119180,fill opacity=0.8] (axis cs:46.6,0) rectangle (axis cs:47.4,2.8135828082846);
\draw[draw=none,fill=steelblue31119180,fill opacity=0.8] (axis cs:47.6,0) rectangle (axis cs:48.4,2.76417817134916);
\draw[draw=none,fill=steelblue31119180,fill opacity=0.8] (axis cs:48.6,0) rectangle (axis cs:49.4,2.81224652212452);
\draw[draw=none,fill=steelblue31119180,fill opacity=0.8] (axis cs:49.6,0) rectangle (axis cs:50.4,5.2740515790206);
\end{axis}

\end{tikzpicture}}
        \end{subfigure}
        \begin{subfigure}{0.32\linewidth}
            \resizebox{\linewidth}{\linewidth}{
\begin{tikzpicture}

\definecolor{darkgray176}{RGB}{176,176,176}
\definecolor{steelblue31119180}{RGB}{31,119,180}

\begin{axis}[
scale only axis,
width=4cm,
height=3cm,
tick align=outside,
tick pos=left,
title={Finance},
tick label style={font=\tiny},
every tick/.style={
black,
semithick,
},
x label style={at={(axis description cs:0.5,-0.1)},anchor=north,font=\tiny},
y label style={at={(axis description cs:-0.1,.5)},rotate=90,anchor=south,font=\tiny},
x grid style={darkgray176},
xlabel={degree},
xmin=-2.94, xmax=52.94,
xtick style={color=black},
xtick={0, 10, 20, 30, 40, 50},
ytick={0, 1, 2, 3, 4, 5, 6},
y grid style={darkgray176},
ylabel={$\log_{10}(\text{counts})$},
ymin=0, ymax=6.2,
ytick style={color=black}
]
\draw[draw=none,fill=steelblue31119180,fill opacity=0.8] (axis cs:-0.4,0) rectangle (axis cs:0.4,5.66028065618358);
\draw[draw=none,fill=steelblue31119180,fill opacity=0.8] (axis cs:0.6,0) rectangle (axis cs:1.4,5.33102961123361);
\draw[draw=none,fill=steelblue31119180,fill opacity=0.8] (axis cs:1.6,0) rectangle (axis cs:2.4,5.18921699373659);
\draw[draw=none,fill=steelblue31119180,fill opacity=0.8] (axis cs:2.6,0) rectangle (axis cs:3.4,4.94023585274421);
\draw[draw=none,fill=steelblue31119180,fill opacity=0.8] (axis cs:3.6,0) rectangle (axis cs:4.4,4.77641546768407);
\draw[draw=none,fill=steelblue31119180,fill opacity=0.8] (axis cs:4.6,0) rectangle (axis cs:5.4,4.5968462910033);
\draw[draw=none,fill=steelblue31119180,fill opacity=0.8] (axis cs:5.6,0) rectangle (axis cs:6.4,4.44934988685274);
\draw[draw=none,fill=steelblue31119180,fill opacity=0.8] (axis cs:6.6,0) rectangle (axis cs:7.4,4.30487104403719);
\draw[draw=none,fill=steelblue31119180,fill opacity=0.8] (axis cs:7.6,0) rectangle (axis cs:8.4,4.18411401309399);
\draw[draw=none,fill=steelblue31119180,fill opacity=0.8] (axis cs:8.6,0) rectangle (axis cs:9.4,4.05218092341025);
\draw[draw=none,fill=steelblue31119180,fill opacity=0.8] (axis cs:9.6,0) rectangle (axis cs:10.4,3.93935305822368);
\draw[draw=none,fill=steelblue31119180,fill opacity=0.8] (axis cs:10.6,0) rectangle (axis cs:11.4,3.83604998998397);
\draw[draw=none,fill=steelblue31119180,fill opacity=0.8] (axis cs:11.6,0) rectangle (axis cs:12.4,3.71220098082787);
\draw[draw=none,fill=steelblue31119180,fill opacity=0.8] (axis cs:12.6,0) rectangle (axis cs:13.4,3.63414173115035);
\draw[draw=none,fill=steelblue31119180,fill opacity=0.8] (axis cs:13.6,0) rectangle (axis cs:14.4,3.51991607043637);
\draw[draw=none,fill=steelblue31119180,fill opacity=0.8] (axis cs:14.6,0) rectangle (axis cs:15.4,3.4335572494351);
\draw[draw=none,fill=steelblue31119180,fill opacity=0.8] (axis cs:15.6,0) rectangle (axis cs:16.4,3.34785111931177);
\draw[draw=none,fill=steelblue31119180,fill opacity=0.8] (axis cs:16.6,0) rectangle (axis cs:17.4,3.26189810386872);
\draw[draw=none,fill=steelblue31119180,fill opacity=0.8] (axis cs:17.6,0) rectangle (axis cs:18.4,3.17250699912943);
\draw[draw=none,fill=steelblue31119180,fill opacity=0.8] (axis cs:18.6,0) rectangle (axis cs:19.4,3.07139272700505);
\draw[draw=none,fill=steelblue31119180,fill opacity=0.8] (axis cs:19.6,0) rectangle (axis cs:20.4,3.00973608913448);
\draw[draw=none,fill=steelblue31119180,fill opacity=0.8] (axis cs:20.6,0) rectangle (axis cs:21.4,2.92976180598464);
\draw[draw=none,fill=steelblue31119180,fill opacity=0.8] (axis cs:21.6,0) rectangle (axis cs:22.4,2.859541378937);
\draw[draw=none,fill=steelblue31119180,fill opacity=0.8] (axis cs:22.6,0) rectangle (axis cs:23.4,2.7683948795402);
\draw[draw=none,fill=steelblue31119180,fill opacity=0.8] (axis cs:23.6,0) rectangle (axis cs:24.4,2.70814149291428);
\draw[draw=none,fill=steelblue31119180,fill opacity=0.8] (axis cs:24.6,0) rectangle (axis cs:25.4,2.6381610153081);
\draw[draw=none,fill=steelblue31119180,fill opacity=0.8] (axis cs:25.6,0) rectangle (axis cs:26.4,2.59623560048195);
\draw[draw=none,fill=steelblue31119180,fill opacity=0.8] (axis cs:26.6,0) rectangle (axis cs:27.4,2.53614332804906);
\draw[draw=none,fill=steelblue31119180,fill opacity=0.8] (axis cs:27.6,0) rectangle (axis cs:28.4,2.49230120368886);
\draw[draw=none,fill=steelblue31119180,fill opacity=0.8] (axis cs:28.6,0) rectangle (axis cs:29.4,2.32974629258692);
\draw[draw=none,fill=steelblue31119180,fill opacity=0.8] (axis cs:29.6,0) rectangle (axis cs:30.4,2.33379245547389);
\draw[draw=none,fill=steelblue31119180,fill opacity=0.8] (axis cs:30.6,0) rectangle (axis cs:31.4,2.29594372683439);
\draw[draw=none,fill=steelblue31119180,fill opacity=0.8] (axis cs:31.6,0) rectangle (axis cs:32.4,2.27800173380809);
\draw[draw=none,fill=steelblue31119180,fill opacity=0.8] (axis cs:32.6,0) rectangle (axis cs:33.4,2.18659290467401);
\draw[draw=none,fill=steelblue31119180,fill opacity=0.8] (axis cs:33.6,0) rectangle (axis cs:34.4,2.10609344837931);
\draw[draw=none,fill=steelblue31119180,fill opacity=0.8] (axis cs:34.6,0) rectangle (axis cs:35.4,2.08874314235702);
\draw[draw=none,fill=steelblue31119180,fill opacity=0.8] (axis cs:35.6,0) rectangle (axis cs:36.4,2.03210015564903);
\draw[draw=none,fill=steelblue31119180,fill opacity=0.8] (axis cs:36.6,0) rectangle (axis cs:37.4,1.9716067762755);
\draw[draw=none,fill=steelblue31119180,fill opacity=0.8] (axis cs:37.6,0) rectangle (axis cs:38.4,1.94778331727875);
\draw[draw=none,fill=steelblue31119180,fill opacity=0.8] (axis cs:38.6,0) rectangle (axis cs:39.4,1.91735502609051);
\draw[draw=none,fill=steelblue31119180,fill opacity=0.8] (axis cs:39.6,0) rectangle (axis cs:40.4,1.85534558086804);
\draw[draw=none,fill=steelblue31119180,fill opacity=0.8] (axis cs:40.6,0) rectangle (axis cs:41.4,1.7533635532277);
\draw[draw=none,fill=steelblue31119180,fill opacity=0.8] (axis cs:41.6,0) rectangle (axis cs:42.4,1.83040483618724);
\draw[draw=none,fill=steelblue31119180,fill opacity=0.8] (axis cs:42.6,0) rectangle (axis cs:43.4,1.80394405359853);
\draw[draw=none,fill=steelblue31119180,fill opacity=0.8] (axis cs:43.6,0) rectangle (axis cs:44.4,1.71324980275614);
\draw[draw=none,fill=steelblue31119180,fill opacity=0.8] (axis cs:44.6,0) rectangle (axis cs:45.4,1.60928858101722);
\draw[draw=none,fill=steelblue31119180,fill opacity=0.8] (axis cs:45.6,0) rectangle (axis cs:46.4,1.42604498768737);
\draw[draw=none,fill=steelblue31119180,fill opacity=0.8] (axis cs:46.6,0) rectangle (axis cs:47.4,1.65964384277276);
\draw[draw=none,fill=steelblue31119180,fill opacity=0.8] (axis cs:47.6,0) rectangle (axis cs:48.4,1.56432689025009);
\draw[draw=none,fill=steelblue31119180,fill opacity=0.8] (axis cs:48.6,0) rectangle (axis cs:49.4,1.52726052049427);
\draw[draw=none,fill=steelblue31119180,fill opacity=0.8] (axis cs:49.6,0) rectangle (axis cs:50.4,2.58508984086927);
\end{axis}

\end{tikzpicture}}
        \end{subfigure}
        \caption{$(0.1, 0)$-differentially private histograms of three benchmark datasets}
        \label{fig: histograms}
    \end{figure}
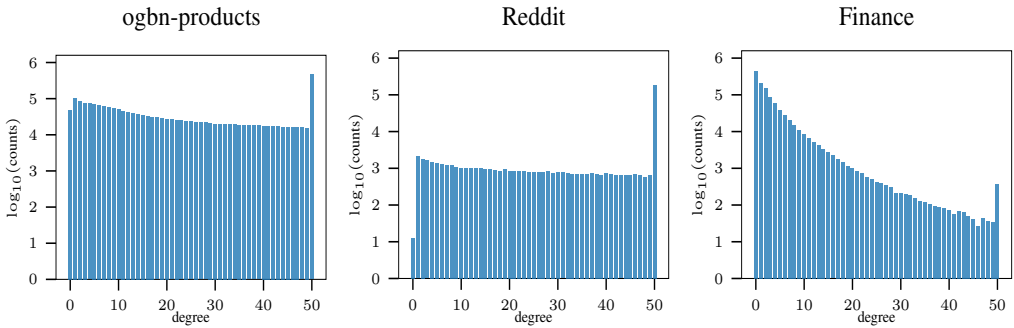
    According to the histograms, we find that both the \ogb and the Reddit contain a large portion of high-degree nodes (as illustrated by the spiking bar at the $\ge 50$ category), while the Finance dataset exhibits a concentration on the lower-degree nodes. In particular, around half of the nodes in the Finance dataset are singular nodes without any neighbors. According to the discussion in section \ref{sec: priv_gnn}, it is expected that the Finance dataset is more challenging for (private) message passing under sum pooling. 

    \subsection{Baselines}\label{sec: baselines}
    \begin{description}
        \item[MLP without edge information] Using an MLP over node features directly is the most trivial solution to the learning task as it totally ignores edge information. Equivalently, this corresponds to removing the PRE module in the \vesper framework. This baseline is of critical importance in evaluating the privacy-utility trade-off since in practice we require the model trained with graph information to significantly outperform MLPs.
        \item[Non-private GNN counterparts] We compare with ordinary GCN and GIN models without privacy guarantees, or equivalently set the $\epsilon$ parameter in the \vesper framework to be infinity. Ideally, the performance of these models should serve as performance upper bounds for corresponding \vesper models.
        \item[GNN models with privacy guarantees] We consider two alternative approaches to private GRL, namely the VFGNN model \cite{zhou2020vertically} and the GAP model \cite{sajadmanesh2022gap}. Both models add Gaussian noise to node embeddings and are implementable in the vertically federated setting. The privacy guarantees stated in \cite{zhou2020vertically} are not directly applicable to the edge privacy setup. Thus we provide an independent privacy analysis in the edge privacy model, similar to theorem \ref{thm: sampled}. For the GAP model \cite{sajadmanesh2022gap}, the effect of sampling was not properly addressed in the original paper, and we instead use a corrected version by noting that the aggregation perturbation mechanism is equivalent to PMP-GIN without learnable parameters and use theorem \ref{thm: sampled} to analyze it.
    \end{description}

    \subsection{Experimental setup}\label{sec: exp_setup_full}
    \textbf{Training configurations} Across all models (i.e., MLP or GNN-related baselines), we used a hidden dimension of $d=128$ for the \ogb dataset, $d=512$ for the Reddit dataset and $d=256$ for the Finance dataset. We use the Adam optimizer with a learning rate $0.001$ across all the tasks. We trained each model for $5$ epochs under the \ogb and Reddit dataset and $2$ epochs under the Finance dataset. For the GNN-related approaches, according to the private degree histogram analysis, we tune the maximum degree with range $\{20, 50\}$ for \ogb and Reddit datasets and $\{10, 20\}$ for the Finance dataset. For \vesper, we tested different decoder architectures as described in section \ref{sec: vesper}. For \vesper using the GCN aggregator, we tune the minimum degree hyperparameter over $\{10, 20, 40\}$ for \ogb and Reddit and $\{3, 5\}$ for Finance. We tested the number of message passing rounds with $L \in \{1, 2, 3\}$. We found that $L=2$ works best in general across all datasets for \vesper and GAP, and $L=1$ works best for VFGNN. We use the \texttt{DGL} framework \cite{wang2019deep} for the implementation of GNN algorithms. \par\noindent
    \textbf{privacy configurations}: All the privacy reports are based on the $(\epsilon, \delta)$-differential privacy model, with $\delta$ being the reciprocal of the number of edges. 
    To adequately inspect the privacy-utility trade-off, we aim to evaluate all the models with differential privacy guarantees under the total privacy costs (privacy budgets) $\epsilon \in \{1, 2, 4, 8, 16, 32\}$, with the privacy costs accounted during the entire training period. 
    We treat the setting where $\epsilon \in \{1, 2\}$ as of \emph{high privacy}, $\epsilon \in \{4, 8\}$ as of \emph{moderate privacy}, and the rest as of \emph{low privacy}.
    For \vesper and VFGNN, we add spectral normalization to each GNN layer. For the privacy accountant, we base our implementation upon AMA implementation available in the \href{https://github.com/google/differential-privacy/tree/main/python}{\texttt{dp-accounting}} library and use an adjusted sampling probability according to theorem \ref{thm: sampled}. For each required privacy level, we compute the minimum scale of Gaussian noise via conducting a binary search over the adjusted AMA, with associating spectral norms of weight matrices fixed at one in all layers. \par\noindent
    \textbf{Evaluation metrics} We adopt classification accuracy (ACC) as the evaluation metric for the \ogb and Reddit datasets, and ROC-AUC score (AUC) as the evaluation metric for the Finance dataset.

    \subsection{A privacy analysis for VFGNN \cite{zhou2020vertically}}\label{sec: vfgnn_privacy}
    The VFGNN model \cite{zhou2020vertically} adds Gaussian noise to the normalized output of an $L$ layer message passing. It is trivial to check that under the edge differential privacy model with noise scale $\theta$, VFGNN is $\left(\alpha, \dfrac{\alpha \mathcal{S}_L^2}{2\theta^2}\right)$-\renyi differentially private for any $\alpha > 1$, with edge sensitivity $\mathcal{S}_L$ slightly generalized (c.f. definition \ref{def: edge_sensitivity}) with $h$ and $h^\prime$ being the output of an $L$-layer message passing procedure. Without loss of generality, we assume the norm of node embeddings before perturbation to be $C$. To compute $\mathcal{S}_L$, first note that the change of any node embedding under an edge addition or removal operation is bounded by $2C$, it remains to bound the number of node embeddings that may get affected upon an edge addition or removal operation. Through similar arguments in the proof of theorem \ref{thm: sampled}, we bound this count from above by $2 \sum_{l=0}^{L-1} D^l$. We conclude the analysis in the following proposition:
    \begin{proposition}[\renyi DP guarantee for VFGNN]
        The output of an $L$-layer VFGNN model with normalization constant $C$ satisfies $\left(\alpha, \dfrac{4\alpha C^2 \left(\sum_{l=0}^{L-1} D^l\right)}{\theta^2}\right)$-\renyi differential privacy for any $\alpha > 1$.
    \end{proposition}

    \subsection{A complete report of performance comparisons}\label{sec: full_performance}
    In this section, we present a complete report of empirical performance containing both the concatenation decoder (denote via using the "-C" postfix) and the GRU decoder  (denote via using the "-G" postfix) listed in table \ref{tab: performance_full}. 
    \begin{table*}[h]
        \centering
        \caption{Experimental results over three benchmark datasets using both non-private and private approaches, reported with format \result{\text{\texttt{mean}}}{\text{\texttt{std}}}, with \texttt{mean} and \texttt{std} (abbreviation for standard deviation) computed under $10$ trials for each setting.}
        \resizebox{\textwidth}{!}{%
            \begin{tabular}{l  c  c  c  c  c  c  c  c  c  c  c  c }
    \toprule
    \multicolumn{13}{c}{Non-private approaches} \\
    \midrule
    Model & \multicolumn{4}{c}{\ogb} & \multicolumn{4}{c}{Reddit} & \multicolumn{4}{c}{Finance} \\
    \midrule
    MLP & \multicolumn{4}{c}{\result{61.06}{0.08}} & \multicolumn{4}{c}{\result{71.07}{0.25}} & \multicolumn{4}{c}{\result{71.30}{0.17}}\\
    GIN-C & \multicolumn{4}{c}{\result{76.84}{0.94}} & \multicolumn{4}{c}{\result{94.85}{0.15}} & \multicolumn{4}{c}{\result{79.76}{0.59}}\\
    GIN-G & \multicolumn{4}{c}{\result{76.10}{0.67}} & \multicolumn{4}{c}{\result{94.38}{0.16}} & \multicolumn{4}{c}{\result{79.78}{0.52}}\\
    GCN-C & \multicolumn{4}{c}{\result{78.26}{0.36}} & \multicolumn{4}{c}{\result{94.56}{0.10}} & \multicolumn{4}{c}{\result{79.70}{0.60}}\\
    GCN-G & \multicolumn{4}{c}{\result{75.80}{0.65}} & \multicolumn{4}{c}{\result{94.37}{0.13}} & \multicolumn{4}{c}{\result{80.13}{0.41}}\\
    \midrule
    \multicolumn{13}{c}{Private approaches} \\
    \midrule
    Model & \multicolumn{4}{c}{\ogb} & \multicolumn{4}{c}{Reddit} & \multicolumn{4}{c}{Finance} \\
    \midrule
     & $\epsilon = 4$ & $\epsilon = 8$ & $\epsilon = 16$ & $\epsilon = 32$ & $\epsilon = 4$ & $\epsilon = 8$ & $\epsilon = 16$ & $\epsilon = 32$ & $\epsilon = 4$ & $\epsilon = 8$ & $\epsilon = 16$ & $\epsilon = 32$ \\
     \midrule
    VFGNN & \result{26.94}{0.00}& \result{40.96}{1.74}& \result{56.27}{1.35}& \result{69.57}{0.26}& \result{19.40}{2.56}& \result{31.20}{1.75}& \result{43.67}{1.18}& \result{86.72}{0.33}& \result{53.87}{6.86}& \result{56.17}{5.31}& \result{54.62}{4.09}& \result{52.12}{4.78} \\

    GAP & \result{59.20}{2.13}& \result{65.02}{0.82}& \result{66.20}{2.34}& \result{67.14}{0.34}& \result{76.84}{1.71}& \result{86.59}{0.48}& \result{88.57}{1.69}& \result{89.65}{0.30}& \result{52.02}{9.39}& \result{48.66}{7.14}& \result{59.04}{7.95}& \result{67.54}{4.41} \\
    \midrule

    \textbf{\vesper(GIN-C)} & \result{71.27}{0.70} & \result{74.46}{0.45} & \result{75.18}{0.52} & \result{75.82}{0.92} & \result{82.34}{1.57} & \result{91.52}{0.22} & \result{93.34}{0.20} & \result{93.77}{0.23} & \result{67.03}{1.88} & \result{67.87}{1.56} & \result{68.08}{1.12} & \result{68.64}{0.91} \\
    \textbf{\vesper(GIN-G)} & \result{69.55}{0.61} & \result{73.45}{0.30} & \result{75.36}{0.49} & \result{75.60}{0.51} & \result{75.48}{2.61} & \result{89.60}{0.44} & \result{92.01}{0.48} & \result{92.95}{0.41} & \result{58.90}{0.67} & \result{64.43}{0.23} & \result{66.05}{0.17} & \result{67.10}{2.15} \\
    \textbf{\vesper(GCN-C)} & \result{66.73}{0.40} & \result{67.53}{0.79} & \result{70.13}{0.55} & \result{70.47}{0.41} & \result{89.85}{0.27} & \result{91.28}{0.17} & \result{92.11}{0.20} & \result{92.57}{0.14} & \result{70.44}{0.89} & \result{71.39}{0.91} & \result{72.56}{0.57} & \result{72.72}{1.11} \\
    \textbf{\vesper(GCN-G)} & \result{67.60}{0.40} & \result{68.68}{0.67} & \result{70.02}{0.48} & \result{70.62}{0.33} & \result{89.54}{0.11} & \result{91.02}{0.31} & \result{91.47}{0.20} & \result{92.41}{0.19} & \result{73.57}{0.73} & \result{73.96}{0.81} & \result{73.90}{0.52} & \result{74.62}{0.35} \\
    \bottomrule
\end{tabular}
        }
        \label{tab: performance_full}
    \end{table*}
    We summarize our experimental findings as follows:
    \begin{description}
        \item[Privacy-utility trade-off] Overall, the \vesper framework exhibits competitive privacy-utility trade-off under both GIN and GCN aggregators. Specifically, \vesper using GCN aggregator achieves better performance than the non-private MLP baseline across all three datasets under a decent privacy protection level with $\epsilon = 4.$. The results suggest that the model has the capability of privately learning the structural information brought by the underlying graph. Moreover, if we are allowed to relax the privacy requirement via adopting bigger privacy budgets (i.e., $\epsilon \in \{16, 32\}$), \vesper might obtain high-performance models that closely match the performance of non-private versions. This phenomenon is particularly evident when using the GIN aggregator under \ogb and Reddit datasets, where the utility loss is cut to be around $1$ percent or fewer in the low-privacy regime. 
        \item[Sparse v.s. dense graphs] On one hand, \vesper using the GIN aggregator performs better than the GCN counterpart in terms of privacy-utility trade-off in moderate to high-privacy regime when the underlying graph is dense, i.e., on \ogb and Reddit datasets, which is likely due to the fact that GIN aggregates from the full neighborhood (with high SNR when the underlying graph is dense as discussed in section \ref{sec: priv_gnn}), while GCN requires truncating a significant fraction of neighborhood to control the noise level. On the other hand, for sparse graphs like the Finance dataset, the performance of the GIN aggregator deteriorates significantly (i.e., failing to match the non-private MLP baseline) due to the lower SNR in the underlying graph. In contrast, using a GCN aggregator equipped with truncated message passing allows finer noise level control, leading to much better results on sparse graphs. Meanwhile, the reduced noise level when applying GCN with truncated message passing demonstrates a significant advantage when the privacy requirements are more stringent. In particular, the performance of the GCN aggregator surpasses GIN in the high-privacy regime across all three datasets. On the Reddit dataset, GCN with truncated message passing achieves much better performance than MLP with $\epsilon = 2$, with GIN totally losing its performance at the same privacy level. 
        \item[Comparison against baselines] \vesper demonstrates a better privacy-utility trade-off compared to other private GNN baselines. The advantage over GAP is attributed to the end-to-end nature of the \vesper framework and better handling of message-passing mechanisms. The advantage over VFGNN is attributed to the tighter sensitivity control provided by the layer-wise perturbation strategy, as shown in the relative advantage of GAP to VFGNN. 
        \item[Comparison of decoders] The result shows that concatenation decoder performs better over \ogb and Reddit datasets, while the GRU decoder performs better over Alipay. 
    \end{description}

    \subsection{Protection against membership inference attacks}\label{sec: mia_full}
    We conduct membership inference attacks (MIAs) to empirically assess the resilience of our model against practical privacy risks. Note that although our method provides edge-privacy protection, we adopted the node MIA~\cite{olatunji2021membership} in this experiment due to two considerations. First, no generic and appropriate edge MIA is relevant to the GNN application in this paper. Therefore, the node MIA is a more realistic threat in our scenarios. Second, the node MIA can be considered as a strengthened variant of edge MIA where the adversary obtains extra node information. Therefore, the model with certain node-membership privacy will guarantee stronger edge-membership privacy.\par\noindent
    \textbf{Attack settings} Following ~\cite{olatunji2021membership,sajadmanesh2022gap}, we adopted the TFTS (train on subgraph, test
    on the full graph) setting of node MIA. Namely, the GNN model is trained on a subgraph, and the attack is reduced to a binary classification problem that distinguishes between nodes inside and outside the training subgraph. We consider an attacker with the following knowledge:
    \begin{itemize}[leftmargin=*]
        \item API access to the trained model, which returns a posterior distribution of node classes.
        \item A shadow dataset consists of 1000 nodes per class sampled randomly from the full graph.
        \item Architecture, hyperparameters of the target model.
    \end{itemize}
    For \ogb and Reddit datasets, we followed the attack procedures in~\cite{olatunji2021membership}. We first train a shadow model with the same architecture and hyperparameters as the target model using the shadow dataset. Then, we construct the attack training dataset by querying the shadow model. Finally, we train a 3-layer MLP as the attack model. For Finance dataset, since there are only two node classes, we adopted the entropy-based MIA as suggested in~\cite{song2021systematic} instead of shadow model training. Specifically, we computed the Shannon entropy of the node class distribution output by the target model. The nodes with smaller entropy (larger classification confidence) tend to be in the training subgraph.
    
    We respectively set the privacy budget $\epsilon\in\{1, 2, 4, 8, 16, 32, \infty\}$, where $\epsilon=\infty$ indicated no privacy protection is adopted. We use ROC-AUC (AUC) to evaluate the attack performance. 
    
    \par\noindent
    \textbf{Results} We report the attack performances in Figure~\ref{fig: mia_attack_auc}. From the results, we observe that when privacy protection is disabled ($\epsilon=\infty$), the attacks show non-negligible effectiveness, especially on obgn-products and Reddit datasets. Generally, with the privacy budget getting smaller (privacy getting stronger), the attack performances sharply decline. With an appropriate privacy budget, the attacks on all three datasets are successfully defended with AUC reduced to around 0.5 (random guess baseline). In conclusion, the above observations demonstrate that our method effectively mitigates the risks of privacy attacks with reasonable privacy budgets.
    
    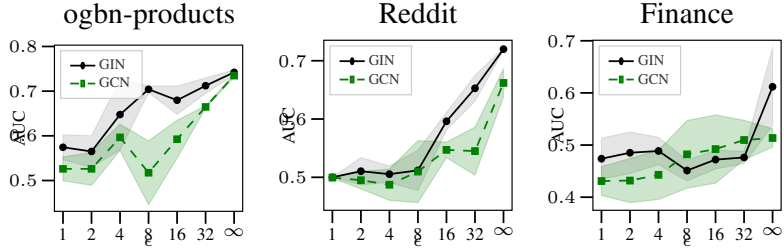
\begin{figure}
        \centering
        \begin{subfigure}{0.25\linewidth}
            \resizebox{\linewidth}{\linewidth}{
\begin{tikzpicture}

\definecolor{darkgray176}{RGB}{176,176,176}
\definecolor{gray}{RGB}{128,128,128}
\definecolor{green}{RGB}{0,128,0}
\definecolor{lightgray204}{RGB}{204,204,204}

\begin{axis}[
legend cell align={left},
legend style={
  fill opacity=0.8,
  draw opacity=1,
  text opacity=1,
  at={(0.03,0.97)},
  anchor=north west,
  draw=lightgray204
},
scale only axis,
width=2.25cm,
height=2cm,
tick align=outside,
tick pos=left,
title={ogbn-products},
title style={yshift=-4pt},
tick label style={font=\tiny},
xticklabel style={yshift=2pt},
yticklabel style={xshift=2pt},
legend style={nodes={scale=0.5, transform shape}},
every tick/.style={
black,
semithick,
},
x label style={at={(axis description cs:0.5,-0.1)},anchor=north,font=\tiny},
y label style={at={(axis description cs:-0.1,.5)},rotate=90,anchor=south,font=\tiny},
x grid style={darkgray176},
xlabel={\(\displaystyle \epsilon\)},
xmin=-0.3, xmax=6.3,
xtick={0,1,2,3,4,5,6},
xticklabels={1,2,4,8,16,32,\(\displaystyle \infty\)},
ylabel={AUC},
ymin=0.431012779067799, ymax=0.802907937653953,
]
\path [draw=gray, fill=gray, opacity=0.2]
(axis cs:0,0.602176660111905)
--(axis cs:0,0.546547380682288)
--(axis cs:1,0.529904540801537)
--(axis cs:2,0.568640844218208)
--(axis cs:3,0.696724663442356)
--(axis cs:4,0.648009936126391)
--(axis cs:5,0.694683482192038)
--(axis cs:6,0.736415792783582)
--(axis cs:6,0.747821794081856)
--(axis cs:6,0.747821794081856)
--(axis cs:5,0.729590082701992)
--(axis cs:4,0.711278402018966)
--(axis cs:3,0.711543666945812)
--(axis cs:2,0.726565760382228)
--(axis cs:1,0.600088459391054)
--(axis cs:0,0.602176660111905)
--cycle;

\path [draw=green, fill=green, opacity=0.2]
(axis cs:0,0.552772369227595)
--(axis cs:0,0.499316778766979)
--(axis cs:1,0.489280708217722)
--(axis cs:2,0.567128732272145)
--(axis cs:3,0.446098922639897)
--(axis cs:4,0.549850633072463)
--(axis cs:5,0.662020794552003)
--(axis cs:6,0.731115861210483)
--(axis cs:6,0.738301930710474)
--(axis cs:6,0.738301930710474)
--(axis cs:5,0.667302460975767)
--(axis cs:4,0.634960972955443)
--(axis cs:3,0.588732108182758)
--(axis cs:2,0.625956354108116)
--(axis cs:1,0.562867629855046)
--(axis cs:0,0.552772369227595)
--cycle;

\addplot [semithick, black, mark=*, mark size=1, mark options={solid}]
table {%
0 0.574362020397096
1 0.564996500096295
2 0.647603302300218
3 0.704134165194084
4 0.679644169072679
5 0.712136782447015
6 0.742118793432719
};
\addlegendentry{GIN}
\addplot [semithick, green, dashed, mark=square*, mark size=1, mark options={solid}]
table {%
0 0.526044573997287
1 0.526074169036384
2 0.596542543190131
3 0.517415515411327
4 0.592405803013953
5 0.664661627763885
6 0.734708895960479
};
\addlegendentry{GCN}
\end{axis}

\end{tikzpicture}}
        \end{subfigure}
        \begin{subfigure}{0.25\linewidth}
            \resizebox{\linewidth}{\linewidth}{
\begin{tikzpicture}

\definecolor{darkgray176}{RGB}{176,176,176}
\definecolor{gray}{RGB}{128,128,128}
\definecolor{green}{RGB}{0,128,0}
\definecolor{lightgray204}{RGB}{204,204,204}

\begin{axis}[
legend cell align={left},
legend style={
  fill opacity=0.8,
  draw opacity=1,
  text opacity=1,
  at={(0.03,0.97)},
  anchor=north west,
  draw=lightgray204
},
scale only axis,
width=2.25cm,
height=2cm,
tick align=outside,
tick pos=left,
title={Reddit},
title style={yshift=-4pt},
tick label style={font=\tiny},
xticklabel style={yshift=2pt},
yticklabel style={xshift=2pt},
legend style={nodes={scale=0.5, transform shape}},
every tick/.style={
black,
semithick,
},
x label style={at={(axis description cs:0.5,-0.1)},anchor=north,font=\tiny},
y label style={at={(axis description cs:-0.1,.5)},rotate=90,anchor=south,font=\tiny},
x grid style={darkgray176},
xlabel={\(\displaystyle \epsilon\)},
xmin=-0.3, xmax=6.3,
xtick={0,1,2,3,4,5,6},
xticklabels={1,2,4,8,16,32,\(\displaystyle \infty\)},
ylabel={AUC},
ymin=0.443367498325971, ymax=0.73515236932523,
]
\path [draw=gray, fill=gray, opacity=0.2]
(axis cs:0,0.500054802167273)
--(axis cs:0,0.499979067434759)
--(axis cs:1,0.487081764245714)
--(axis cs:2,0.491429487093707)
--(axis cs:3,0.477217057040684)
--(axis cs:4,0.581001404889022)
--(axis cs:5,0.630675341920215)
--(axis cs:6,0.717920330274454)
--(axis cs:6,0.721889420643446)
--(axis cs:6,0.721889420643446)
--(axis cs:5,0.674843008535527)
--(axis cs:4,0.611171206762723)
--(axis cs:3,0.54652932531443)
--(axis cs:2,0.519379303803085)
--(axis cs:1,0.533820380279406)
--(axis cs:0,0.500054802167273)
--cycle;

\path [draw=green, fill=green, opacity=0.2]
(axis cs:0,0.500497211681264)
--(axis cs:0,0.498622855484365)
--(axis cs:1,0.480031393230902)
--(axis cs:2,0.46046978823723)
--(axis cs:3,0.456630447007755)
--(axis cs:4,0.533704798268444)
--(axis cs:5,0.503816401633104)
--(axis cs:6,0.637916182764434)
--(axis cs:6,0.685824128299612)
--(axis cs:6,0.685824128299612)
--(axis cs:5,0.586101478204373)
--(axis cs:4,0.56041968614143)
--(axis cs:3,0.563076264034495)
--(axis cs:2,0.514048879835651)
--(axis cs:1,0.509900828848723)
--(axis cs:0,0.500497211681264)
--cycle;

\addplot [semithick, black, mark=*, mark size=1, mark options={solid}]
table {%
0 0.500016934801016
1 0.51045107226256
2 0.505404395448396
3 0.511873191177557
4 0.596086305825873
5 0.652759175227871
6 0.71990487545895
};
\addlegendentry{GIN}
\addplot [semithick, green, dashed, mark=square*, mark size=1, mark options={solid}]
table {%
0 0.499560033582815
1 0.494966111039813
2 0.487259334036441
3 0.509853355521125
4 0.547062242204937
5 0.544958939918739
6 0.661870155532023
};
\addlegendentry{GCN}
\end{axis}
\end{tikzpicture}}
        \end{subfigure}
        \begin{subfigure}{0.25\linewidth}
            \resizebox{\linewidth}{\linewidth}{
\begin{tikzpicture}

\definecolor{darkgray176}{RGB}{176,176,176}
\definecolor{gray}{RGB}{128,128,128}
\definecolor{green}{RGB}{0,128,0}
\definecolor{lightgray204}{RGB}{204,204,204}

\begin{axis}[
legend cell align={left},
legend style={
  fill opacity=0.8,
  draw opacity=1,
  text opacity=1,
  at={(0.03,0.97)},
  anchor=north west,
  draw=lightgray204
},
scale only axis,
width=2.25cm,
height=2cm,
tick align=outside,
tick pos=left,
title={Finance},
tick label style={font=\tiny},
title style={yshift=-4pt},
xticklabel style={yshift=2pt},
yticklabel style={xshift=2pt},
legend style={nodes={scale=0.5, transform shape}},
every tick/.style={
black,
semithick,
},
x label style={at={(axis description cs:0.5,-0.1)},anchor=north,font=\tiny},
y label style={at={(axis description cs:-0.1,.5)},rotate=90,anchor=south,font=\tiny},
x grid style={darkgray176},
xlabel={\(\displaystyle \epsilon\)},
xmin=-0.3, xmax=6.3,
xtick={0,1,2,3,4,5,6},
xticklabels={1,2,4,8,16,32,\(\displaystyle \infty\)},
ylabel={AUC},
ymin=0.374827492589487, ymax=0.701017795046696,
]
\path [draw=gray, fill=gray, opacity=0.2]
(axis cs:0,0.512993688378993)
--(axis cs:0,0.43453785840271)
--(axis cs:1,0.445376487222329)
--(axis cs:2,0.462456322672348)
--(axis cs:3,0.431289839114708)
--(axis cs:4,0.45424108049143)
--(axis cs:5,0.464402734058819)
--(axis cs:6,0.537435839612462)
--(axis cs:6,0.686190963116823)
--(axis cs:6,0.686190963116823)
--(axis cs:5,0.487933310940003)
--(axis cs:4,0.489979829821448)
--(axis cs:3,0.470972627133641)
--(axis cs:2,0.514818555059631)
--(axis cs:1,0.525107788836534)
--(axis cs:0,0.512993688378993)
--cycle;

\path [draw=green, fill=green, opacity=0.2]
(axis cs:0,0.458626776897333)
--(axis cs:0,0.403290966681947)
--(axis cs:1,0.38965432451936)
--(axis cs:2,0.395600106938558)
--(axis cs:3,0.417302759501524)
--(axis cs:4,0.426797114248234)
--(axis cs:5,0.471714698863776)
--(axis cs:6,0.495997459990456)
--(axis cs:6,0.531098581674719)
--(axis cs:6,0.531098581674719)
--(axis cs:5,0.547566118085103)
--(axis cs:4,0.557744845109438)
--(axis cs:3,0.54701172782659)
--(axis cs:2,0.490060128933953)
--(axis cs:1,0.474112207967592)
--(axis cs:0,0.458626776897333)
--cycle;

\addplot [semithick, black, mark=*, mark size=1, mark options={solid}]
table {%
0 0.473765773390851
1 0.485242138029432
2 0.488637438865989
3 0.451131233124175
4 0.472110455156439
5 0.476168022499411
6 0.611813401364642
};
\addlegendentry{GIN}
\addplot [semithick, green, dashed, mark=square*, mark size=1, mark options={solid}]
table {%
0 0.43095887178964
1 0.431883266243476
2 0.442830117936255
3 0.482157243664057
4 0.492270979678836
5 0.509640408474439
6 0.513548020832588
};
\addlegendentry{GCN}
\end{axis}
\end{tikzpicture}}
        \end{subfigure}
        \caption{AUC ($\text{mean}\pm\text{std}$ over $10$ trials) of membership inference attacks.}
        \label{fig: mia_attack_auc}
    \end{figure}
    
    \subsection{A complete report of ablation study}\label{sec: full_ablation}
    In this section, we investigate the effects of several critical hyperparameters in the \vesper framework via assessing their corresponding privacy-utility trade-off. All the results reported in this section will be based on a \vesper instantiation with GCN aggregator and concatenation decoder.\par\noindent
    \textbf{On the effect of maximum degree $D$} In GNN training with neighborhood sampling, a larger $D$ might retain more structural information of the underlying graphs, at the same time weakening the privacy amplification effect, resulting in a higher noise level. We evaluate the effect of $D$ under the range $\{20, 50\}$ for \ogb and Reddit dataset and $\{10, 20\}$ for the Finance dataset. We plot privacy-utility trade-off curves in figure \ref{fig: ablation_max_degree_gin_c}, \ref{fig: ablation_max_degree_gcn_c}, \ref{fig: ablation_max_degree_gin_g}, \ref{fig: ablation_max_degree_gcn_g}. The results imply the trade-off that larger $D$ may not always be beneficial in private GRL, especially for sparse graphs, where efficient control of noise level becomes more important than retaining structural information. \par\noindent
    \begin{figure}
        \centering
        \begin{subfigure}{0.32\linewidth}
            \resizebox{\linewidth}{\linewidth}{
\begin{tikzpicture}

\definecolor{crimson2143940}{RGB}{214,39,40}
\definecolor{darkgray176}{RGB}{176,176,176}
\definecolor{darkorange25512714}{RGB}{255,127,14}
\definecolor{forestgreen4416044}{RGB}{44,160,44}
\definecolor{lightgray204}{RGB}{204,204,204}
\definecolor{mediumpurple148103189}{RGB}{148,103,189}
\definecolor{steelblue31119180}{RGB}{31,119,180}

\begin{axis}[
legend cell align={left},
legend style={
  fill opacity=0.8,
  draw opacity=1,
  text opacity=1,
  at={(0.97,0.03)},
  anchor=south east,
  draw=lightgray204
},
scale only axis,
width=4cm,
height=3cm,
tick align=outside,
tick pos=left,
title={GIN-C, \ogb},
tick label style={font=\tiny},
legend style={nodes={scale=0.5, transform shape}},
every tick/.style={
black,
semithick,
},
x label style={at={(axis description cs:0.5,-0.1)},anchor=north,font=\tiny},
y label style={at={(axis description cs:-0.1,.5)},rotate=90,anchor=south,font=\tiny},
x grid style={darkgray176},
xlabel={$\epsilon$},
xmin=2.6, xmax=33.4,
xtick style={color=black},
y grid style={darkgray176},
ylabel={ACC},
ymin=0.682516214574389, ymax=0.771431717097416,
ytick style={color=black}
]
\path [fill=steelblue31119180, fill opacity=0.2]
(axis cs:4,0.7197549408467)
--(axis cs:4,0.70571278732173)
--(axis cs:8,0.740175850764731)
--(axis cs:16,0.746634995034394)
--(axis cs:32,0.749100452170779)
--(axis cs:32,0.767390103346369)
--(axis cs:32,0.767390103346369)
--(axis cs:16,0.757031415429523)
--(axis cs:8,0.749100369688924)
--(axis cs:4,0.7197549408467)
--cycle;

\path [fill=darkorange25512714, fill opacity=0.2]
(axis cs:4,0.702870622651049)
--(axis cs:4,0.686557828325436)
--(axis cs:8,0.724385586066143)
--(axis cs:16,0.738414236908568)
--(axis cs:32,0.750092685463124)
--(axis cs:32,0.763619493566116)
--(axis cs:32,0.763619493566116)
--(axis cs:16,0.752401007471351)
--(axis cs:8,0.739843313694418)
--(axis cs:4,0.702870622651049)
--cycle;

\addplot [semithick, steelblue31119180]
table {%
4 0.712733864084215
8 0.744638110226828
16 0.751833205231958
32 0.758245277758574
};
\addlegendentry{$D=20$}
\addplot [semithick, darkorange25512714]
table {%
4 0.694714225488243
8 0.732114449880281
16 0.745407622189959
32 0.75685608951462
};
\addlegendentry{$D=50$}
\end{axis}

\end{tikzpicture}}
        \end{subfigure}
        \begin{subfigure}{0.32\linewidth}
            \resizebox{\linewidth}{\linewidth}{
\begin{tikzpicture}

\definecolor{crimson2143940}{RGB}{214,39,40}
\definecolor{darkgray176}{RGB}{176,176,176}
\definecolor{darkorange25512714}{RGB}{255,127,14}
\definecolor{forestgreen4416044}{RGB}{44,160,44}
\definecolor{lightgray204}{RGB}{204,204,204}
\definecolor{mediumpurple148103189}{RGB}{148,103,189}
\definecolor{steelblue31119180}{RGB}{31,119,180}

\begin{axis}[
legend cell align={left},
legend style={
  fill opacity=0.8,
  draw opacity=1,
  text opacity=1,
  at={(0.97,0.03)},
  anchor=south east,
  draw=lightgray204
},
scale only axis,
width=4cm,
height=3cm,
tick align=outside,
tick pos=left,
title={GIN-C, Reddit},
tick label style={font=\tiny},
legend style={nodes={scale=0.5, transform shape}},
every tick/.style={
black,
semithick,
},
x label style={at={(axis description cs:0.5,-0.1)},anchor=north,font=\tiny},
y label style={at={(axis description cs:-0.1,.5)},rotate=90,anchor=south,font=\tiny},
x grid style={darkgray176},
xlabel={$\epsilon$},
xmin=2.6, xmax=33.4,
xtick style={color=black},
y grid style={darkgray176},
ylabel={ACC},
ymin=0.889446033417524, ymax=0.928857154503278,
ytick style={color=black}
]
\path [fill=steelblue31119180, fill opacity=0.2]
(axis cs:4,0.901205130589604)
--(axis cs:4,0.895829140455044)
--(axis cs:8,0.900666337896145)
--(axis cs:16,0.908276333396107)
--(axis cs:32,0.909589798806771)
--(axis cs:32,0.915916906397616)
--(axis cs:32,0.915916906397616)
--(axis cs:16,0.913521415378645)
--(axis cs:8,0.907627649860349)
--(axis cs:4,0.901205130589604)
--cycle;

\path [fill=darkorange25512714, fill opacity=0.2]
(axis cs:4,0.901179477467445)
--(axis cs:4,0.891237448012331)
--(axis cs:8,0.91111645676936)
--(axis cs:16,0.919163365854363)
--(axis cs:32,0.924299536647549)
--(axis cs:32,0.927065739908471)
--(axis cs:32,0.927065739908471)
--(axis cs:16,0.923085705111293)
--(axis cs:8,0.914465648323723)
--(axis cs:4,0.901179477467445)
--cycle;

\addplot [semithick, steelblue31119180]
table {%
4 0.898517135522324
8 0.904146993878247
16 0.910898874387376
32 0.912753352602194
};
\addlegendentry{$D=20$}
\addplot [semithick, darkorange25512714]
table {%
4 0.896208462739888
8 0.912791052546541
16 0.921124535482828
32 0.92568263827801
};
\addlegendentry{$D=50$}
\end{axis}

\end{tikzpicture}}
        \end{subfigure}
        \begin{subfigure}{0.32\linewidth}
            \resizebox{\linewidth}{\linewidth}{
\begin{tikzpicture}

\definecolor{crimson2143940}{RGB}{214,39,40}
\definecolor{darkgray176}{RGB}{176,176,176}
\definecolor{darkorange25512714}{RGB}{255,127,14}
\definecolor{forestgreen4416044}{RGB}{44,160,44}
\definecolor{lightgray204}{RGB}{204,204,204}
\definecolor{mediumpurple148103189}{RGB}{148,103,189}
\definecolor{steelblue31119180}{RGB}{31,119,180}

\begin{axis}[
legend cell align={left},
legend style={
  fill opacity=0.8,
  draw opacity=1,
  text opacity=1,
  at={(0.97,0.03)},
  anchor=south east,
  draw=lightgray204
},
scale only axis,
width=4cm,
height=3cm,
tick align=outside,
tick pos=left,
title={GCN-C, Finance},
tick label style={font=\tiny},
legend style={nodes={scale=0.5, transform shape}},
every tick/.style={
black,
semithick,
},
x label style={at={(axis description cs:0.5,-0.1)},anchor=north,font=\tiny},
y label style={at={(axis description cs:-0.1,.5)},rotate=90,anchor=south,font=\tiny},
x grid style={darkgray176},
xlabel={$\epsilon$},
xmin=2.6, xmax=33.4,
xtick style={color=black},
y grid style={darkgray176},
ylabel={AUC},
ymin=0.6493550423701, ymax=0.69772387903286,
ytick style={color=black}
]
\path [fill=steelblue31119180, fill opacity=0.2]
(axis cs:4,0.68900594354819)
--(axis cs:4,0.651553625854771)
--(axis cs:8,0.66307876137867)
--(axis cs:16,0.672186716154177)
--(axis cs:32,0.659643485982827)
--(axis cs:32,0.692224191983989)
--(axis cs:32,0.692224191983989)
--(axis cs:16,0.686880890039866)
--(axis cs:8,0.69425621541033)
--(axis cs:4,0.68900594354819)
--cycle;

\path [fill=darkorange25512714, fill opacity=0.2]
(axis cs:4,0.681123852510348)
--(axis cs:4,0.653452022625872)
--(axis cs:8,0.667051133124192)
--(axis cs:16,0.669652752762486)
--(axis cs:32,0.677301034943372)
--(axis cs:32,0.695525295548189)
--(axis cs:32,0.695525295548189)
--(axis cs:16,0.691976141862098)
--(axis cs:8,0.683044252696769)
--(axis cs:4,0.681123852510348)
--cycle;

\addplot [semithick, steelblue31119180]
table {%
4 0.67027978470148
8 0.6786674883945
16 0.679533803097021
32 0.675933838983408
};
\addlegendentry{$D=10$}
\addplot [semithick, darkorange25512714]
table {%
4 0.66728793756811
8 0.67504769291048
16 0.680814447312292
32 0.68641316524578
};
\addlegendentry{$D=20$}
\end{axis}

\end{tikzpicture}}
        \end{subfigure}
        \caption{Performance ($\text{mean}\pm\text{std}$ over $10$ trials) of \vesper under varying max degree $D$, using GIN aggregator and CONCAT decoder.}
        \label{fig: ablation_max_degree_gin_c}
    \end{figure}
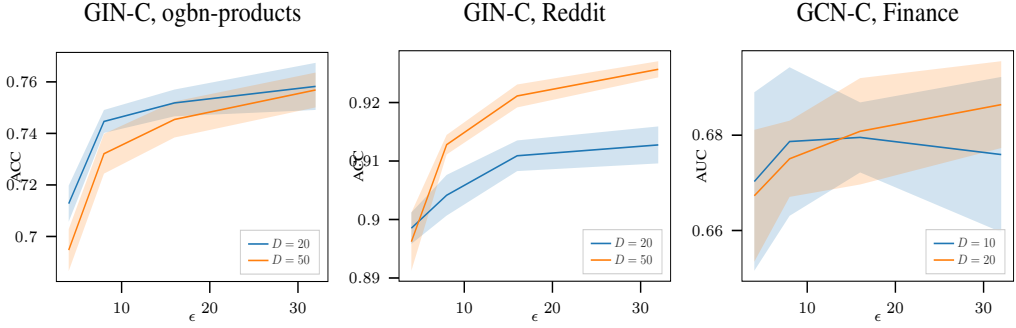
    
    \begin{figure}
        \centering
        \begin{subfigure}{0.32\linewidth}
            \resizebox{\linewidth}{\linewidth}{
\begin{tikzpicture}

\definecolor{crimson2143940}{RGB}{214,39,40}
\definecolor{darkgray176}{RGB}{176,176,176}
\definecolor{darkorange25512714}{RGB}{255,127,14}
\definecolor{forestgreen4416044}{RGB}{44,160,44}
\definecolor{lightgray204}{RGB}{204,204,204}
\definecolor{mediumpurple148103189}{RGB}{148,103,189}
\definecolor{steelblue31119180}{RGB}{31,119,180}

\begin{axis}[
legend cell align={left},
legend style={
  fill opacity=0.8,
  draw opacity=1,
  text opacity=1,
  at={(0.97,0.03)},
  anchor=south east,
  draw=lightgray204
},
scale only axis,
width=4cm,
height=3cm,
tick align=outside,
tick pos=left,
title={GCN-C, \ogb},
tick label style={font=\tiny},
legend style={nodes={scale=0.5, transform shape}},
every tick/.style={
black,
semithick,
},
x label style={at={(axis description cs:0.5,-0.1)},anchor=north,font=\tiny},
y label style={at={(axis description cs:-0.1,.5)},rotate=90,anchor=south,font=\tiny},
x grid style={darkgray176},
xlabel={$\epsilon$},
xmin=2.6, xmax=33.4,
xtick style={color=black},
y grid style={darkgray176},
ylabel={ACC},
ymin=0.66081423309623, ymax=0.712373938231697,
ytick style={color=black}
]
\path [fill=steelblue31119180, fill opacity=0.2]
(axis cs:4,0.671236894707091)
--(axis cs:4,0.663300004137105)
--(axis cs:8,0.669329899195047)
--(axis cs:16,0.69588391704559)
--(axis cs:32,0.700655316445865)
--(axis cs:32,0.708866479088074)
--(axis cs:32,0.708866479088074)
--(axis cs:16,0.706809600753723)
--(axis cs:8,0.679990937589342)
--(axis cs:4,0.671236894707091)
--cycle;

\path [fill=darkorange25512714, fill opacity=0.2]
(axis cs:4,0.671150584038842)
--(axis cs:4,0.663157856056933)
--(axis cs:8,0.667405408960443)
--(axis cs:16,0.693967663977903)
--(axis cs:32,0.693855561992977)
--(axis cs:32,0.710030315270994)
--(axis cs:32,0.710030315270994)
--(axis cs:16,0.707777587347054)
--(axis cs:8,0.683173306510362)
--(axis cs:4,0.671150584038842)
--cycle;

\addplot [semithick, steelblue31119180]
table {%
4 0.667268449422098
8 0.674660418392194
16 0.701346758899657
32 0.704760897766969
};
\addlegendentry{$D=20$}
\addplot [semithick, darkorange25512714]
table {%
4 0.667154220047888
8 0.675289357735403
16 0.700872625662478
32 0.701942938631986
};
\addlegendentry{$D=50$}
\end{axis}

\end{tikzpicture}}
        \end{subfigure}
        \begin{subfigure}{0.32\linewidth}
            \resizebox{\linewidth}{\linewidth}{
\begin{tikzpicture}

\definecolor{crimson2143940}{RGB}{214,39,40}
\definecolor{darkgray176}{RGB}{176,176,176}
\definecolor{darkorange25512714}{RGB}{255,127,14}
\definecolor{forestgreen4416044}{RGB}{44,160,44}
\definecolor{lightgray204}{RGB}{204,204,204}
\definecolor{mediumpurple148103189}{RGB}{148,103,189}
\definecolor{steelblue31119180}{RGB}{31,119,180}

\begin{axis}[
legend cell align={left},
legend style={
  fill opacity=0.8,
  draw opacity=1,
  text opacity=1,
  at={(0.97,0.03)},
  anchor=south east,
  draw=lightgray204
},
scale only axis,
width=4cm,
height=3cm,
tick align=outside,
tick pos=left,
title={GCN-C, Reddit},
tick label style={font=\tiny},
legend style={nodes={scale=0.5, transform shape}},
every tick/.style={
black,
semithick,
},
x label style={at={(axis description cs:0.5,-0.1)},anchor=north,font=\tiny},
y label style={at={(axis description cs:-0.1,.5)},rotate=90,anchor=south,font=\tiny},
x grid style={darkgray176},
xlabel={$\epsilon$},
xmin=2.6, xmax=33.4,
xtick style={color=black},
y grid style={darkgray176},
ylabel={ACC},
ymin=0.889446033417524, ymax=0.928857154503278,
ytick style={color=black}
]
\path [fill=steelblue31119180, fill opacity=0.2]
(axis cs:4,0.901205130589604)
--(axis cs:4,0.895829140455044)
--(axis cs:8,0.900666337896145)
--(axis cs:16,0.908276333396107)
--(axis cs:32,0.909589798806771)
--(axis cs:32,0.915916906397616)
--(axis cs:32,0.915916906397616)
--(axis cs:16,0.913521415378645)
--(axis cs:8,0.907627649860349)
--(axis cs:4,0.901205130589604)
--cycle;

\path [fill=darkorange25512714, fill opacity=0.2]
(axis cs:4,0.901179477467445)
--(axis cs:4,0.891237448012331)
--(axis cs:8,0.91111645676936)
--(axis cs:16,0.919163365854363)
--(axis cs:32,0.924299536647549)
--(axis cs:32,0.927065739908471)
--(axis cs:32,0.927065739908471)
--(axis cs:16,0.923085705111293)
--(axis cs:8,0.914465648323723)
--(axis cs:4,0.901179477467445)
--cycle;

\addplot [semithick, steelblue31119180]
table {%
4 0.898517135522324
8 0.904146993878247
16 0.910898874387376
32 0.912753352602194
};
\addlegendentry{$D=20$}
\addplot [semithick, darkorange25512714]
table {%
4 0.896208462739888
8 0.912791052546541
16 0.921124535482828
32 0.92568263827801
};
\addlegendentry{$D=50$}
\end{axis}

\end{tikzpicture}}
        \end{subfigure}
        \begin{subfigure}{0.32\linewidth}
            \resizebox{\linewidth}{\linewidth}{
\begin{tikzpicture}

\definecolor{crimson2143940}{RGB}{214,39,40}
\definecolor{darkgray176}{RGB}{176,176,176}
\definecolor{darkorange25512714}{RGB}{255,127,14}
\definecolor{forestgreen4416044}{RGB}{44,160,44}
\definecolor{lightgray204}{RGB}{204,204,204}
\definecolor{mediumpurple148103189}{RGB}{148,103,189}
\definecolor{steelblue31119180}{RGB}{31,119,180}

\begin{axis}[
legend cell align={left},
legend style={
  fill opacity=0.8,
  draw opacity=1,
  text opacity=1,
  at={(0.97,0.03)},
  anchor=south east,
  draw=lightgray204
},
scale only axis,
width=4cm,
height=3cm,
tick align=outside,
tick pos=left,
title={GCN-C, Finance},
tick label style={font=\tiny},
legend style={nodes={scale=0.5, transform shape}},
every tick/.style={
black,
semithick,
},
x label style={at={(axis description cs:0.5,-0.1)},anchor=north,font=\tiny},
y label style={at={(axis description cs:-0.1,.5)},rotate=90,anchor=south,font=\tiny},
x grid style={darkgray176},
xlabel={$\epsilon$},
xmin=2.6, xmax=33.4,
xtick style={color=black},
y grid style={darkgray176},
ylabel={AUC},
ymin=0.692762525649503, ymax=0.740592882254496,
ytick style={color=black}
]
\path [fill=steelblue31119180, fill opacity=0.2]
(axis cs:4,0.713308700130452)
--(axis cs:4,0.695578977852054)
--(axis cs:8,0.704781784510618)
--(axis cs:16,0.719910308412774)
--(axis cs:32,0.71606570449207)
--(axis cs:32,0.738418775136088)
--(axis cs:32,0.738418775136088)
--(axis cs:16,0.731320865663228)
--(axis cs:8,0.723067670314412)
--(axis cs:4,0.713308700130452)
--cycle;

\path [fill=darkorange25512714, fill opacity=0.2]
(axis cs:4,0.709416774789518)
--(axis cs:4,0.694936632767912)
--(axis cs:8,0.705384348211806)
--(axis cs:16,0.712760467484455)
--(axis cs:32,0.71376828584445)
--(axis cs:32,0.730895476197186)
--(axis cs:32,0.730895476197186)
--(axis cs:16,0.724262924768823)
--(axis cs:8,0.717196366459426)
--(axis cs:4,0.709416774789518)
--cycle;

\addplot [semithick, steelblue31119180]
table {%
4 0.704443838991253
8 0.713924727412515
16 0.725615587038001
32 0.727242239814079
};
\addlegendentry{$D=10$}
\addplot [semithick, darkorange25512714]
table {%
4 0.702176703778715
8 0.711290357335616
16 0.718511696126639
32 0.722331881020818
};
\addlegendentry{$D=20$}
\end{axis}

\end{tikzpicture}}
        \end{subfigure}
        \caption{Performance ($\text{mean}\pm\text{std}$ over $10$ trials) of \vesper under varying max degree $D$, using GCN aggregator and CONCAT decoder.}
        \label{fig: ablation_max_degree_gcn_c}
    \end{figure}
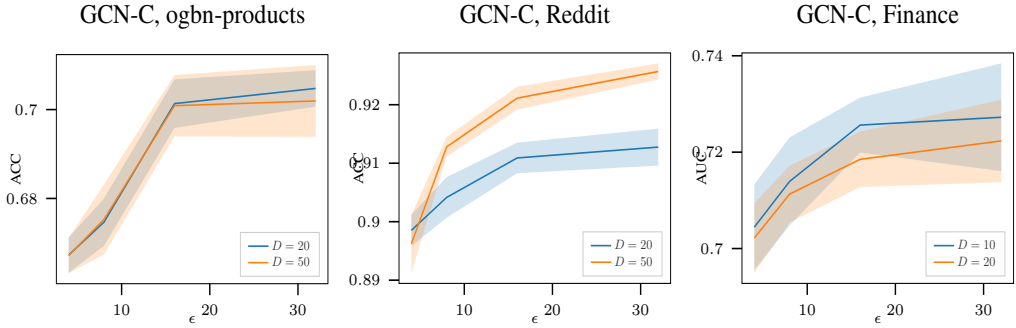
    
    \begin{figure}
        \centering
        \begin{subfigure}{0.32\linewidth}
            \resizebox{\linewidth}{\linewidth}{
\begin{tikzpicture}

\definecolor{crimson2143940}{RGB}{214,39,40}
\definecolor{darkgray176}{RGB}{176,176,176}
\definecolor{darkorange25512714}{RGB}{255,127,14}
\definecolor{forestgreen4416044}{RGB}{44,160,44}
\definecolor{lightgray204}{RGB}{204,204,204}
\definecolor{mediumpurple148103189}{RGB}{148,103,189}
\definecolor{steelblue31119180}{RGB}{31,119,180}

\begin{axis}[
legend cell align={left},
legend style={
  fill opacity=0.8,
  draw opacity=1,
  text opacity=1,
  at={(0.97,0.03)},
  anchor=south east,
  draw=lightgray204
},
scale only axis,
width=4cm,
height=3cm,
tick align=outside,
tick pos=left,
title={GCN-G, \ogb},
tick label style={font=\tiny},
legend style={nodes={scale=0.5, transform shape}},
every tick/.style={
black,
semithick,
},
x label style={at={(axis description cs:0.5,-0.1)},anchor=north,font=\tiny},
y label style={at={(axis description cs:-0.1,.5)},rotate=90,anchor=south,font=\tiny},
x grid style={darkgray176},
xlabel={$\epsilon$},
xmin=2.6, xmax=33.4,
xtick style={color=black},
y grid style={darkgray176},
ylabel={ACC},
ymin=0.656141676674975, ymax=0.712101918542113,
ytick style={color=black}
]
\path [fill=steelblue31119180, fill opacity=0.2]
(axis cs:4,0.680000881786496)
--(axis cs:4,0.672056489555216)
--(axis cs:8,0.680185788219953)
--(axis cs:16,0.695438475566557)
--(axis cs:32,0.70287122222538)
--(axis cs:32,0.709558271184516)
--(axis cs:32,0.709558271184516)
--(axis cs:16,0.704988438645285)
--(axis cs:8,0.693412676551717)
--(axis cs:4,0.680000881786496)
--cycle;

\path [fill=darkorange25512714, fill opacity=0.2]
(axis cs:4,0.66684372109029)
--(axis cs:4,0.658685324032572)
--(axis cs:8,0.679745515186942)
--(axis cs:16,0.686247378727248)
--(axis cs:32,0.695786921328228)
--(axis cs:32,0.703955249237735)
--(axis cs:32,0.703955249237735)
--(axis cs:16,0.699198588022425)
--(axis cs:8,0.686084990653079)
--(axis cs:4,0.66684372109029)
--cycle;

\addplot [semithick, steelblue31119180]
table {%
4 0.676028685670856
8 0.686799232385835
16 0.700213457105921
32 0.706214746704948
};
\addlegendentry{$D=20$}
\addplot [semithick, darkorange25512714]
table {%
4 0.662764522561431
8 0.682915252920011
16 0.692722983374836
32 0.699871085282982
};
\addlegendentry{$D=50$}
\end{axis}

\end{tikzpicture}}
        \end{subfigure}
        \begin{subfigure}{0.32\linewidth}
            \resizebox{\linewidth}{\linewidth}{
\begin{tikzpicture}

\definecolor{crimson2143940}{RGB}{214,39,40}
\definecolor{darkgray176}{RGB}{176,176,176}
\definecolor{darkorange25512714}{RGB}{255,127,14}
\definecolor{forestgreen4416044}{RGB}{44,160,44}
\definecolor{lightgray204}{RGB}{204,204,204}
\definecolor{mediumpurple148103189}{RGB}{148,103,189}
\definecolor{steelblue31119180}{RGB}{31,119,180}

\begin{axis}[
legend cell align={left},
legend style={
  fill opacity=0.8,
  draw opacity=1,
  text opacity=1,
  at={(0.97,0.03)},
  anchor=south east,
  draw=lightgray204
},
scale only axis,
width=4cm,
height=3cm,
tick align=outside,
tick pos=left,
title={GCN-G, Reddit},
tick label style={font=\tiny},
legend style={nodes={scale=0.5, transform shape}},
every tick/.style={
black,
semithick,
},
x label style={at={(axis description cs:0.5,-0.1)},anchor=north,font=\tiny},
y label style={at={(axis description cs:-0.1,.5)},rotate=90,anchor=south,font=\tiny},
x grid style={darkgray176},
xlabel={$\epsilon$},
xmin=2.6, xmax=33.4,
xtick style={color=black},
y grid style={darkgray176},
ylabel={ACC},
ymin=0.888654526730083, ymax=0.927778206239199,
ytick style={color=black}
]
\path [fill=steelblue31119180, fill opacity=0.2]
(axis cs:4,0.899409682052782)
--(axis cs:4,0.89043287579868)
--(axis cs:8,0.900697515891767)
--(axis cs:16,0.907371961242697)
--(axis cs:32,0.9106158506824)
--(axis cs:32,0.917397003203388)
--(axis cs:32,0.917397003203388)
--(axis cs:16,0.913585617343735)
--(axis cs:8,0.907837033414374)
--(axis cs:4,0.899409682052782)
--cycle;

\path [fill=darkorange25512714, fill opacity=0.2]
(axis cs:4,0.896485500605356)
--(axis cs:4,0.894348027211817)
--(axis cs:8,0.907118899049947)
--(axis cs:16,0.91272935596653)
--(axis cs:32,0.922284795361577)
--(axis cs:32,0.925999857170603)
--(axis cs:32,0.925999857170603)
--(axis cs:16,0.916723276746251)
--(axis cs:8,0.913249842310482)
--(axis cs:4,0.896485500605356)
--cycle;

\addplot [semithick, steelblue31119180]
table {%
4 0.894921278925731
8 0.904267274653071
16 0.910478789293216
32 0.914006426942894
};
\addlegendentry{$D=20$}
\addplot [semithick, darkorange25512714]
table {%
4 0.895416763908586
8 0.910184370680215
16 0.91472631635639
32 0.92414232626609
};
\addlegendentry{$D=50$}
\end{axis}

\end{tikzpicture}}
        \end{subfigure}
        \begin{subfigure}{0.32\linewidth}
            \resizebox{\linewidth}{\linewidth}{
\begin{tikzpicture}

\definecolor{crimson2143940}{RGB}{214,39,40}
\definecolor{darkgray176}{RGB}{176,176,176}
\definecolor{darkorange25512714}{RGB}{255,127,14}
\definecolor{forestgreen4416044}{RGB}{44,160,44}
\definecolor{lightgray204}{RGB}{204,204,204}
\definecolor{mediumpurple148103189}{RGB}{148,103,189}
\definecolor{steelblue31119180}{RGB}{31,119,180}

\begin{axis}[
legend cell align={left},
legend style={
  fill opacity=0.8,
  draw opacity=1,
  text opacity=1,
  at={(0.97,0.03)},
  anchor=south east,
  draw=lightgray204
},
scale only axis,
width=4cm,
height=3cm,
tick align=outside,
tick pos=left,
title={GCN-G, Finance},
tick label style={font=\tiny},
legend style={nodes={scale=0.5, transform shape}},
every tick/.style={
black,
semithick,
},
x label style={at={(axis description cs:0.5,-0.1)},anchor=north,font=\tiny},
y label style={at={(axis description cs:-0.1,.5)},rotate=90,anchor=south,font=\tiny},
x grid style={darkgray176},
xlabel={$\epsilon$},
xmin=2.6, xmax=33.4,
xtick style={color=black},
y grid style={darkgray176},
ylabel={AUC},
ymin=0.720746291810819, ymax=0.752691956320407,
ytick style={color=black}
]
\path [fill=steelblue31119180, fill opacity=0.2]
(axis cs:4,0.743018694014603)
--(axis cs:4,0.728438173641977)
--(axis cs:8,0.731523147425954)
--(axis cs:16,0.733753921838726)
--(axis cs:32,0.739272738149152)
--(axis cs:32,0.751239880660881)
--(axis cs:32,0.751239880660881)
--(axis cs:16,0.744160500084317)
--(axis cs:8,0.747717439163771)
--(axis cs:4,0.743018694014603)
--cycle;

\path [fill=darkorange25512714, fill opacity=0.2]
(axis cs:4,0.746460346895058)
--(axis cs:4,0.723346560232252)
--(axis cs:8,0.732845564766982)
--(axis cs:16,0.722198367470346)
--(axis cs:32,0.742662609660882)
--(axis cs:32,0.749757925331492)
--(axis cs:32,0.749757925331492)
--(axis cs:16,0.745055687786583)
--(axis cs:8,0.744364055100684)
--(axis cs:4,0.746460346895058)
--cycle;

\addplot [semithick, steelblue31119180]
table {%
4 0.73572843382829
8 0.739620293294862
16 0.738957210961521
32 0.745256309405016
};
\addlegendentry{$D=10$}
\addplot [semithick, darkorange25512714]
table {%
4 0.734903453563655
8 0.738604809933833
16 0.733627027628464
32 0.746210267496187
};
\addlegendentry{$D=20$}
\end{axis}

\end{tikzpicture}}
        \end{subfigure}
        \caption{Performance ($\text{mean}\pm\text{std}$ over $10$ trials) of \vesper under varying max degree $D$, using GCN aggregator and GRU decoder.}
        \label{fig: ablation_max_degree_gcn_g}
    \end{figure}
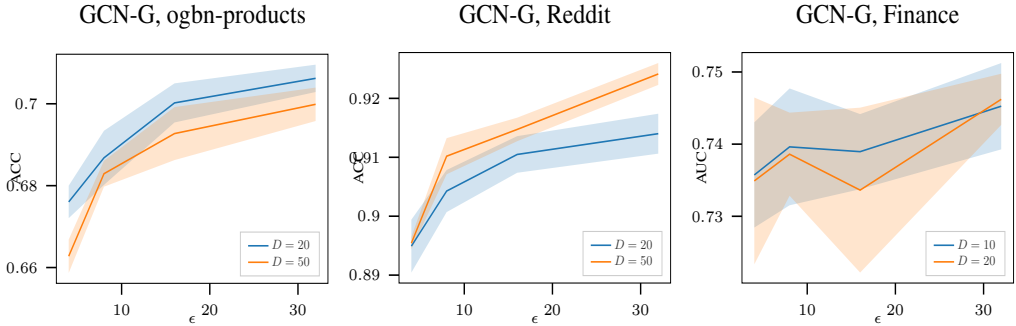

    \begin{figure}
        \centering
        \begin{subfigure}{0.32\linewidth}
            \resizebox{\linewidth}{\linewidth}{
\begin{tikzpicture}

\definecolor{crimson2143940}{RGB}{214,39,40}
\definecolor{darkgray176}{RGB}{176,176,176}
\definecolor{darkorange25512714}{RGB}{255,127,14}
\definecolor{forestgreen4416044}{RGB}{44,160,44}
\definecolor{lightgray204}{RGB}{204,204,204}
\definecolor{mediumpurple148103189}{RGB}{148,103,189}
\definecolor{steelblue31119180}{RGB}{31,119,180}

\begin{axis}[
legend cell align={left},
legend style={
  fill opacity=0.8,
  draw opacity=1,
  text opacity=1,
  at={(0.97,0.03)},
  anchor=south east,
  draw=lightgray204
},
scale only axis,
width=4cm,
height=3cm,
tick align=outside,
tick pos=left,
title={GIN-G, \ogb},
tick label style={font=\tiny},
legend style={nodes={scale=0.5, transform shape}},
every tick/.style={
black,
semithick,
},
x label style={at={(axis description cs:0.5,-0.1)},anchor=north,font=\tiny},
y label style={at={(axis description cs:-0.1,.5)},rotate=90,anchor=south,font=\tiny},
x grid style={darkgray176},
xlabel={$\epsilon$},
xmin=2.6, xmax=33.4,
xtick style={color=black},
y grid style={darkgray176},
ylabel={ACC},
ymin=0.674810962336219, ymax=0.765264439728598,
ytick style={color=black}
]
\path [fill=steelblue31119180, fill opacity=0.2]
(axis cs:4,0.701582728714592)
--(axis cs:4,0.689377606852268)
--(axis cs:8,0.729902731378523)
--(axis cs:16,0.74113015154025)
--(axis cs:32,0.749355694636969)
--(axis cs:32,0.76094532777919)
--(axis cs:32,0.76094532777919)
--(axis cs:16,0.755188074867973)
--(axis cs:8,0.737856615164389)
--(axis cs:4,0.701582728714592)
--cycle;

\path [fill=darkorange25512714, fill opacity=0.2]
(axis cs:4,0.697725651987455)
--(axis cs:4,0.678922484035872)
--(axis cs:8,0.73147129051379)
--(axis cs:16,0.748734374516823)
--(axis cs:32,0.751006771744319)
--(axis cs:32,0.761152918028944)
--(axis cs:32,0.761152918028944)
--(axis cs:16,0.758467182039143)
--(axis cs:8,0.737453123348993)
--(axis cs:4,0.697725651987455)
--cycle;

\addplot [semithick, steelblue31119180]
table {%
4 0.69548016778343
8 0.733879673271456
16 0.748159113204111
32 0.75515051120808
};
\addlegendentry{$D=20$}
\addplot [semithick, darkorange25512714]
table {%
4 0.688324068011663
8 0.734462206931392
16 0.753600778277983
32 0.756079844886631
};
\addlegendentry{$D=50$}
\end{axis}

\end{tikzpicture}}
        \end{subfigure}
        \begin{subfigure}{0.32\linewidth}
            \resizebox{\linewidth}{\linewidth}{
\begin{tikzpicture}

\definecolor{crimson2143940}{RGB}{214,39,40}
\definecolor{darkgray176}{RGB}{176,176,176}
\definecolor{darkorange25512714}{RGB}{255,127,14}
\definecolor{forestgreen4416044}{RGB}{44,160,44}
\definecolor{lightgray204}{RGB}{204,204,204}
\definecolor{mediumpurple148103189}{RGB}{148,103,189}
\definecolor{steelblue31119180}{RGB}{31,119,180}

\begin{axis}[
legend cell align={left},
legend style={
  fill opacity=0.8,
  draw opacity=1,
  text opacity=1,
  at={(0.97,0.03)},
  anchor=south east,
  draw=lightgray204
},
scale only axis,
width=4cm,
height=3cm,
tick align=outside,
tick pos=left,
title={GIN-G, Reddit},
tick label style={font=\tiny},
legend style={nodes={scale=0.5, transform shape}},
every tick/.style={
black,
semithick,
},
x label style={at={(axis description cs:0.5,-0.1)},anchor=north,font=\tiny},
y label style={at={(axis description cs:-0.1,.5)},rotate=90,anchor=south,font=\tiny},
x grid style={darkgray176},
xlabel={$\epsilon$},
xmin=2.6, xmax=33.4,
xtick style={color=black},
y grid style={darkgray176},
ylabel={ACC},
ymin=0.677324365239341, ymax=0.945719932775256,
ytick style={color=black}
]
\path [fill=steelblue31119180, fill opacity=0.2]
(axis cs:4,0.780824927068668)
--(axis cs:4,0.728691616025959)
--(axis cs:8,0.866424149971383)
--(axis cs:16,0.900853965427953)
--(axis cs:32,0.913622026167442)
--(axis cs:32,0.916444577067572)
--(axis cs:32,0.916444577067572)
--(axis cs:16,0.908879801155534)
--(axis cs:8,0.883445695458845)
--(axis cs:4,0.780824927068668)
--cycle;

\path [fill=darkorange25512714, fill opacity=0.2]
(axis cs:4,0.753252706422824)
--(axis cs:4,0.689524163763701)
--(axis cs:8,0.89155721174369)
--(axis cs:16,0.915339149918999)
--(axis cs:32,0.925424626354457)
--(axis cs:32,0.933520134250897)
--(axis cs:32,0.933520134250897)
--(axis cs:16,0.924910028760785)
--(axis cs:8,0.900339095457)
--(axis cs:4,0.753252706422824)
--cycle;

\addplot [semithick, steelblue31119180]
table {%
4 0.754758271547313
8 0.874934922715114
16 0.904866883291744
32 0.915033301617507
};
\addlegendentry{$D=20$}
\addplot [semithick, darkorange25512714]
table {%
4 0.721388435093262
8 0.895948153600345
16 0.920124589339892
32 0.929472380302677
};
\addlegendentry{$D=50$}
\end{axis}

\end{tikzpicture}}
        \end{subfigure}
        \begin{subfigure}{0.32\linewidth}
            \resizebox{\linewidth}{\linewidth}{
\begin{tikzpicture}

\definecolor{crimson2143940}{RGB}{214,39,40}
\definecolor{darkgray176}{RGB}{176,176,176}
\definecolor{darkorange25512714}{RGB}{255,127,14}
\definecolor{forestgreen4416044}{RGB}{44,160,44}
\definecolor{lightgray204}{RGB}{204,204,204}
\definecolor{mediumpurple148103189}{RGB}{148,103,189}
\definecolor{steelblue31119180}{RGB}{31,119,180}

\begin{axis}[
legend cell align={left},
legend style={
  fill opacity=0.8,
  draw opacity=1,
  text opacity=1,
  at={(0.97,0.03)},
  anchor=south east,
  draw=lightgray204
},
scale only axis,
width=4cm,
height=3cm,
tick align=outside,
tick pos=left,
title={GIN-G, Finance},
tick label style={font=\tiny},
legend style={nodes={scale=0.5, transform shape}},
every tick/.style={
black,
semithick,
},
x label style={at={(axis description cs:0.5,-0.1)},anchor=north,font=\tiny},
y label style={at={(axis description cs:-0.1,.5)},rotate=90,anchor=south,font=\tiny},
x grid style={darkgray176},
xlabel={$\epsilon$},
xmin=2.6, xmax=33.4,
xtick style={color=black},
y grid style={darkgray176},
ylabel={AUC},
ymin=0.457487626927219, ymax=0.715391336111662,
ytick style={color=black}
]
\path [fill=steelblue31119180, fill opacity=0.2]
(axis cs:4,0.656195335189903)
--(axis cs:4,0.521713709845578)
--(axis cs:8,0.621752060479667)
--(axis cs:16,0.643608619994545)
--(axis cs:32,0.64952683682845)
--(axis cs:32,0.692497253428399)
--(axis cs:32,0.692497253428399)
--(axis cs:16,0.677473112021095)
--(axis cs:8,0.666962099922711)
--(axis cs:4,0.656195335189903)
--cycle;

\path [fill=darkorange25512714, fill opacity=0.2]
(axis cs:4,0.612215384840749)
--(axis cs:4,0.469210522799239)
--(axis cs:8,0.518206650474516)
--(axis cs:16,0.554684526848958)
--(axis cs:32,0.645561048267037)
--(axis cs:32,0.685931166256786)
--(axis cs:32,0.685931166256786)
--(axis cs:16,0.703668440239641)
--(axis cs:8,0.656412398993672)
--(axis cs:4,0.612215384840749)
--cycle;

\addplot [semithick, steelblue31119180]
table {%
4 0.588954522517741
8 0.644357080201189
16 0.66054086600782
32 0.671012045128424
};
\addlegendentry{$D=10$}
\addplot [semithick, darkorange25512714]
table {%
4 0.540712953819994
8 0.587309524734094
16 0.6291764835443
32 0.665746107261912
};
\addlegendentry{$D=20$}
\end{axis}

\end{tikzpicture}}
        \end{subfigure}
        \caption{Performance ($\text{mean}\pm\text{std}$ over $10$ trials) of \vesper under varying max degree $D$, using GIN aggregator and GRU decoder.}
        \label{fig: ablation_max_degree_gin_g}
    \end{figure}
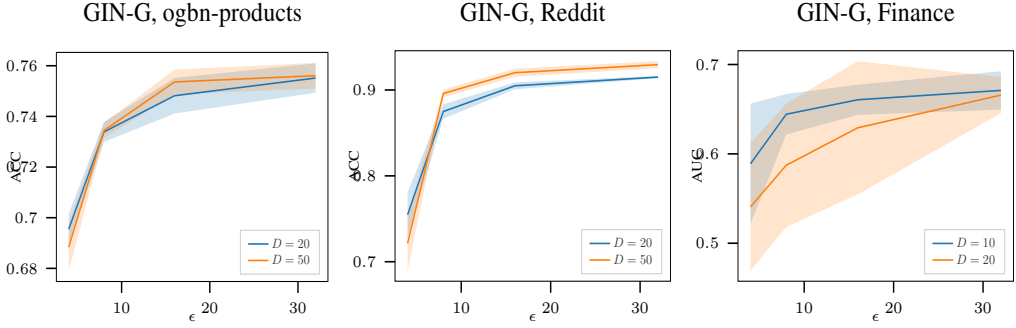

    \textbf{On the effect of minimum degree $\mind$ for PMP-GCN} As discussed in section \ref{sec: pmp_gcn}, the $\mind$ parameter trades off the amount of structural information involved during message passing and the noise scale during perturbation. We evaluate the effect of $\mind$ under the range $\{10, 20, 40\}$ for \ogb and Reddit datasets and $\{3, 5\}$ for the Finance dataset and plot the resulting curves in figure \ref{fig: ablation_min_degree_gcn_c}, \ref{fig: ablation_min_degree_gcn_g}. We observe two interesting phenomena. First, adopting larger $\mind$ makes the noise scale less sensitive with respect to privacy constraints, resulting in a flatter privacy-utility curve. On the two dense graph datasets, this shows the potential benefits of using a larger $\mind$ when the required privacy level is more stringent. Second, an \emph{crossing effect} is observed on \ogb and Reddit datasets, suggesting that a lower $\mind$ is beneficial in the low-privacy regime, where the noise scale is efficiently controlled and the incorporation of more structural information becomes effective. \par\noindent
    
    \begin{figure}
        \centering
        \begin{subfigure}{0.32\linewidth}
            \resizebox{\linewidth}{\linewidth}{
\begin{tikzpicture}

\definecolor{crimson2143940}{RGB}{214,39,40}
\definecolor{darkgray176}{RGB}{176,176,176}
\definecolor{darkorange25512714}{RGB}{255,127,14}
\definecolor{forestgreen4416044}{RGB}{44,160,44}
\definecolor{lightgray204}{RGB}{204,204,204}
\definecolor{mediumpurple148103189}{RGB}{148,103,189}
\definecolor{steelblue31119180}{RGB}{31,119,180}

\begin{axis}[
legend cell align={left},
legend style={
  fill opacity=0.8,
  draw opacity=1,
  text opacity=1,
  at={(0.97,0.03)},
  anchor=south east,
  draw=lightgray204
},
scale only axis,
width=4cm,
height=3cm,
tick align=outside,
tick pos=left,
title={GCN-C, \ogb},
tick label style={font=\tiny},
legend style={nodes={scale=0.5, transform shape}},
every tick/.style={
black,
semithick,
},
x label style={at={(axis description cs:0.5,-0.1)},anchor=north,font=\tiny},
y label style={at={(axis description cs:-0.1,.5)},rotate=90,anchor=south,font=\tiny},
x grid style={darkgray176},
xlabel={$\epsilon$},
xmin=2.6, xmax=33.4,
xtick style={color=black},
y grid style={darkgray176},
ylabel={ACC},
ymin=0.459247501575414, ymax=0.720753097064867,
ytick style={color=black}
]
\path [fill=steelblue31119180, fill opacity=0.2]
(axis cs:4,0.495488400714695)
--(axis cs:4,0.471134119552208)
--(axis cs:8,0.573825218559298)
--(axis cs:16,0.69588391704559)
--(axis cs:32,0.700655316445865)
--(axis cs:32,0.708866479088074)
--(axis cs:32,0.708866479088074)
--(axis cs:16,0.706809600753723)
--(axis cs:8,0.5919483533363)
--(axis cs:4,0.495488400714695)
--cycle;

\path [fill=darkorange25512714, fill opacity=0.2]
(axis cs:4,0.671236894707091)
--(axis cs:4,0.663300004137105)
--(axis cs:8,0.667405408960443)
--(axis cs:16,0.64872853153904)
--(axis cs:32,0.67507986339463)
--(axis cs:32,0.68263927242039)
--(axis cs:32,0.68263927242039)
--(axis cs:16,0.65834713773981)
--(axis cs:8,0.683173306510362)
--(axis cs:4,0.671236894707091)
--cycle;

\path [fill=forestgreen4416044, fill opacity=0.2]
(axis cs:4,0.589910447416111)
--(axis cs:4,0.586156691260066)
--(axis cs:8,0.640640635143339)
--(axis cs:16,0.64123331115237)
--(axis cs:32,0.640622497518217)
--(axis cs:32,0.649952449467695)
--(axis cs:32,0.649952449467695)
--(axis cs:16,0.646538678341057)
--(axis cs:8,0.647112376368614)
--(axis cs:4,0.589910447416111)
--cycle;

\addplot [semithick, steelblue31119180]
table {%
4 0.483311260133451
8 0.582886785947799
16 0.701346758899657
32 0.704760897766969
};
\addlegendentry{$\mind=10$}
\addplot [semithick, darkorange25512714]
table {%
4 0.667268449422098
8 0.675289357735403
16 0.653537834639425
32 0.67885956790751
};
\addlegendentry{$\mind=20$}
\addplot [semithick, forestgreen4416044]
table {%
4 0.588033569338089
8 0.643876505755977
16 0.643885994746714
32 0.645287473492956
};
\addlegendentry{$\mind=40$}
\end{axis}

\end{tikzpicture}}
        \end{subfigure}
        \begin{subfigure}{0.32\linewidth}
            \resizebox{\linewidth}{\linewidth}{
\begin{tikzpicture}

\definecolor{crimson2143940}{RGB}{214,39,40}
\definecolor{darkgray176}{RGB}{176,176,176}
\definecolor{darkorange25512714}{RGB}{255,127,14}
\definecolor{forestgreen4416044}{RGB}{44,160,44}
\definecolor{lightgray204}{RGB}{204,204,204}
\definecolor{mediumpurple148103189}{RGB}{148,103,189}
\definecolor{steelblue31119180}{RGB}{31,119,180}

\begin{axis}[
legend cell align={left},
legend style={
  fill opacity=0.8,
  draw opacity=1,
  text opacity=1,
  at={(0.97,0.03)},
  anchor=south east,
  draw=lightgray204
},
scale only axis,
width=4cm,
height=3cm,
tick align=outside,
tick pos=left,
title={GCN-C, Reddit},
tick label style={font=\tiny},
legend style={nodes={scale=0.5, transform shape}},
every tick/.style={
black,
semithick,
},
x label style={at={(axis description cs:0.5,-0.1)},anchor=north,font=\tiny},
y label style={at={(axis description cs:-0.1,.5)},rotate=90,anchor=south,font=\tiny},
x grid style={darkgray176},
xlabel={$\epsilon$},
xmin=2.6, xmax=33.4,
xtick style={color=black},
y grid style={darkgray176},
ylabel={ACC},
ymin=0.869623642600484, ymax=0.929801077875518,
ytick style={color=black}
]
\path [fill=steelblue31119180, fill opacity=0.2]
(axis cs:4,0.883923445872697)
--(axis cs:4,0.872358980567531)
--(axis cs:8,0.901390855237169)
--(axis cs:16,0.919163365854363)
--(axis cs:32,0.924299536647549)
--(axis cs:32,0.927065739908471)
--(axis cs:32,0.927065739908471)
--(axis cs:16,0.923085705111293)
--(axis cs:8,0.905097125661526)
--(axis cs:4,0.883923445872697)
--cycle;

\path [fill=darkorange25512714, fill opacity=0.2]
(axis cs:4,0.901205130589604)
--(axis cs:4,0.895829140455044)
--(axis cs:8,0.91111645676936)
--(axis cs:16,0.914355992521677)
--(axis cs:32,0.914660708629969)
--(axis cs:32,0.921520466531153)
--(axis cs:32,0.921520466531153)
--(axis cs:16,0.918514768478628)
--(axis cs:8,0.914465648323723)
--(axis cs:4,0.901205130589604)
--cycle;

\path [fill=forestgreen4416044, fill opacity=0.2]
(axis cs:4,0.89878166463381)
--(axis cs:4,0.891121931222786)
--(axis cs:8,0.897491871866255)
--(axis cs:16,0.901811399876718)
--(axis cs:32,0.903877016652382)
--(axis cs:32,0.905792121455081)
--(axis cs:32,0.905792121455081)
--(axis cs:16,0.905337227665784)
--(axis cs:8,0.900781111635531)
--(axis cs:4,0.89878166463381)
--cycle;

\addplot [semithick, steelblue31119180]
table {%
4 0.878141213220114
8 0.903243990449347
16 0.921124535482828
32 0.92568263827801
};
\addlegendentry{$\mind=10$}
\addplot [semithick, darkorange25512714]
table {%
4 0.898517135522324
8 0.912791052546541
16 0.916435380500152
32 0.918090587580561
};
\addlegendentry{$\mind=20$}
\addplot [semithick, forestgreen4416044]
table {%
4 0.894951797928298
8 0.899136491750893
16 0.903574313771251
32 0.904834569053731
};
\addlegendentry{$\mind=40$}
\end{axis}

\end{tikzpicture}}
        \end{subfigure}
        \begin{subfigure}{0.32\linewidth}
            \resizebox{\linewidth}{\linewidth}{
\begin{tikzpicture}

\definecolor{crimson2143940}{RGB}{214,39,40}
\definecolor{darkgray176}{RGB}{176,176,176}
\definecolor{darkorange25512714}{RGB}{255,127,14}
\definecolor{forestgreen4416044}{RGB}{44,160,44}
\definecolor{lightgray204}{RGB}{204,204,204}
\definecolor{mediumpurple148103189}{RGB}{148,103,189}
\definecolor{steelblue31119180}{RGB}{31,119,180}

\begin{axis}[
legend cell align={left},
legend style={
  fill opacity=0.8,
  draw opacity=1,
  text opacity=1,
  at={(0.97,0.03)},
  anchor=south east,
  draw=lightgray204
},
scale only axis,
width=4cm,
height=3cm,
tick align=outside,
tick pos=left,
title={GCN-C, Finance},
tick label style={font=\tiny},
legend style={nodes={scale=0.5, transform shape}},
every tick/.style={
black,
semithick,
},
x label style={at={(axis description cs:0.5,-0.1)},anchor=north,font=\tiny},
y label style={at={(axis description cs:-0.1,.5)},rotate=90,anchor=south,font=\tiny},
x grid style={darkgray176},
xlabel={$\epsilon$},
xmin=2.6, xmax=33.4,
xtick style={color=black},
y grid style={darkgray176},
ylabel={AUC},
ymin=0.643220052924958, ymax=0.742952047622332,
ytick style={color=black}
]
\path [fill=steelblue31119180, fill opacity=0.2]
(axis cs:4,0.713308700130452)
--(axis cs:4,0.695578977852054)
--(axis cs:8,0.704781784510618)
--(axis cs:16,0.719910308412774)
--(axis cs:32,0.71606570449207)
--(axis cs:32,0.738418775136088)
--(axis cs:32,0.738418775136088)
--(axis cs:16,0.731320865663228)
--(axis cs:8,0.723067670314412)
--(axis cs:4,0.713308700130452)
--cycle;

\path [fill=darkorange25512714, fill opacity=0.2]
(axis cs:4,0.687253964588904)
--(axis cs:4,0.647753325411202)
--(axis cs:8,0.669942860732738)
--(axis cs:16,0.67852723481068)
--(axis cs:32,0.681161788642319)
--(axis cs:32,0.702097030521673)
--(axis cs:32,0.702097030521673)
--(axis cs:16,0.689658895135507)
--(axis cs:8,0.694249024278113)
--(axis cs:4,0.687253964588904)
--cycle;

\addplot [semithick, steelblue31119180]
table {%
4 0.704443838991253
8 0.713924727412515
16 0.725615587038001
32 0.727242239814079
};
\addlegendentry{$\mind=3$}
\addplot [semithick, darkorange25512714]
table {%
4 0.667503645000053
8 0.682095942505425
16 0.684093064973094
32 0.691629409581996
};
\addlegendentry{$\mind=5$}
\end{axis}

\end{tikzpicture}}
        \end{subfigure}
        \caption{Performance ($\text{mean}\pm\text{std}$ over $10$ trials) of \vesper under varying minmimum degree $\mind$, using GCN aggregator and CONCAT decoder.}
        \label{fig: ablation_min_degree_gcn_c}
    \end{figure}

    \begin{figure}
        \centering
        \begin{subfigure}{0.32\linewidth}
            \resizebox{\linewidth}{\linewidth}{
\begin{tikzpicture}

\definecolor{crimson2143940}{RGB}{214,39,40}
\definecolor{darkgray176}{RGB}{176,176,176}
\definecolor{darkorange25512714}{RGB}{255,127,14}
\definecolor{forestgreen4416044}{RGB}{44,160,44}
\definecolor{lightgray204}{RGB}{204,204,204}
\definecolor{mediumpurple148103189}{RGB}{148,103,189}
\definecolor{steelblue31119180}{RGB}{31,119,180}

\begin{axis}[
legend cell align={left},
legend style={
  fill opacity=0.8,
  draw opacity=1,
  text opacity=1,
  at={(0.97,0.03)},
  anchor=south east,
  draw=lightgray204
},
scale only axis,
width=4cm,
height=3cm,
tick align=outside,
tick pos=left,
title={GCN-G, \ogb},
tick label style={font=\tiny},
legend style={nodes={scale=0.5, transform shape}},
every tick/.style={
black,
semithick,
},
x label style={at={(axis description cs:0.5,-0.1)},anchor=north,font=\tiny},
y label style={at={(axis description cs:-0.1,.5)},rotate=90,anchor=south,font=\tiny},
x grid style={darkgray176},
xlabel={$\epsilon$},
xmin=2.6, xmax=33.4,
xtick style={color=black},
y grid style={darkgray176},
ylabel={ACC},
ymin=0.630649628346624, ymax=0.713315825605368,
ytick style={color=black}
]
\path [fill=steelblue31119180, fill opacity=0.2]
(axis cs:4,0.680000881786496)
--(axis cs:4,0.672056489555216)
--(axis cs:8,0.680185788219953)
--(axis cs:16,0.695438475566557)
--(axis cs:32,0.70287122222538)
--(axis cs:32,0.709558271184516)
--(axis cs:32,0.709558271184516)
--(axis cs:16,0.704988438645285)
--(axis cs:8,0.693412676551717)
--(axis cs:4,0.680000881786496)
--cycle;

\path [fill=darkorange25512714, fill opacity=0.2]
(axis cs:4,0.670429237744585)
--(axis cs:4,0.659030689569746)
--(axis cs:8,0.666792358866377)
--(axis cs:16,0.673779727153682)
--(axis cs:32,0.67452788019413)
--(axis cs:32,0.685661194724163)
--(axis cs:32,0.685661194724163)
--(axis cs:16,0.681695849765659)
--(axis cs:8,0.673404135538959)
--(axis cs:4,0.670429237744585)
--cycle;

\path [fill=forestgreen4416044, fill opacity=0.2]
(axis cs:4,0.641637408259283)
--(axis cs:4,0.634407182767476)
--(axis cs:8,0.635888050881532)
--(axis cs:16,0.641684351335065)
--(axis cs:32,0.641261537384843)
--(axis cs:32,0.645196904676509)
--(axis cs:32,0.645196904676509)
--(axis cs:16,0.647736192149139)
--(axis cs:8,0.645999318413269)
--(axis cs:4,0.641637408259283)
--cycle;

\addplot [semithick, steelblue31119180]
table {%
4 0.676028685670856
8 0.686799232385835
16 0.700213457105921
32 0.706214746704948
};
\addlegendentry{$\mind=10$}
\addplot [semithick, darkorange25512714]
table {%
4 0.664729963657166
8 0.670098247202668
16 0.67773778845967
32 0.680094537459146
};
\addlegendentry{$\mind=20$}
\addplot [semithick, forestgreen4416044]
table {%
4 0.638022295513379
8 0.6409436846474
16 0.644710271742102
32 0.643229221030676
};
\addlegendentry{$\mind=40$}
\end{axis}

\end{tikzpicture}}
        \end{subfigure}
        \begin{subfigure}{0.32\linewidth}
            \resizebox{\linewidth}{\linewidth}{
\begin{tikzpicture}

\definecolor{crimson2143940}{RGB}{214,39,40}
\definecolor{darkgray176}{RGB}{176,176,176}
\definecolor{darkorange25512714}{RGB}{255,127,14}
\definecolor{forestgreen4416044}{RGB}{44,160,44}
\definecolor{lightgray204}{RGB}{204,204,204}
\definecolor{mediumpurple148103189}{RGB}{148,103,189}
\definecolor{steelblue31119180}{RGB}{31,119,180}

\begin{axis}[
legend cell align={left},
legend style={
  fill opacity=0.8,
  draw opacity=1,
  text opacity=1,
  at={(0.97,0.03)},
  anchor=south east,
  draw=lightgray204
},
scale only axis,
width=4cm,
height=3cm,
tick align=outside,
tick pos=left,
title={GCN-G, Reddit},
tick label style={font=\tiny},
legend style={nodes={scale=0.5, transform shape}},
every tick/.style={
black,
semithick,
},
x label style={at={(axis description cs:0.5,-0.1)},anchor=north,font=\tiny},
y label style={at={(axis description cs:-0.1,.5)},rotate=90,anchor=south,font=\tiny},
x grid style={darkgray176},
xlabel={$\epsilon$},
xmin=2.6, xmax=33.4,
xtick style={color=black},
y grid style={darkgray176},
ylabel={ACC},
ymin=0.805869217640365, ymax=0.9317203638149,
ytick style={color=black}
]
\path [fill=steelblue31119180, fill opacity=0.2]
(axis cs:4,0.8654833238456)
--(axis cs:4,0.811589724284662)
--(axis cs:8,0.886629132159241)
--(axis cs:16,0.90992763882684)
--(axis cs:32,0.922284795361577)
--(axis cs:32,0.925999857170603)
--(axis cs:32,0.925999857170603)
--(axis cs:16,0.9160386466515)
--(axis cs:8,0.898251394921886)
--(axis cs:4,0.8654833238456)
--cycle;

\path [fill=darkorange25512714, fill opacity=0.2]
(axis cs:4,0.899409682052782)
--(axis cs:4,0.89043287579868)
--(axis cs:8,0.907118899049947)
--(axis cs:16,0.91272935596653)
--(axis cs:32,0.918120224060804)
--(axis cs:32,0.920308585877619)
--(axis cs:32,0.920308585877619)
--(axis cs:16,0.916723276746251)
--(axis cs:8,0.913249842310482)
--(axis cs:4,0.899409682052782)
--cycle;

\path [fill=forestgreen4416044, fill opacity=0.2]
(axis cs:4,0.896485500605356)
--(axis cs:4,0.894348027211817)
--(axis cs:8,0.896944312260268)
--(axis cs:16,0.902902016531102)
--(axis cs:32,0.903421043008538)
--(axis cs:32,0.90673639914)
--(axis cs:32,0.90673639914)
--(axis cs:16,0.905043695549037)
--(axis cs:8,0.901475880548019)
--(axis cs:4,0.896485500605356)
--cycle;

\addplot [semithick, steelblue31119180]
table {%
4 0.838536524065131
8 0.892440263540563
16 0.91298314273917
32 0.92414232626609
};
\addlegendentry{$\mind=10$}
\addplot [semithick, darkorange25512714]
table {%
4 0.894921278925731
8 0.910184370680215
16 0.91472631635639
32 0.919214404969212
};
\addlegendentry{$\mind=20$}
\addplot [semithick, forestgreen4416044]
table {%
4 0.895416763908586
8 0.899210096404143
16 0.90397285604007
32 0.905078721074269
};
\addlegendentry{$\mind=40$}
\end{axis}

\end{tikzpicture}}
        \end{subfigure}
        \begin{subfigure}{0.32\linewidth}
            \resizebox{\linewidth}{\linewidth}{
\begin{tikzpicture}

\definecolor{crimson2143940}{RGB}{214,39,40}
\definecolor{darkgray176}{RGB}{176,176,176}
\definecolor{darkorange25512714}{RGB}{255,127,14}
\definecolor{forestgreen4416044}{RGB}{44,160,44}
\definecolor{lightgray204}{RGB}{204,204,204}
\definecolor{mediumpurple148103189}{RGB}{148,103,189}
\definecolor{steelblue31119180}{RGB}{31,119,180}

\begin{axis}[
legend cell align={left},
legend style={
  fill opacity=0.8,
  draw opacity=1,
  text opacity=1,
  at={(0.97,0.03)},
  anchor=south east,
  draw=lightgray204
},
scale only axis,
width=4cm,
height=3cm,
tick align=outside,
tick pos=left,
title={GCN-G, Finance},
tick label style={font=\tiny},
legend style={nodes={scale=0.5, transform shape}},
every tick/.style={
black,
semithick,
},
x label style={at={(axis description cs:0.5,-0.1)},anchor=north,font=\tiny},
y label style={at={(axis description cs:-0.1,.5)},rotate=90,anchor=south,font=\tiny},
x grid style={darkgray176},
xlabel={$\epsilon$},
xmin=2.6, xmax=33.4,
xtick style={color=black},
y grid style={darkgray176},
ylabel={AUC},
ymin=0.705531607616328, ymax=0.751863940460786,
ytick style={color=black}
]
\path [fill=steelblue31119180, fill opacity=0.2]
(axis cs:4,0.743018694014603)
--(axis cs:4,0.728438173641977)
--(axis cs:8,0.731523147425954)
--(axis cs:16,0.733753921838726)
--(axis cs:32,0.742662609660882)
--(axis cs:32,0.749757925331492)
--(axis cs:32,0.749757925331492)
--(axis cs:16,0.744160500084317)
--(axis cs:8,0.747717439163771)
--(axis cs:4,0.743018694014603)
--cycle;

\path [fill=darkorange25512714, fill opacity=0.2]
(axis cs:4,0.714664170912978)
--(axis cs:4,0.709237007937914)
--(axis cs:8,0.707637622745622)
--(axis cs:16,0.709678504102748)
--(axis cs:32,0.710596510677419)
--(axis cs:32,0.720257552073195)
--(axis cs:32,0.720257552073195)
--(axis cs:16,0.717758001667091)
--(axis cs:8,0.717439569968065)
--(axis cs:4,0.714664170912978)
--cycle;

\addplot [semithick, steelblue31119180]
table {%
4 0.73572843382829
8 0.739620293294862
16 0.738957210961521
32 0.746210267496187
};
\addlegendentry{$\mind=3$}
\addplot [semithick, darkorange25512714]
table {%
4 0.711950589425446
8 0.712538596356843
16 0.71371825288492
32 0.715427031375307
};
\addlegendentry{$\mind=5$}
\end{axis}

\end{tikzpicture}}
        \end{subfigure}
        \caption{Performance ($\text{mean}\pm\text{std}$ over $10$ trials) of \vesper under varying minmimum degree $\mind$, using GCN aggregator and GRU decoder.}
        \label{fig: ablation_min_degree_gcn_g}
    \end{figure}

    \textbf{On the effect of batch size} Finally, we assess the effect of varying batch sizes. According to composition results \cite{abadi2016deep, mironov2017renyi, wang2019subsampled}, consider running a fixed amount of epochs under a given sample and privacy constraint, choosing a smaller batch size in general leads to a smaller per-step noise scale. However, an overly small batch size may cause the stochastic gradients to be too noisy for good performance. We evaluate the effect of batch size under the range $\{32, 64, 128, 256, 512\}$ across all three datasets and plot the resulting curves in figure \ref{fig: ablation_batch_size_gin_c}, \ref{fig: ablation_batch_size_gcn_c}, \ref{fig: ablation_batch_size_gin_g}, \ref{fig: ablation_batch_size_gcn_g}. According to the results, we find that the reduction in noise scale caused by choosing smaller batch sizes may produces better performance on \ogb and Reddit dataset, while on the Finance dataset changing batch size does not produce a statistically significant difference in privacy-utility trade-offs, as indicated by the overlapping region in the figure.
    
    \begin{figure}
        \centering
        \begin{subfigure}{0.32\linewidth}
            \resizebox{\linewidth}{\linewidth}{
\begin{tikzpicture}

\definecolor{crimson2143940}{RGB}{214,39,40}
\definecolor{darkgray176}{RGB}{176,176,176}
\definecolor{darkorange25512714}{RGB}{255,127,14}
\definecolor{forestgreen4416044}{RGB}{44,160,44}
\definecolor{lightgray204}{RGB}{204,204,204}
\definecolor{mediumpurple148103189}{RGB}{148,103,189}
\definecolor{steelblue31119180}{RGB}{31,119,180}

\begin{axis}[
legend cell align={left},
legend style={
  fill opacity=0.8,
  draw opacity=1,
  text opacity=1,
  at={(0.97,0.03)},
  anchor=south east,
  draw=lightgray204
},
scale only axis,
width=4cm,
height=3cm,
tick align=outside,
tick pos=left,
title={GIN-C, \ogb},
tick label style={font=\tiny},
legend style={nodes={scale=0.5, transform shape}},
every tick/.style={
black,
semithick,
},
x label style={at={(axis description cs:0.5,-0.1)},anchor=north,font=\tiny},
y label style={at={(axis description cs:-0.1,.5)},rotate=90,anchor=south,font=\tiny},
x grid style={darkgray176},
xlabel={$\epsilon$},
xmin=2.6, xmax=33.4,
xtick style={color=black},
y grid style={darkgray176},
ylabel={ACC},
ymin=0.300488021369875, ymax=0.78962353582144,
ytick style={color=black}
]
\path [fill=steelblue31119180, fill opacity=0.2]
(axis cs:4,0.7197549408467)
--(axis cs:4,0.70571278732173)
--(axis cs:8,0.740175850764731)
--(axis cs:16,0.746634995034394)
--(axis cs:32,0.749100452170779)
--(axis cs:32,0.767390103346369)
--(axis cs:32,0.767390103346369)
--(axis cs:16,0.757031415429523)
--(axis cs:8,0.749100369688924)
--(axis cs:4,0.7197549408467)
--cycle;

\path [fill=darkorange25512714, fill opacity=0.2]
(axis cs:4,0.65318397845076)
--(axis cs:4,0.633752347258395)
--(axis cs:8,0.728742140008732)
--(axis cs:16,0.742783174848598)
--(axis cs:32,0.74933819907649)
--(axis cs:32,0.760001181906939)
--(axis cs:32,0.760001181906939)
--(axis cs:16,0.751746964219339)
--(axis cs:8,0.737226679167705)
--(axis cs:4,0.65318397845076)
--cycle;

\path [fill=forestgreen4416044, fill opacity=0.2]
(axis cs:4,0.546316963766723)
--(axis cs:4,0.51849338519769)
--(axis cs:8,0.694719848125915)
--(axis cs:16,0.737501786090669)
--(axis cs:32,0.744990067036807)
--(axis cs:32,0.753178874050568)
--(axis cs:32,0.753178874050568)
--(axis cs:16,0.746354955453172)
--(axis cs:8,0.704893904765403)
--(axis cs:4,0.546316963766723)
--cycle;

\path [fill=crimson2143940, fill opacity=0.2]
(axis cs:4,0.486818558042197)
--(axis cs:4,0.424687204712249)
--(axis cs:8,0.609262788620289)
--(axis cs:16,0.716550450280534)
--(axis cs:32,0.735737027612018)
--(axis cs:32,0.747350924939413)
--(axis cs:32,0.747350924939413)
--(axis cs:16,0.727216751339282)
--(axis cs:8,0.633857444573917)
--(axis cs:4,0.486818558042197)
--cycle;

\path [fill=mediumpurple148103189, fill opacity=0.2]
(axis cs:4,0.405244635213299)
--(axis cs:4,0.322721453844946)
--(axis cs:8,0.526670434535552)
--(axis cs:16,0.672135265223444)
--(axis cs:32,0.724315514522605)
--(axis cs:32,0.739449509147909)
--(axis cs:32,0.739449509147909)
--(axis cs:16,0.687641354897464)
--(axis cs:8,0.55663332477665)
--(axis cs:4,0.405244635213299)
--cycle;

\addplot [semithick, steelblue31119180]
table {%
4 0.712733864084215
8 0.744638110226828
16 0.751833205231958
32 0.758245277758574
};
\addlegendentry{$B=32$}
\addplot [semithick, darkorange25512714]
table {%
4 0.643468162854578
8 0.732984409588219
16 0.747265069533968
32 0.754669690491715
};
\addlegendentry{$B=64$}
\addplot [semithick, forestgreen4416044]
table {%
4 0.532405174482206
8 0.699806876445659
16 0.74192837077192
32 0.749084470543688
};
\addlegendentry{$B=128$}
\addplot [semithick, crimson2143940]
table {%
4 0.455752881377223
8 0.621560116597103
16 0.721883600809908
32 0.741543976275716
};
\addlegendentry{$B=256$}
\addplot [semithick, mediumpurple148103189]
table {%
4 0.363983044529122
8 0.541651879656101
16 0.679888310060454
32 0.731882511835257
};
\addlegendentry{$B=512$}
\end{axis}

\end{tikzpicture}}
        \end{subfigure}
        \begin{subfigure}{0.32\linewidth}
            \resizebox{\linewidth}{\linewidth}{
\begin{tikzpicture}

\definecolor{crimson2143940}{RGB}{214,39,40}
\definecolor{darkgray176}{RGB}{176,176,176}
\definecolor{darkorange25512714}{RGB}{255,127,14}
\definecolor{forestgreen4416044}{RGB}{44,160,44}
\definecolor{lightgray204}{RGB}{204,204,204}
\definecolor{mediumpurple148103189}{RGB}{148,103,189}
\definecolor{steelblue31119180}{RGB}{31,119,180}

\begin{axis}[
legend cell align={left},
legend style={
  fill opacity=0.8,
  draw opacity=1,
  text opacity=1,
  at={(0.97,0.03)},
  anchor=south east,
  draw=lightgray204
},
scale only axis,
width=4cm,
height=3cm,
tick align=outside,
tick pos=left,
title={GIN-C, Reddit},
tick label style={font=\tiny},
legend style={nodes={scale=0.5, transform shape}},
every tick/.style={
black,
semithick,
},
x label style={at={(axis description cs:0.5,-0.1)},anchor=north,font=\tiny},
y label style={at={(axis description cs:-0.1,.5)},rotate=90,anchor=south,font=\tiny},
x grid style={darkgray176},
xlabel={$\epsilon$},
xmin=2.6, xmax=33.4,
xtick style={color=black},
y grid style={darkgray176},
ylabel={ACC},
ymin=0.131430354213292, ymax=0.978470090547045,
ytick style={color=black}
]
\path [fill=steelblue31119180, fill opacity=0.2]
(axis cs:4,0.839157261175465)
--(axis cs:4,0.807734288651295)
--(axis cs:8,0.913018942656196)
--(axis cs:16,0.931412147073208)
--(axis cs:32,0.935417242461785)
--(axis cs:32,0.939968284350056)
--(axis cs:32,0.939968284350056)
--(axis cs:16,0.935367021014686)
--(axis cs:8,0.917478517085649)
--(axis cs:4,0.839157261175465)
--cycle;

\path [fill=darkorange25512714, fill opacity=0.2]
(axis cs:4,0.680243803520987)
--(axis cs:4,0.609611320978591)
--(axis cs:8,0.873689318074211)
--(axis cs:16,0.924750619406875)
--(axis cs:32,0.936232129214464)
--(axis cs:32,0.938130113393654)
--(axis cs:32,0.938130113393654)
--(axis cs:16,0.92961270034251)
--(axis cs:8,0.894348310060718)
--(axis cs:4,0.680243803520987)
--cycle;

\path [fill=forestgreen4416044, fill opacity=0.2]
(axis cs:4,0.515405029760775)
--(axis cs:4,0.371735698745769)
--(axis cs:8,0.720728644034999)
--(axis cs:16,0.907309293990895)
--(axis cs:32,0.931485624818726)
--(axis cs:32,0.935189419613351)
--(axis cs:32,0.935189419613351)
--(axis cs:16,0.92122058770309)
--(axis cs:8,0.790561591679414)
--(axis cs:4,0.515405029760775)
--cycle;

\path [fill=crimson2143940, fill opacity=0.2]
(axis cs:4,0.396483550733231)
--(axis cs:4,0.24023619506143)
--(axis cs:8,0.576578747197536)
--(axis cs:16,0.830433039602076)
--(axis cs:32,0.921156019546149)
--(axis cs:32,0.928144736247974)
--(axis cs:32,0.928144736247974)
--(axis cs:16,0.870344297345665)
--(axis cs:8,0.693622139648774)
--(axis cs:4,0.396483550733231)
--cycle;

\path [fill=mediumpurple148103189, fill opacity=0.2]
(axis cs:4,0.274359888491933)
--(axis cs:4,0.169932160410281)
--(axis cs:8,0.360733698506469)
--(axis cs:16,0.692045352268548)
--(axis cs:32,0.889016793756991)
--(axis cs:32,0.912563013434722)
--(axis cs:32,0.912563013434722)
--(axis cs:16,0.775706833430606)
--(axis cs:8,0.449068287023934)
--(axis cs:4,0.274359888491933)
--cycle;

\addplot [semithick, steelblue31119180]
table {%
4 0.82344577491338
8 0.915248729870923
16 0.933389584043947
32 0.937692763405921
};
\addlegendentry{$B=32$}
\addplot [semithick, darkorange25512714]
table {%
4 0.644927562249789
8 0.884018814067465
16 0.927181659874693
32 0.937181121304059
};
\addlegendentry{$B=64$}
\addplot [semithick, forestgreen4416044]
table {%
4 0.443570364253272
8 0.755645117857207
16 0.914264940846992
32 0.933337522216039
};
\addlegendentry{$B=128$}
\addplot [semithick, crimson2143940]
table {%
4 0.31835987289733
8 0.635100443423155
16 0.85038866847387
32 0.924650377897061
};
\addlegendentry{$B=256$}
\addplot [semithick, mediumpurple148103189]
table {%
4 0.222146024451107
8 0.404900992765201
16 0.733876092849577
32 0.900789903595856
};
\addlegendentry{$B=512$}
\end{axis}

\end{tikzpicture}}
        \end{subfigure}
        \begin{subfigure}{0.32\linewidth}
            \resizebox{\linewidth}{\linewidth}{
\begin{tikzpicture}

\definecolor{crimson2143940}{RGB}{214,39,40}
\definecolor{darkgray176}{RGB}{176,176,176}
\definecolor{darkorange25512714}{RGB}{255,127,14}
\definecolor{forestgreen4416044}{RGB}{44,160,44}
\definecolor{lightgray204}{RGB}{204,204,204}
\definecolor{mediumpurple148103189}{RGB}{148,103,189}
\definecolor{steelblue31119180}{RGB}{31,119,180}

\begin{axis}[
legend cell align={left},
legend style={
  fill opacity=0.8,
  draw opacity=1,
  text opacity=1,
  at={(0.97,0.03)},
  anchor=south east,
  draw=lightgray204
},
scale only axis,
width=4cm,
height=3cm,
tick align=outside,
tick pos=left,
title={GIN-C, Finance},
tick label style={font=\tiny},
legend style={nodes={scale=0.5, transform shape}},
every tick/.style={
black,
semithick,
},
x label style={at={(axis description cs:0.5,-0.1)},anchor=north,font=\tiny},
y label style={at={(axis description cs:-0.1,.5)},rotate=90,anchor=south,font=\tiny},
x grid style={darkgray176},
xlabel={$\epsilon$},
xmin=2.6, xmax=33.4,
xtick style={color=black},
y grid style={darkgray176},
ylabel={AUC},
ymin=0.642839578389711, ymax=0.698034139222402,
ytick style={color=black}
]
\path [fill=steelblue31119180, fill opacity=0.2]
(axis cs:4,0.68900594354819)
--(axis cs:4,0.651553625854771)
--(axis cs:8,0.66307876137867)
--(axis cs:16,0.645348422063924)
--(axis cs:16,0.682937497177455)
--(axis cs:16,0.682937497177455)
--(axis cs:8,0.69425621541033)
--(axis cs:4,0.68900594354819)
--cycle;

\path [fill=darkorange25512714, fill opacity=0.2]
(axis cs:4,0.676172918449245)
--(axis cs:4,0.658832335816067)
--(axis cs:8,0.667051133124192)
--(axis cs:16,0.657032690697198)
--(axis cs:32,0.650665819692631)
--(axis cs:32,0.690569340213505)
--(axis cs:32,0.690569340213505)
--(axis cs:16,0.688462593174523)
--(axis cs:8,0.683044252696769)
--(axis cs:4,0.676172918449245)
--cycle;

\path [fill=forestgreen4416044, fill opacity=0.2]
(axis cs:4,0.681123852510348)
--(axis cs:4,0.653452022625872)
--(axis cs:8,0.652288482767219)
--(axis cs:16,0.669652752762486)
--(axis cs:32,0.666688887947501)
--(axis cs:32,0.693213077385499)
--(axis cs:32,0.693213077385499)
--(axis cs:16,0.691976141862098)
--(axis cs:8,0.686539219300425)
--(axis cs:4,0.681123852510348)
--cycle;

\path [fill=crimson2143940, fill opacity=0.2]
(axis cs:4,0.672337063210232)
--(axis cs:4,0.652157086671849)
--(axis cs:8,0.657970476353382)
--(axis cs:16,0.672186716154177)
--(axis cs:32,0.677301034943372)
--(axis cs:32,0.695525295548189)
--(axis cs:32,0.695525295548189)
--(axis cs:16,0.686880890039866)
--(axis cs:8,0.679283256424776)
--(axis cs:4,0.672337063210232)
--cycle;

\path [fill=mediumpurple148103189, fill opacity=0.2]
(axis cs:4,0.67738492883017)
--(axis cs:4,0.658832151905589)
--(axis cs:8,0.657050665990762)
--(axis cs:16,0.660804667546509)
--(axis cs:32,0.67397841095339)
--(axis cs:32,0.687450372005565)
--(axis cs:32,0.687450372005565)
--(axis cs:16,0.689691555183909)
--(axis cs:8,0.675420921696343)
--(axis cs:4,0.67738492883017)
--cycle;

\addplot [semithick, steelblue31119180]
table {%
4 0.67027978470148
8 0.6786674883945
16 0.66414295962069
};
\addlegendentry{$B=32$}
\addplot [semithick, darkorange25512714]
table {%
4 0.667502627132656
8 0.67504769291048
16 0.67274764193586
32 0.670617579953068
};
\addlegendentry{$B=64$}
\addplot [semithick, forestgreen4416044]
table {%
4 0.66728793756811
8 0.669413851033822
16 0.680814447312292
32 0.6799509826665
};
\addlegendentry{$B=128$}
\addplot [semithick, crimson2143940]
table {%
4 0.66224707494104
8 0.668626866389079
16 0.679533803097021
32 0.68641316524578
};
\addlegendentry{$B=256$}
\addplot [semithick, mediumpurple148103189]
table {%
4 0.668108540367879
8 0.666235793843552
16 0.675248111365209
32 0.680714391479478
};
\addlegendentry{$B=512$}
\end{axis}

\end{tikzpicture}}
        \end{subfigure}
        \caption{Performance ($\text{mean}\pm\text{std}$ over $10$ trials) of \vesper under varying batch size $B$, using GIN aggregator and CONCAT decoder.}
        \label{fig: ablation_batch_size_gin_c}
    \end{figure}
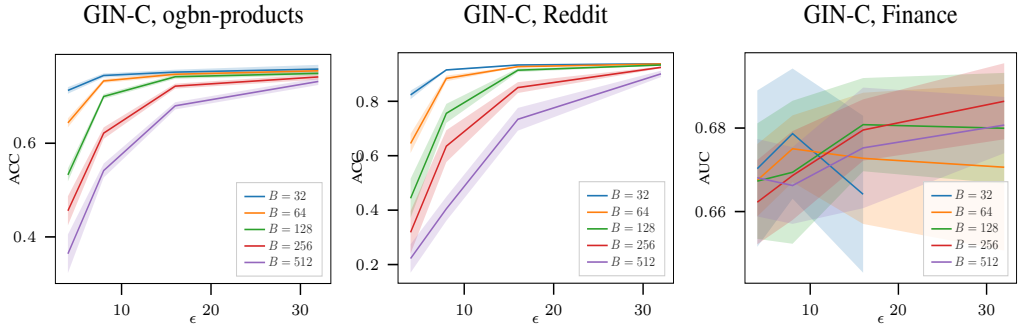


    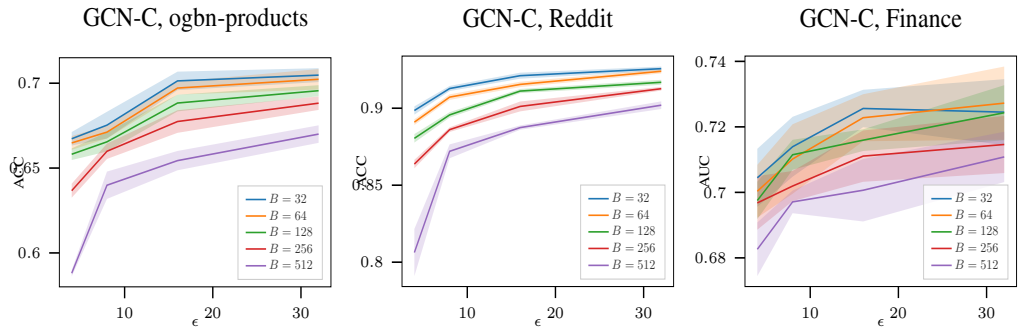
\begin{figure}
        \centering
        \begin{subfigure}{0.32\linewidth}
            \resizebox{\linewidth}{\linewidth}{
\begin{tikzpicture}

\definecolor{crimson2143940}{RGB}{214,39,40}
\definecolor{darkgray176}{RGB}{176,176,176}
\definecolor{darkorange25512714}{RGB}{255,127,14}
\definecolor{forestgreen4416044}{RGB}{44,160,44}
\definecolor{lightgray204}{RGB}{204,204,204}
\definecolor{mediumpurple148103189}{RGB}{148,103,189}
\definecolor{steelblue31119180}{RGB}{31,119,180}

\begin{axis}[
legend cell align={left},
legend style={
  fill opacity=0.8,
  draw opacity=1,
  text opacity=1,
  at={(0.97,0.03)},
  anchor=south east,
  draw=lightgray204
},
scale only axis,
width=4cm,
height=3cm,
tick align=outside,
tick pos=left,
title={GCN-C, \ogb},
tick label style={font=\tiny},
legend style={nodes={scale=0.5, transform shape}},
every tick/.style={
black,
semithick,
},
x label style={at={(axis description cs:0.5,-0.1)},anchor=north,font=\tiny},
y label style={at={(axis description cs:-0.1,.5)},rotate=90,anchor=south,font=\tiny},
x grid style={darkgray176},
xlabel={$\epsilon$},
xmin=2.6, xmax=33.4,
xtick style={color=black},
y grid style={darkgray176},
ylabel={ACC},
ymin=0.580021201868666, ymax=0.715001968479474,
ytick style={color=black}
]
\path [fill=steelblue31119180, fill opacity=0.2]
(axis cs:4,0.671236894707091)
--(axis cs:4,0.663300004137105)
--(axis cs:8,0.667405408960443)
--(axis cs:16,0.69588391704559)
--(axis cs:32,0.700655316445865)
--(axis cs:32,0.708866479088074)
--(axis cs:32,0.708866479088074)
--(axis cs:16,0.706809600753723)
--(axis cs:8,0.683173306510362)
--(axis cs:4,0.671236894707091)
--cycle;

\path [fill=darkorange25512714, fill opacity=0.2]
(axis cs:4,0.66886731322877)
--(axis cs:4,0.660745703136125)
--(axis cs:8,0.667408222492391)
--(axis cs:16,0.693269748060835)
--(axis cs:32,0.696276638146028)
--(axis cs:32,0.708335915065746)
--(axis cs:32,0.708335915065746)
--(axis cs:16,0.701005317898947)
--(axis cs:8,0.674953930713239)
--(axis cs:4,0.66886731322877)
--cycle;

\path [fill=forestgreen4416044, fill opacity=0.2]
(axis cs:4,0.661566064851102)
--(axis cs:4,0.654723233691073)
--(axis cs:8,0.659735830694155)
--(axis cs:16,0.683382755017522)
--(axis cs:32,0.692330753403105)
--(axis cs:32,0.698832465822855)
--(axis cs:32,0.698832465822855)
--(axis cs:16,0.693241974828652)
--(axis cs:8,0.671167869153705)
--(axis cs:4,0.661566064851102)
--cycle;

\path [fill=crimson2143940, fill opacity=0.2]
(axis cs:4,0.640790701080829)
--(axis cs:4,0.632549617957114)
--(axis cs:8,0.655358106646228)
--(axis cs:16,0.67075906416104)
--(axis cs:32,0.684257186253877)
--(axis cs:32,0.692208218919249)
--(axis cs:32,0.692208218919249)
--(axis cs:16,0.684075418469814)
--(axis cs:8,0.664433442820106)
--(axis cs:4,0.640790701080829)
--cycle;

\path [fill=mediumpurple148103189, fill opacity=0.2]
(axis cs:4,0.589910447416111)
--(axis cs:4,0.586156691260066)
--(axis cs:8,0.631834954046615)
--(axis cs:16,0.648690737769499)
--(axis cs:32,0.664850410998202)
--(axis cs:32,0.675176727515307)
--(axis cs:32,0.675176727515307)
--(axis cs:16,0.660078128942263)
--(axis cs:8,0.64793953331066)
--(axis cs:4,0.589910447416111)
--cycle;

\addplot [semithick, steelblue31119180]
table {%
4 0.667268449422098
8 0.675289357735403
16 0.701346758899657
32 0.704760897766969
};
\addlegendentry{$B=32$}
\addplot [semithick, darkorange25512714]
table {%
4 0.664806508182447
8 0.671181076602815
16 0.697137532979891
32 0.702306276605887
};
\addlegendentry{$B=64$}
\addplot [semithick, forestgreen4416044]
table {%
4 0.658144649271087
8 0.66545184992393
16 0.688312364923087
32 0.69558160961298
};
\addlegendentry{$B=128$}
\addplot [semithick, crimson2143940]
table {%
4 0.636670159518972
8 0.659895774733167
16 0.677417241315427
32 0.688232702586563
};
\addlegendentry{$B=256$}
\addplot [semithick, mediumpurple148103189]
table {%
4 0.588033569338089
8 0.639887243678638
16 0.654384433355881
32 0.670013569256754
};
\addlegendentry{$B=512$}
\end{axis}

\end{tikzpicture}}
        \end{subfigure}
        \begin{subfigure}{0.32\linewidth}
            \resizebox{\linewidth}{\linewidth}{
\begin{tikzpicture}

\definecolor{crimson2143940}{RGB}{214,39,40}
\definecolor{darkgray176}{RGB}{176,176,176}
\definecolor{darkorange25512714}{RGB}{255,127,14}
\definecolor{forestgreen4416044}{RGB}{44,160,44}
\definecolor{lightgray204}{RGB}{204,204,204}
\definecolor{mediumpurple148103189}{RGB}{148,103,189}
\definecolor{steelblue31119180}{RGB}{31,119,180}

\begin{axis}[
legend cell align={left},
legend style={
  fill opacity=0.8,
  draw opacity=1,
  text opacity=1,
  at={(0.97,0.03)},
  anchor=south east,
  draw=lightgray204
},
scale only axis,
width=4cm,
height=3cm,
tick align=outside,
tick pos=left,
title={GCN-C, Reddit},
tick label style={font=\tiny},
legend style={nodes={scale=0.5, transform shape}},
every tick/.style={
black,
semithick,
},
x label style={at={(axis description cs:0.5,-0.1)},anchor=north,font=\tiny},
y label style={at={(axis description cs:-0.1,.5)},rotate=90,anchor=south,font=\tiny},
x grid style={darkgray176},
xlabel={$\epsilon$},
xmin=2.6, xmax=33.4,
xtick style={color=black},
y grid style={darkgray176},
ylabel={ACC},
ymin=0.784170537257303, ymax=0.933870273368051,
ytick style={color=black}
]
\path [fill=steelblue31119180, fill opacity=0.2]
(axis cs:4,0.901205130589604)
--(axis cs:4,0.895829140455044)
--(axis cs:8,0.91111645676936)
--(axis cs:16,0.919163365854363)
--(axis cs:32,0.924299536647549)
--(axis cs:32,0.927065739908471)
--(axis cs:32,0.927065739908471)
--(axis cs:16,0.923085705111293)
--(axis cs:8,0.914465648323723)
--(axis cs:4,0.901205130589604)
--cycle;

\path [fill=darkorange25512714, fill opacity=0.2]
(axis cs:4,0.892838702927744)
--(axis cs:4,0.88907968566892)
--(axis cs:8,0.905409162413522)
--(axis cs:16,0.913719884370103)
--(axis cs:32,0.9225875670483)
--(axis cs:32,0.92542061922533)
--(axis cs:32,0.92542061922533)
--(axis cs:16,0.917201250936792)
--(axis cs:8,0.909013759152642)
--(axis cs:4,0.892838702927744)
--cycle;

\path [fill=forestgreen4416044, fill opacity=0.2]
(axis cs:4,0.883220509824135)
--(axis cs:4,0.877664900297402)
--(axis cs:8,0.893884424443282)
--(axis cs:16,0.909764709167118)
--(axis cs:32,0.915235970241586)
--(axis cs:32,0.918442646709027)
--(axis cs:32,0.918442646709027)
--(axis cs:16,0.912503391294257)
--(axis cs:8,0.897433816944076)
--(axis cs:4,0.883220509824135)
--cycle;

\path [fill=crimson2143940, fill opacity=0.2]
(axis cs:4,0.866485105222174)
--(axis cs:4,0.860976611381959)
--(axis cs:8,0.884583529905576)
--(axis cs:16,0.897946393650631)
--(axis cs:32,0.911471687134752)
--(axis cs:32,0.913877037350464)
--(axis cs:32,0.913877037350464)
--(axis cs:16,0.904369460073585)
--(axis cs:8,0.887579549282259)
--(axis cs:4,0.866485105222174)
--cycle;

\path [fill=mediumpurple148103189, fill opacity=0.2]
(axis cs:4,0.821573625044567)
--(axis cs:4,0.790975070716882)
--(axis cs:8,0.86750988794471)
--(axis cs:16,0.886009058274813)
--(axis cs:32,0.89948394198981)
--(axis cs:32,0.904451214105913)
--(axis cs:32,0.904451214105913)
--(axis cs:16,0.888810969371813)
--(axis cs:8,0.87651826134707)
--(axis cs:4,0.821573625044567)
--cycle;

\addplot [semithick, steelblue31119180]
table {%
4 0.898517135522324
8 0.912791052546541
16 0.921124535482828
32 0.92568263827801
};
\addlegendentry{$B=32$}
\addplot [semithick, darkorange25512714]
table {%
4 0.890959194298332
8 0.907211460783082
16 0.915460567653448
32 0.924004093136815
};
\addlegendentry{$B=64$}
\addplot [semithick, forestgreen4416044]
table {%
4 0.880442705060769
8 0.895659120693679
16 0.911134050230688
32 0.916839308475306
};
\addlegendentry{$B=128$}
\addplot [semithick, crimson2143940]
table {%
4 0.863730858302066
8 0.886081539593918
16 0.901157926862108
32 0.912674362242608
};
\addlegendentry{$B=256$}
\addplot [semithick, mediumpurple148103189]
table {%
4 0.806274347880725
8 0.87201407464589
16 0.887410013823313
32 0.901967578047861
};
\addlegendentry{$B=512$}
\end{axis}

\end{tikzpicture}}
        \end{subfigure}
        \begin{subfigure}{0.32\linewidth}
            \resizebox{\linewidth}{\linewidth}{
\begin{tikzpicture}

\definecolor{crimson2143940}{RGB}{214,39,40}
\definecolor{darkgray176}{RGB}{176,176,176}
\definecolor{darkorange25512714}{RGB}{255,127,14}
\definecolor{forestgreen4416044}{RGB}{44,160,44}
\definecolor{lightgray204}{RGB}{204,204,204}
\definecolor{mediumpurple148103189}{RGB}{148,103,189}
\definecolor{steelblue31119180}{RGB}{31,119,180}

\begin{axis}[
legend cell align={left},
legend style={
  fill opacity=0.8,
  draw opacity=1,
  text opacity=1,
  at={(0.97,0.03)},
  anchor=south east,
  draw=lightgray204
},
scale only axis,
width=4cm,
height=3cm,
tick align=outside,
tick pos=left,
title={GCN-C, Finance},
tick label style={font=\tiny},
legend style={nodes={scale=0.5, transform shape}},
every tick/.style={
black,
semithick,
},
x label style={at={(axis description cs:0.5,-0.1)},anchor=north,font=\tiny},
y label style={at={(axis description cs:-0.1,.5)},rotate=90,anchor=south,font=\tiny},
x grid style={darkgray176},
xlabel={$\epsilon$},
xmin=2.6, xmax=33.4,
xtick style={color=black},
y grid style={darkgray176},
ylabel={AUC},
ymin=0.671243575817462, ymax=0.74161759415126,
ytick style={color=black}
]
\path [fill=steelblue31119180, fill opacity=0.2]
(axis cs:4,0.713308700130452)
--(axis cs:4,0.695578977852054)
--(axis cs:8,0.704781784510618)
--(axis cs:16,0.719910308412774)
--(axis cs:32,0.714425147422215)
--(axis cs:32,0.734565121769204)
--(axis cs:32,0.734565121769204)
--(axis cs:16,0.731320865663228)
--(axis cs:8,0.723067670314412)
--(axis cs:4,0.713308700130452)
--cycle;

\path [fill=darkorange25512714, fill opacity=0.2]
(axis cs:4,0.70839626709129)
--(axis cs:4,0.692277979516865)
--(axis cs:8,0.699571530682937)
--(axis cs:16,0.715609827336004)
--(axis cs:32,0.71606570449207)
--(axis cs:32,0.738418775136088)
--(axis cs:32,0.738418775136088)
--(axis cs:16,0.729969284859488)
--(axis cs:8,0.72079289255505)
--(axis cs:4,0.70839626709129)
--cycle;

\path [fill=forestgreen4416044, fill opacity=0.2]
(axis cs:4,0.703227460751436)
--(axis cs:4,0.691804671974662)
--(axis cs:8,0.706913427000345)
--(axis cs:16,0.712626840839856)
--(axis cs:32,0.715895241392506)
--(axis cs:32,0.732725447795712)
--(axis cs:32,0.732725447795712)
--(axis cs:16,0.719248321473524)
--(axis cs:8,0.716053740701018)
--(axis cs:4,0.703227460751436)
--cycle;

\path [fill=crimson2143940, fill opacity=0.2]
(axis cs:4,0.704972425085373)
--(axis cs:4,0.688589781134879)
--(axis cs:8,0.697322619439467)
--(axis cs:16,0.703224600469374)
--(axis cs:32,0.705890440995683)
--(axis cs:32,0.7233400179115)
--(axis cs:32,0.7233400179115)
--(axis cs:16,0.71891231794628)
--(axis cs:8,0.706531200887055)
--(axis cs:4,0.704972425085373)
--cycle;

\path [fill=mediumpurple148103189, fill opacity=0.2]
(axis cs:4,0.690803028415663)
--(axis cs:4,0.674442394832635)
--(axis cs:8,0.693636163542217)
--(axis cs:16,0.691108045632734)
--(axis cs:32,0.70313692115139)
--(axis cs:32,0.718512128816626)
--(axis cs:32,0.718512128816626)
--(axis cs:16,0.710183066067861)
--(axis cs:8,0.700547265202993)
--(axis cs:4,0.690803028415663)
--cycle;

\addplot [semithick, steelblue31119180]
table {%
4 0.704443838991253
8 0.713924727412515
16 0.725615587038001
32 0.724495134595709
};
\addlegendentry{$B=32$}
\addplot [semithick, darkorange25512714]
table {%
4 0.700337123304077
8 0.710182211618993
16 0.722789556097746
32 0.727242239814079
};
\addlegendentry{$B=64$}
\addplot [semithick, forestgreen4416044]
table {%
4 0.697516066363049
8 0.711483583850682
16 0.71593758115669
32 0.724310344594109
};
\addlegendentry{$B=128$}
\addplot [semithick, crimson2143940]
table {%
4 0.696781103110126
8 0.701926910163261
16 0.711068459207827
32 0.714615229453592
};
\addlegendentry{$B=256$}
\addplot [semithick, mediumpurple148103189]
table {%
4 0.682622711624149
8 0.697091714372605
16 0.700645555850297
32 0.710824524984008
};
\addlegendentry{$B=512$}
\end{axis}

\end{tikzpicture}}
        \end{subfigure}
        \caption{Performance ($\text{mean}\pm\text{std}$ over $10$ trials) of \vesper under varying batch size $B$, using GCN aggregator and CONCAT decoder.}
        \label{fig: ablation_batch_size_gcn_c}
    \end{figure}

    \begin{figure}
        \centering
        \begin{subfigure}{0.32\linewidth}
            \resizebox{\linewidth}{\linewidth}{
\begin{tikzpicture}

\definecolor{crimson2143940}{RGB}{214,39,40}
\definecolor{darkgray176}{RGB}{176,176,176}
\definecolor{darkorange25512714}{RGB}{255,127,14}
\definecolor{forestgreen4416044}{RGB}{44,160,44}
\definecolor{lightgray204}{RGB}{204,204,204}
\definecolor{mediumpurple148103189}{RGB}{148,103,189}
\definecolor{steelblue31119180}{RGB}{31,119,180}

\begin{axis}[
legend cell align={left},
legend style={
  fill opacity=0.8,
  draw opacity=1,
  text opacity=1,
  at={(0.97,0.03)},
  anchor=south east,
  draw=lightgray204
},
scale only axis,
width=4cm,
height=3cm,
tick align=outside,
tick pos=left,
title={GIN-G, \ogb},
tick label style={font=\tiny},
legend style={nodes={scale=0.5, transform shape}},
every tick/.style={
black,
semithick,
},
x label style={at={(axis description cs:0.5,-0.1)},anchor=north,font=\tiny},
y label style={at={(axis description cs:-0.1,.5)},rotate=90,anchor=south,font=\tiny},
x grid style={darkgray176},
xlabel={$\epsilon$},
xmin=2.6, xmax=33.4,
xtick style={color=black},
y grid style={darkgray176},
ylabel={ACC},
ymin=0.282744112478169, ymax=0.783934289721838,
ytick style={color=black}
]
\path [fill=steelblue31119180, fill opacity=0.2]
(axis cs:4,0.701582728714592)
--(axis cs:4,0.689377606852268)
--(axis cs:8,0.73147129051379)
--(axis cs:16,0.748734374516823)
--(axis cs:32,0.751006771744319)
--(axis cs:32,0.761152918028944)
--(axis cs:32,0.761152918028944)
--(axis cs:16,0.758467182039143)
--(axis cs:8,0.737453123348993)
--(axis cs:4,0.701582728714592)
--cycle;

\path [fill=darkorange25512714, fill opacity=0.2]
(axis cs:4,0.635492043452123)
--(axis cs:4,0.612487727827052)
--(axis cs:8,0.712606911053162)
--(axis cs:16,0.738170483072885)
--(axis cs:32,0.746892409981435)
--(axis cs:32,0.7589155301349)
--(axis cs:32,0.7589155301349)
--(axis cs:16,0.745830852615526)
--(axis cs:8,0.722523501568822)
--(axis cs:4,0.635492043452123)
--cycle;

\path [fill=forestgreen4416044, fill opacity=0.2]
(axis cs:4,0.536689041863883)
--(axis cs:4,0.497170840084035)
--(axis cs:8,0.661895984000062)
--(axis cs:16,0.724373258189479)
--(axis cs:32,0.744353222238335)
--(axis cs:32,0.755634803559068)
--(axis cs:32,0.755634803559068)
--(axis cs:16,0.730264260105069)
--(axis cs:8,0.676980456236693)
--(axis cs:4,0.536689041863883)
--cycle;

\path [fill=crimson2143940, fill opacity=0.2]
(axis cs:4,0.447346104436077)
--(axis cs:4,0.39623122699765)
--(axis cs:8,0.580693501673978)
--(axis cs:16,0.690922128696545)
--(axis cs:32,0.728091665833389)
--(axis cs:32,0.739016826316279)
--(axis cs:32,0.739016826316279)
--(axis cs:16,0.704889611534652)
--(axis cs:8,0.60190631008252)
--(axis cs:4,0.447346104436077)
--cycle;

\path [fill=mediumpurple148103189, fill opacity=0.2]
(axis cs:4,0.344841536434958)
--(axis cs:4,0.305525484171063)
--(axis cs:8,0.422123768972078)
--(axis cs:16,0.632916909937599)
--(axis cs:32,0.708004656329101)
--(axis cs:32,0.717170810924618)
--(axis cs:32,0.717170810924618)
--(axis cs:16,0.646117797627522)
--(axis cs:8,0.48813405594339)
--(axis cs:4,0.344841536434958)
--cycle;

\addplot [semithick, steelblue31119180]
table {%
4 0.69548016778343
8 0.734462206931392
16 0.753600778277983
32 0.756079844886631
};
\addlegendentry{$B=32$}
\addplot [semithick, darkorange25512714]
table {%
4 0.623989885639587
8 0.717565206310992
16 0.742000667844205
32 0.752903970058168
};
\addlegendentry{$B=64$}
\addplot [semithick, forestgreen4416044]
table {%
4 0.516929940973959
8 0.669438220118378
16 0.727318759147274
32 0.749994012898701
};
\addlegendentry{$B=128$}
\addplot [semithick, crimson2143940]
table {%
4 0.421788665716864
8 0.591299905878249
16 0.697905870115598
32 0.733554246074834
};
\addlegendentry{$B=256$}
\addplot [semithick, mediumpurple148103189]
table {%
4 0.325183510303011
8 0.455128912457734
16 0.63951735378256
32 0.712587733626859
};
\addlegendentry{$B=512$}
\end{axis}

\end{tikzpicture}}
        \end{subfigure}
        \begin{subfigure}{0.32\linewidth}
            \resizebox{\linewidth}{\linewidth}{
\begin{tikzpicture}

\definecolor{crimson2143940}{RGB}{214,39,40}
\definecolor{darkgray176}{RGB}{176,176,176}
\definecolor{darkorange25512714}{RGB}{255,127,14}
\definecolor{forestgreen4416044}{RGB}{44,160,44}
\definecolor{lightgray204}{RGB}{204,204,204}
\definecolor{mediumpurple148103189}{RGB}{148,103,189}
\definecolor{steelblue31119180}{RGB}{31,119,180}

\begin{axis}[
legend cell align={left},
legend style={
  fill opacity=0.8,
  draw opacity=1,
  text opacity=1,
  at={(0.97,0.03)},
  anchor=south east,
  draw=lightgray204
},
scale only axis,
width=4cm,
height=3cm,
tick align=outside,
tick pos=left,
title={GIN-G, Reddit},
tick label style={font=\tiny},
legend style={nodes={scale=0.5, transform shape}},
every tick/.style={
black,
semithick,
},
x label style={at={(axis description cs:0.5,-0.1)},anchor=north,font=\tiny},
y label style={at={(axis description cs:-0.1,.5)},rotate=90,anchor=south,font=\tiny},
x grid style={darkgray176},
xlabel={$\epsilon$},
xmin=2.6, xmax=33.4,
xtick style={color=black},
y grid style={darkgray176},
ylabel={ACC},
ymin=0.116411393461281, ymax=0.972430074288497,
ytick style={color=black}
]
\path [fill=steelblue31119180, fill opacity=0.2]
(axis cs:4,0.780824927068668)
--(axis cs:4,0.728691616025959)
--(axis cs:8,0.89155721174369)
--(axis cs:16,0.915339149918999)
--(axis cs:32,0.925424626354457)
--(axis cs:32,0.933520134250897)
--(axis cs:32,0.933520134250897)
--(axis cs:16,0.924910028760785)
--(axis cs:8,0.900339095457)
--(axis cs:4,0.780824927068668)
--cycle;

\path [fill=darkorange25512714, fill opacity=0.2]
(axis cs:4,0.612871749093459)
--(axis cs:4,0.54830448917019)
--(axis cs:8,0.828246053310321)
--(axis cs:16,0.900077212279308)
--(axis cs:32,0.922553270364535)
--(axis cs:32,0.929081291508255)
--(axis cs:32,0.929081291508255)
--(axis cs:16,0.912852073396508)
--(axis cs:8,0.848809042465489)
--(axis cs:4,0.612871749093459)
--cycle;

\path [fill=forestgreen4416044, fill opacity=0.2]
(axis cs:4,0.420333822430692)
--(axis cs:4,0.361523528160838)
--(axis cs:8,0.685163991228097)
--(axis cs:16,0.88189881661134)
--(axis cs:32,0.916836969861053)
--(axis cs:32,0.920396213270915)
--(axis cs:32,0.920396213270915)
--(axis cs:16,0.892910439622615)
--(axis cs:8,0.755641710439677)
--(axis cs:4,0.420333822430692)
--cycle;

\path [fill=crimson2143940, fill opacity=0.2]
(axis cs:4,0.389848214025561)
--(axis cs:4,0.254447425347543)
--(axis cs:8,0.452544037617359)
--(axis cs:16,0.826126916752426)
--(axis cs:32,0.896426696918679)
--(axis cs:32,0.912879803628868)
--(axis cs:32,0.912879803628868)
--(axis cs:16,0.854985411129376)
--(axis cs:8,0.643967819912756)
--(axis cs:4,0.389848214025561)
--cycle;

\path [fill=mediumpurple148103189, fill opacity=0.2]
(axis cs:4,0.289135877064283)
--(axis cs:4,0.155321333498882)
--(axis cs:8,0.39942956818402)
--(axis cs:16,0.68680012671411)
--(axis cs:32,0.883533799272405)
--(axis cs:32,0.90104871872483)
--(axis cs:32,0.90104871872483)
--(axis cs:16,0.768812676905084)
--(axis cs:8,0.514072397598792)
--(axis cs:4,0.289135877064283)
--cycle;

\addplot [semithick, steelblue31119180]
table {%
4 0.754758271547313
8 0.895948153600345
16 0.920124589339892
32 0.929472380302677
};
\addlegendentry{$B=32$}
\addplot [semithick, darkorange25512714]
table {%
4 0.580588119131824
8 0.838527547887905
16 0.906464642837908
32 0.925817280936395
};
\addlegendentry{$B=64$}
\addplot [semithick, forestgreen4416044]
table {%
4 0.390928675295765
8 0.720402850833887
16 0.887404628116978
32 0.918616591565984
};
\addlegendentry{$B=128$}
\addplot [semithick, crimson2143940]
table {%
4 0.322147819686552
8 0.548255928765058
16 0.840556163940901
32 0.904653250273773
};
\addlegendentry{$B=256$}
\addplot [semithick, mediumpurple148103189]
table {%
4 0.222228605281583
8 0.456750982891406
16 0.727806401809597
32 0.892291258998618
};
\addlegendentry{$B=512$}
\end{axis}

\end{tikzpicture}}
        \end{subfigure}
        \begin{subfigure}{0.32\linewidth}
            \resizebox{\linewidth}{\linewidth}{
\begin{tikzpicture}

\definecolor{crimson2143940}{RGB}{214,39,40}
\definecolor{darkgray176}{RGB}{176,176,176}
\definecolor{darkorange25512714}{RGB}{255,127,14}
\definecolor{forestgreen4416044}{RGB}{44,160,44}
\definecolor{lightgray204}{RGB}{204,204,204}
\definecolor{mediumpurple148103189}{RGB}{148,103,189}
\definecolor{steelblue31119180}{RGB}{31,119,180}

\begin{axis}[
legend cell align={left},
legend style={
  fill opacity=0.8,
  draw opacity=1,
  text opacity=1,
  at={(0.97,0.03)},
  anchor=south east,
  draw=lightgray204
},
scale only axis,
width=4cm,
height=3cm,
tick align=outside,
tick pos=left,
title={GIN-G, Finance},
tick label style={font=\tiny},
legend style={nodes={scale=0.5, transform shape}},
every tick/.style={
black,
semithick,
},
x label style={at={(axis description cs:0.5,-0.1)},anchor=north,font=\tiny},
y label style={at={(axis description cs:-0.1,.5)},rotate=90,anchor=south,font=\tiny},
x grid style={darkgray176},
xlabel={$\epsilon$},
xmin=2.6, xmax=33.4,
xtick style={color=black},
y grid style={darkgray176},
ylabel={AUC},
ymin=0.42306427228965, ymax=0.705327395387387,
ytick style={color=black}
]
\path [fill=steelblue31119180, fill opacity=0.2]
(axis cs:4,0.637676014403577)
--(axis cs:4,0.520234803938276)
--(axis cs:8,0.616278376182362)
--(axis cs:16,0.58861944752177)
--(axis cs:32,0.577236026051426)
--(axis cs:32,0.660531682039558)
--(axis cs:32,0.660531682039558)
--(axis cs:16,0.646530652477468)
--(axis cs:8,0.662968186957879)
--(axis cs:4,0.637676014403577)
--cycle;

\path [fill=darkorange25512714, fill opacity=0.2]
(axis cs:4,0.656195335189903)
--(axis cs:4,0.521713709845578)
--(axis cs:8,0.628470364000997)
--(axis cs:16,0.609652747935221)
--(axis cs:32,0.601848282264333)
--(axis cs:32,0.668305033183523)
--(axis cs:32,0.668305033183523)
--(axis cs:16,0.660783602637624)
--(axis cs:8,0.653706038965636)
--(axis cs:4,0.656195335189903)
--cycle;

\path [fill=forestgreen4416044, fill opacity=0.2]
(axis cs:4,0.660560912406391)
--(axis cs:4,0.456958101046531)
--(axis cs:8,0.520008449420258)
--(axis cs:16,0.624487767317517)
--(axis cs:32,0.613030726690395)
--(axis cs:32,0.663933164634128)
--(axis cs:32,0.663933164634128)
--(axis cs:16,0.667987829173884)
--(axis cs:8,0.670881228042695)
--(axis cs:4,0.660560912406391)
--cycle;

\path [fill=crimson2143940, fill opacity=0.2]
(axis cs:4,0.656576886354846)
--(axis cs:4,0.435894414248638)
--(axis cs:8,0.518206650474516)
--(axis cs:16,0.624899847372887)
--(axis cs:32,0.63819847820802)
--(axis cs:32,0.665605630703848)
--(axis cs:32,0.665605630703848)
--(axis cs:16,0.67100702793054)
--(axis cs:8,0.656412398993672)
--(axis cs:4,0.656576886354846)
--cycle;

\path [fill=mediumpurple148103189, fill opacity=0.2]
(axis cs:4,0.66209156895582)
--(axis cs:4,0.505803897416658)
--(axis cs:8,0.621752060479667)
--(axis cs:16,0.643608619994545)
--(axis cs:32,0.64952683682845)
--(axis cs:32,0.692497253428399)
--(axis cs:32,0.692497253428399)
--(axis cs:16,0.677473112021095)
--(axis cs:8,0.666962099922711)
--(axis cs:4,0.66209156895582)
--cycle;

\addplot [semithick, steelblue31119180]
table {%
4 0.578955409170926
8 0.639623281570121
16 0.617575049999619
32 0.618883854045492
};
\addlegendentry{$B=32$}
\addplot [semithick, darkorange25512714]
table {%
4 0.588954522517741
8 0.641088201483316
16 0.635218175286422
32 0.635076657723928
};
\addlegendentry{$B=64$}
\addplot [semithick, forestgreen4416044]
table {%
4 0.558759506726461
8 0.595444838731477
16 0.646237798245701
32 0.638481945662262
};
\addlegendentry{$B=128$}
\addplot [semithick, crimson2143940]
table {%
4 0.546235650301742
8 0.587309524734094
16 0.647953437651714
32 0.651902054455934
};
\addlegendentry{$B=256$}
\addplot [semithick, mediumpurple148103189]
table {%
4 0.583947733186239
8 0.644357080201189
16 0.66054086600782
32 0.671012045128424
};
\addlegendentry{$B=512$}
\end{axis}

\end{tikzpicture}}
        \end{subfigure}
        \caption{Performance ($\text{mean}\pm\text{std}$ over $10$ trials) of \vesper under varying batch size $B$, using GIN aggregator and GRU decoder.}
        \label{fig: ablation_batch_size_gin_g}
    \end{figure}
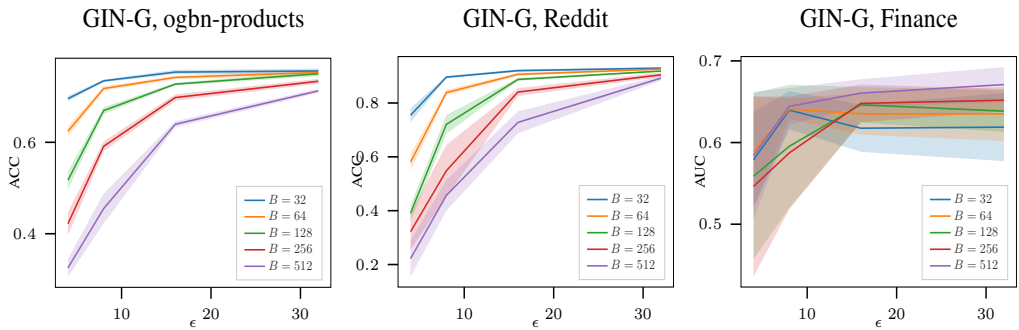
    \begin{figure}
        \centering
        \begin{subfigure}{0.32\linewidth}
            \resizebox{\linewidth}{\linewidth}{
\begin{tikzpicture}

\definecolor{crimson2143940}{RGB}{214,39,40}
\definecolor{darkgray176}{RGB}{176,176,176}
\definecolor{darkorange25512714}{RGB}{255,127,14}
\definecolor{forestgreen4416044}{RGB}{44,160,44}
\definecolor{lightgray204}{RGB}{204,204,204}
\definecolor{mediumpurple148103189}{RGB}{148,103,189}
\definecolor{steelblue31119180}{RGB}{31,119,180}

\begin{axis}[
legend cell align={left},
legend style={
  fill opacity=0.8,
  draw opacity=1,
  text opacity=1,
  at={(0.97,0.03)},
  anchor=south east,
  draw=lightgray204
},
scale only axis,
width=4cm,
height=3cm,
tick align=outside,
tick pos=left,
title={GCN-G, \ogb},
tick label style={font=\tiny},
legend style={nodes={scale=0.5, transform shape}},
every tick/.style={
black,
semithick,
},
x label style={at={(axis description cs:0.5,-0.1)},anchor=north,font=\tiny},
y label style={at={(axis description cs:-0.1,.5)},rotate=90,anchor=south,font=\tiny},
x grid style={darkgray176},
xlabel={$\epsilon$},
xmin=2.6, xmax=33.4,
xtick style={color=black},
y grid style={darkgray176},
ylabel={ACC},
ymin=0.58996503252521, ymax=0.715253187311149,
ytick style={color=black}
]
\path [fill=steelblue31119180, fill opacity=0.2]
(axis cs:4,0.680000881786496)
--(axis cs:4,0.672056489555216)
--(axis cs:8,0.680185788219953)
--(axis cs:16,0.695438475566557)
--(axis cs:32,0.70287122222538)
--(axis cs:32,0.709558271184516)
--(axis cs:32,0.709558271184516)
--(axis cs:16,0.704988438645285)
--(axis cs:8,0.693412676551717)
--(axis cs:4,0.680000881786496)
--cycle;

\path [fill=darkorange25512714, fill opacity=0.2]
(axis cs:4,0.665242991775001)
--(axis cs:4,0.654021692777058)
--(axis cs:8,0.675782110024301)
--(axis cs:16,0.684495848132512)
--(axis cs:32,0.692475600926922)
--(axis cs:32,0.70359369762429)
--(axis cs:32,0.70359369762429)
--(axis cs:16,0.694929037694596)
--(axis cs:8,0.686355913220111)
--(axis cs:4,0.665242991775001)
--cycle;

\path [fill=forestgreen4416044, fill opacity=0.2]
(axis cs:4,0.651930300205159)
--(axis cs:4,0.643255392565722)
--(axis cs:8,0.655878260864903)
--(axis cs:16,0.674177336173325)
--(axis cs:32,0.684511731840808)
--(axis cs:32,0.697389690197418)
--(axis cs:32,0.697389690197418)
--(axis cs:16,0.684313390145656)
--(axis cs:8,0.667619326898094)
--(axis cs:4,0.651930300205159)
--cycle;

\path [fill=crimson2143940, fill opacity=0.2]
(axis cs:4,0.629754859475903)
--(axis cs:4,0.613659532431163)
--(axis cs:8,0.645364466704493)
--(axis cs:16,0.651958057886733)
--(axis cs:32,0.673275439654511)
--(axis cs:32,0.682724110296214)
--(axis cs:32,0.682724110296214)
--(axis cs:16,0.668745519146475)
--(axis cs:8,0.655059781552809)
--(axis cs:4,0.629754859475903)
--cycle;

\path [fill=mediumpurple148103189, fill opacity=0.2]
(axis cs:4,0.609543994611222)
--(axis cs:4,0.595659948651843)
--(axis cs:8,0.620225301817304)
--(axis cs:16,0.640277262103055)
--(axis cs:32,0.650951674521644)
--(axis cs:32,0.665719985161577)
--(axis cs:32,0.665719985161577)
--(axis cs:16,0.647871846993679)
--(axis cs:8,0.630844084845964)
--(axis cs:4,0.609543994611222)
--cycle;

\addplot [semithick, steelblue31119180]
table {%
4 0.676028685670856
8 0.686799232385835
16 0.700213457105921
32 0.706214746704948
};
\addlegendentry{$B=32$}
\addplot [semithick, darkorange25512714]
table {%
4 0.659632342276029
8 0.681069011622206
16 0.689712442913554
32 0.698034649275606
};
\addlegendentry{$B=64$}
\addplot [semithick, forestgreen4416044]
table {%
4 0.64759284638544
8 0.661748793881499
16 0.67924536315949
32 0.690950711019113
};
\addlegendentry{$B=128$}
\addplot [semithick, crimson2143940]
table {%
4 0.621707195953533
8 0.650212124128651
16 0.660351788516604
32 0.677999774975363
};
\addlegendentry{$B=256$}
\addplot [semithick, mediumpurple148103189]
table {%
4 0.602601971631533
8 0.625534693331634
16 0.644074554548367
32 0.658335829841611
};
\addlegendentry{$B=512$}
\end{axis}

\end{tikzpicture}}
        \end{subfigure}
        \begin{subfigure}{0.32\linewidth}
            \resizebox{\linewidth}{\linewidth}{
\begin{tikzpicture}

\definecolor{crimson2143940}{RGB}{214,39,40}
\definecolor{darkgray176}{RGB}{176,176,176}
\definecolor{darkorange25512714}{RGB}{255,127,14}
\definecolor{forestgreen4416044}{RGB}{44,160,44}
\definecolor{lightgray204}{RGB}{204,204,204}
\definecolor{mediumpurple148103189}{RGB}{148,103,189}
\definecolor{steelblue31119180}{RGB}{31,119,180}

\begin{axis}[
legend cell align={left},
legend style={
  fill opacity=0.8,
  draw opacity=1,
  text opacity=1,
  at={(0.97,0.03)},
  anchor=south east,
  draw=lightgray204
},
scale only axis,
width=4cm,
height=3cm,
tick align=outside,
tick pos=left,
title={GCN-G, Reddit},
tick label style={font=\tiny},
legend style={nodes={scale=0.5, transform shape}},
every tick/.style={
black,
semithick,
},
x label style={at={(axis description cs:0.5,-0.1)},anchor=north,font=\tiny},
y label style={at={(axis description cs:-0.1,.5)},rotate=90,anchor=south,font=\tiny},
x grid style={darkgray176},
xlabel={$\epsilon$},
xmin=2.6, xmax=33.4,
xtick style={color=black},
y grid style={darkgray176},
ylabel={ACC},
ymin=0.688451859594446, ymax=0.937311666578991,
ytick style={color=black}
]
\path [fill=steelblue31119180, fill opacity=0.2]
(axis cs:4,0.896485500605356)
--(axis cs:4,0.894348027211817)
--(axis cs:8,0.907118899049947)
--(axis cs:16,0.91272935596653)
--(axis cs:32,0.922284795361577)
--(axis cs:32,0.925999857170603)
--(axis cs:32,0.925999857170603)
--(axis cs:16,0.916723276746251)
--(axis cs:8,0.913249842310482)
--(axis cs:4,0.896485500605356)
--cycle;

\path [fill=darkorange25512714, fill opacity=0.2]
(axis cs:4,0.890918184250263)
--(axis cs:4,0.884526585331267)
--(axis cs:8,0.898114476735784)
--(axis cs:16,0.910362572598692)
--(axis cs:32,0.919153345528558)
--(axis cs:32,0.922761811644306)
--(axis cs:32,0.922761811644306)
--(axis cs:16,0.914724047511535)
--(axis cs:8,0.905497536979814)
--(axis cs:4,0.890918184250263)
--cycle;

\path [fill=forestgreen4416044, fill opacity=0.2]
(axis cs:4,0.881073347673974)
--(axis cs:4,0.870408619185979)
--(axis cs:8,0.890747608685765)
--(axis cs:16,0.902784907386804)
--(axis cs:32,0.910405415758881)
--(axis cs:32,0.916052764231424)
--(axis cs:32,0.916052764231424)
--(axis cs:16,0.909214446328436)
--(axis cs:8,0.894057518506667)
--(axis cs:4,0.881073347673974)
--cycle;

\path [fill=crimson2143940, fill opacity=0.2]
(axis cs:4,0.83609117429827)
--(axis cs:4,0.817489422797039)
--(axis cs:8,0.88071597858268)
--(axis cs:16,0.890891881559714)
--(axis cs:32,0.903521682946296)
--(axis cs:32,0.908829536916189)
--(axis cs:32,0.908829536916189)
--(axis cs:16,0.895112462910063)
--(axis cs:8,0.885034519595156)
--(axis cs:4,0.83609117429827)
--cycle;

\path [fill=mediumpurple148103189, fill opacity=0.2]
(axis cs:4,0.704954209746964)
--(axis cs:4,0.699763669002834)
--(axis cs:8,0.844111578906199)
--(axis cs:16,0.875383397003115)
--(axis cs:32,0.889029772149004)
--(axis cs:32,0.898755445882341)
--(axis cs:32,0.898755445882341)
--(axis cs:16,0.884949080601323)
--(axis cs:8,0.875113597475683)
--(axis cs:4,0.704954209746964)
--cycle;

\addplot [semithick, steelblue31119180]
table {%
4 0.895416763908586
8 0.910184370680215
16 0.91472631635639
32 0.92414232626609
};
\addlegendentry{$B=32$}
\addplot [semithick, darkorange25512714]
table {%
4 0.887722384790765
8 0.901806006857799
16 0.912543310055114
32 0.920957578586432
};
\addlegendentry{$B=64$}
\addplot [semithick, forestgreen4416044]
table {%
4 0.875740983429977
8 0.892402563596216
16 0.90599967685762
32 0.913229089995153
};
\addlegendentry{$B=128$}
\addplot [semithick, crimson2143940]
table {%
4 0.826790298547655
8 0.882875249088918
16 0.893002172234888
32 0.906175609931243
};
\addlegendentry{$B=256$}
\addplot [semithick, mediumpurple148103189]
table {%
4 0.702358939374899
8 0.859612588190941
16 0.880166238802219
32 0.893892609015672
};
\addlegendentry{$B=512$}
\end{axis}

\end{tikzpicture}}
        \end{subfigure}
        \begin{subfigure}{0.32\linewidth}
            \resizebox{\linewidth}{\linewidth}{
\begin{tikzpicture}

\definecolor{crimson2143940}{RGB}{214,39,40}
\definecolor{darkgray176}{RGB}{176,176,176}
\definecolor{darkorange25512714}{RGB}{255,127,14}
\definecolor{forestgreen4416044}{RGB}{44,160,44}
\definecolor{lightgray204}{RGB}{204,204,204}
\definecolor{mediumpurple148103189}{RGB}{148,103,189}
\definecolor{steelblue31119180}{RGB}{31,119,180}

\begin{axis}[
legend cell align={left},
legend style={
  fill opacity=0.8,
  draw opacity=1,
  text opacity=1,
  at={(0.97,0.03)},
  anchor=south east,
  draw=lightgray204
},
scale only axis,
width=4cm,
height=3cm,
tick align=outside,
tick pos=left,
title={GCN-G, Finance},
tick label style={font=\tiny},
legend style={nodes={scale=0.5, transform shape}},
every tick/.style={
black,
semithick,
},
x label style={at={(axis description cs:0.5,-0.1)},anchor=north,font=\tiny},
y label style={at={(axis description cs:-0.1,.5)},rotate=90,anchor=south,font=\tiny},
x grid style={darkgray176},
xlabel={$\epsilon$},
xmin=2.6, xmax=33.4,
xtick style={color=black},
y grid style={darkgray176},
ylabel={AUC},
ymin=0.695563722559968, ymax=0.752970160552511,
ytick style={color=black}
]
\path [fill=steelblue31119180, fill opacity=0.2]
(axis cs:4,0.742675868030323)
--(axis cs:4,0.713714547986529)
--(axis cs:8,0.709982935570472)
--(axis cs:16,0.70065010299721)
--(axis cs:32,0.717999931297702)
--(axis cs:32,0.74513203902536)
--(axis cs:32,0.74513203902536)
--(axis cs:16,0.748191750686598)
--(axis cs:8,0.74793602374935)
--(axis cs:4,0.742675868030323)
--cycle;

\path [fill=darkorange25512714, fill opacity=0.2]
(axis cs:4,0.743018694014603)
--(axis cs:4,0.728438173641977)
--(axis cs:8,0.725658781083547)
--(axis cs:16,0.723960198891379)
--(axis cs:32,0.710517214564141)
--(axis cs:32,0.745099412101337)
--(axis cs:32,0.745099412101337)
--(axis cs:16,0.743790515605025)
--(axis cs:8,0.741721575930431)
--(axis cs:4,0.743018694014603)
--cycle;

\path [fill=forestgreen4416044, fill opacity=0.2]
(axis cs:4,0.746460346895058)
--(axis cs:4,0.723346560232252)
--(axis cs:8,0.731523147425954)
--(axis cs:16,0.727130327602962)
--(axis cs:32,0.736300784782351)
--(axis cs:32,0.748437197807644)
--(axis cs:32,0.748437197807644)
--(axis cs:16,0.750360777007396)
--(axis cs:8,0.747717439163771)
--(axis cs:4,0.746460346895058)
--cycle;

\path [fill=crimson2143940, fill opacity=0.2]
(axis cs:4,0.740409195208341)
--(axis cs:4,0.708793419949734)
--(axis cs:8,0.721654779708114)
--(axis cs:16,0.733753921838726)
--(axis cs:32,0.742662609660882)
--(axis cs:32,0.749757925331492)
--(axis cs:32,0.749757925331492)
--(axis cs:16,0.744160500084317)
--(axis cs:8,0.736127960021465)
--(axis cs:4,0.740409195208341)
--cycle;

\path [fill=mediumpurple148103189, fill opacity=0.2]
(axis cs:4,0.716203346403553)
--(axis cs:4,0.698173106105083)
--(axis cs:8,0.721479402140716)
--(axis cs:16,0.720426555525513)
--(axis cs:32,0.722085183479948)
--(axis cs:32,0.739550440705964)
--(axis cs:32,0.739550440705964)
--(axis cs:16,0.744668212608048)
--(axis cs:8,0.736047210249402)
--(axis cs:4,0.716203346403553)
--cycle;

\addplot [semithick, steelblue31119180]
table {%
4 0.728195208008426
8 0.728959479659911
16 0.724420926841904
32 0.731565985161531
};
\addlegendentry{$B=32$}
\addplot [semithick, darkorange25512714]
table {%
4 0.73572843382829
8 0.733690178506989
16 0.733875357248202
32 0.727808313332739
};
\addlegendentry{$B=64$}
\addplot [semithick, forestgreen4416044]
table {%
4 0.734903453563655
8 0.739620293294862
16 0.738745552305179
32 0.742368991294997
};
\addlegendentry{$B=128$}
\addplot [semithick, crimson2143940]
table {%
4 0.724601307579037
8 0.728891369864789
16 0.738957210961521
32 0.746210267496187
};
\addlegendentry{$B=256$}
\addplot [semithick, mediumpurple148103189]
table {%
4 0.707188226254318
8 0.728763306195059
16 0.732547384066781
32 0.730817812092956
};
\addlegendentry{$B=512$}
\end{axis}

\end{tikzpicture}}
        \end{subfigure}
        \caption{Performance ($\text{mean}\pm\text{std}$ over $10$ trials) of \vesper under varying batch size $B$, using GCN aggregator and GRU decoder.}
        \label{fig: ablation_batch_size_gcn_g}
    \end{figure}
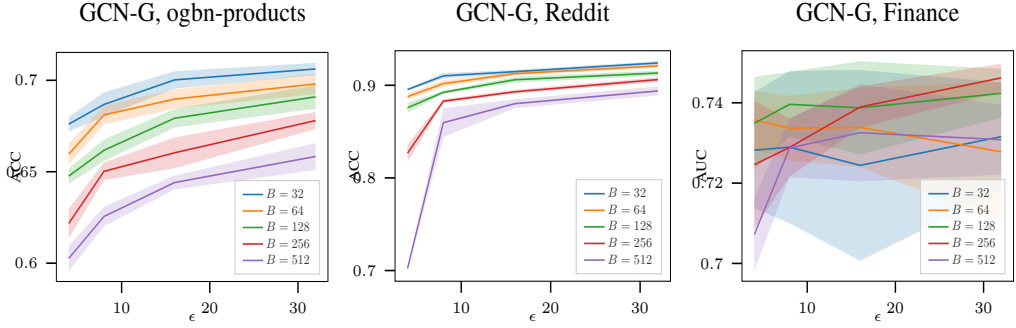

    \section{Discussions}\label{sec: extensions}
    In this section, we discuss two extensions of the proposed framework regarding the threat model and the privacy model.\par\noindent
    \subsection{Beyond one-sided adversary} 
    In this paper, we are mainly interested in the threat model with only one malicious party B which possesses only label data. Such a one-sided threat model might be further extended to a more complicated setup where party B owns not only label data, but also its own node features and graph structure and allows party A to be a semi-honest adversary to infer party B's label information and graph structure. We provide a straightforward solution to this extended scenario. Specifically, party B needs to protect the edge privacy of its local graph as well as the precise label data from being reconstructed by Party A. For the former privacy requirement, party B may apply the PMP framework to its local graph representation learning procedure. For protecting label data, we can simply adopt the randomized response technique for \emph{label differential privacy} \cite{ghazi2021deep} during the loss computation step, and deploy three accountants with one accounts with the privacy budget corresponding to party A that is identical to the one used in this paper, and the other two accountant tracking the cumulative privacy cost incurred by PMP and randomized response mechanisms conducted by party B, which is trivial to analyze \cite{dwork2014algorithmic}. Moreover, extending the above scenario to more than two parties is also technically feasible using similar algorithmic procedures. In this paper, we did not examine such kinds of scenarios empirically since there are no publicly available graph VFL datasets that provide natural feature/graph splits between multiple parties. \par\noindent
    
    \subsection{On extensions to node-level DP} 
    The node-level differential privacy (node DP) model \cite{kasiviswanathan2013analyzing} is a strictly stronger notion of privacy than the edge-level DP model regarding graph-input queries. In particular, node DP is analogously defined as in definition \ref{def: edgeDP} under the approximate $(\epsilon, \delta)$-DP model with the \emph{adjacency relation} modified in the sense that two graphs $G, G^\prime$ are node-level adjacent if $G$ could be edited into $G^\prime$ via adding or removing one node as well as its adjacent edges. According to its original proposal, node DP targets the \emph{protection of node memberships}, which is somewhat subtle under the VFL context since both party A and party B know the participating nodes' identities throughout the VFL process. Nevertheless, it is still possible to use additive noise perturbation to guarantee that the node embeddings are probabilistically similar with or without the participation of some specific nodes during the message-passing procedure. Formally, we establish the node DP guarantee of PMP without neighborhood sampling in the following theorem:
    \begin{theorem}[RDP guarantee of PMP under node DP, non-sampling version]\label{thm: nodeDP}
        For graphs with bounded maximum degree $D$, the released output of the entire graph $\mathbf{H}_L$ in algorithm \ref{alg: priv_gnn_sample} without neighborhood sampling is $\left(\alpha, \frac{\alpha \sum_{l = 1}^L (1 + \sqrt{D}\mathcal{S}_l)^2}{2\theta^2}\right)$-\renyi differentially private for any $\alpha > 1$.
    \end{theorem}
        \begin{proof}
        From the proof of theorem \ref{thm: sampled}, it suffices to show that the node sensitivity defined as 
        \begin{align}
            \mathcal{S}_{\text{n}} &= \max_{G, G^\prime}\sqrt{\sum_{v \in V\setminus\{u^*\}}\|h_v - h^\prime_v\|_2^2 + \|h_{u^*}\|_2^2}
        \end{align}
        is bounded from above by $1 + \sqrt{D}\mathcal{S}_l$ in the $l$-th layer. Note that removing a node $v^*$ and all its adjacent edges affects two parts of node embeddings: $h_{v^*}$ and $\{h_u, u \in N(v^*)\|$, leaving the rest embeddings untouched. Therefore we bound both using corresponding upper bounds:
        \begin{align*}
            \mathcal{S}_{\text{n}} &\le \sqrt{\max_{G, G^\prime}\sum_{v \in N(v^*)}\|h_v - h^\prime_v\|_2^2 + \max_{G}\|h_{u^*}\|_2^2} \\
            &\le \sqrt{\sum_{v \in N(v^*)}\max_{G, G^\prime}\|h_v - h^\prime_v\|_2^2 + 1} \\
            &\le 1 + \sqrt{D}\mathcal{S}_{l}
        \end{align*}
    \end{proof}
    Theorem \ref{thm: nodeDP} implies that, under general PMP mechanisms, the node-level privacy guarantee becomes much weaker than that of edge-level by a factor of the order $O(\sqrt{D})$. Moreover, in stochastic training paradigms, the sampling amplification phenomenon is also weaker than theorem \ref{alg: priv_gnn_sample} \cite[Theorem 1]{daigavane2021node}. In our experiments, we find node DP guarantee to be overly stringent which produces meaningless results under moderate privacy budgets. Therefore we report only edge DP results in this paper. However, as illustrated in section \label{sec: mia} we empirically investigated the protection of applying edge-level private mechanisms against node-level membership inference adversaries and the results are confirmatory, this serves as empirical evidence that edge-level privacy might be adequate for reasonable privacy protection rather than sticking to node-level DP definitions.
    \subsection{Beyond GIN and GCN aggregators}\label{sec: other_gnn}
    In this section, we discuss possible extensions of the PMP framework into other aggregation schemes. Throughout the discussion, we adopt the ReLU function as the default nonlinearity and focus mainly on the aggregation step.
    \subsubsection{On max-pooling aggregation of SAGE \cite{hamilton2017inductive}}
    Technically, the case of max-pooling does not directly fit into the message passing form in \eqref{eqn: gnn}, due to the fact that the max-pooling operation is applied along each coordinate, or:
    \begin{align}
        \left[\widetilde{h}^{(l)}_v\right]_i = \max\left( \left[W^{(l)}_1 h_v^{(l-1)}\right]_i, \left\lbrace\left[W^{(l)}_2 h^{(l-1)}_u\right]_i\right\rbrace_{u \in N(v)}\right),
    \end{align}
    where we use the notation $[a]_i$ to denote the $i$-th coordinate of some vector $a$. Nonetheless, we may still analyze the associating edge sensitivity directly, i.e. via carefully inspecting the geometry of the max-pooling operation in high-dimensional spaces. However, it is straightforward to check that max-pooling causes high edge sensitivity, which is no smaller than that of summation pooling in the worse case, resulting in relatively large noise scales. Meanwhile, the "signal" brought by the aggregation does not scale with neighborhood size. Therefore it is intuitively clear that using the global sensitivity framework to privatize SAGE in this max-pooling form would lead to poor privacy-utility trade-offs. It is worth mentioning that chances are that adopting more elegant techniques like the smooth sensitivity paradigm \cite{nissim2007smooth} may allow meaningful privatization of the max-pooling aggregator, which is beyond the scope of the current paper and delegated to futrue explorations. 
    \subsubsection{On attentive pooling of GAT \cite{velickovic2018graph}}
    Next we consider the renowned GAT model \cite{velickovic2018graph} with updating rule:
    \begin{align}
        \widetilde{h}_v \leftarrow \sum_{u \in N(v)\cup \{v\}}\beta_{uv}W h_u,
    \end{align}
    with the attention coefficients defined as
    \begin{align}
        \beta_{uv}= \dfrac{e^{k_\phi(u, v)}}{\sum_{u^\prime \in N(v)\cup \{v\}}e^{k_\phi(u^\prime, v)}}.
    \end{align}
    under the \emph{attention kernel} $k_\phi$. In the original implementation of GAT, the authors used additive attention kernels. Later extensions use alternatives such as the multiplicative kernel in graph transformer architectures\cite{ying2021transformers}. The protocol of attentive aggregation is also a special case of \eqref{eqn: gnn}, which may be understood as an interpolation between mean-pooling (where all the attention coefficients are equal) and max-pooling (where one of the attention kernels being extremely large in value that dominates the rest). As a consequence, the noise scale required under the edge sensitivity calculation paradigm (i.e., theorem \ref{thm: non_sampling}) will be between that obtained by mean-pooling and max-pooling, depending on the \emph{range} of the attention kernel. In particular, if the attention kernel has an unbounded range, i.e. the entire real line. Then the resulting edge sensitivty is almost the same as the one obtained by max-pooling and is thus impractical. Hence, to reduce the noise scale required for privatization, we need to use \emph{bounded} attention via effectively controlling the output range of the attention kernel, i.e., via applying bounded range nonlinearities like Tanh. The analysis could be done in the same manner as that of GCN with extra hyperparameters controlling the upper and lower bounds of attention coefficients. We leave related developments to future works. 

    \section{Some further remarks}
    \subsection{Practical considerations in implementing \vesper} \label{sec: impl}
    According to proposition \ref{prop: gin_es} and \ref{prop: gcn_es}, precise tracking of privacy budgets under PMP requires computing the operator norm of each layer's weight matrix $\{\opnorm{W^{(l)}}\}_{1 \le l \le L}$ which is computationally demanding. In practice, we instead add a spectral normalization operation \cite{miyato2018spectral} to each layer's weight matrix so that we may approximately control all the operator norms throughout the training process to be around $1$.
    \footnote{Technically, most of the current spectral normalization algorithms do not offer strict control over spectral norms but are instead carried out using approximations like power iteration. We observe in our experiments that the approximation error brought by inexact normalization is pretty benign, i.e., the total privacy budget accounted using exact normalization and using a single power iteration differs in their absolute value by less than $0.1$.}\par\noindent
    
    \subsection{Complexity analysis of \vesper} \label{sec: complexity}
     The computational complexity of vesper is of the same order as that in a standard GRL pipeline with neighborhood sampling. The communication complexity of \vesper is dominated by the data volume of (forward) embedding and (backward) gradients that get transmitted during each VFL step, which are both of the order $O(BLdK)$, with $B$ being the batch size and $K$ being the number of bits required to represent a scalar number. Note that the communication complexity may be further optimized via quantization techniques \cite{suresh2017distributed} or asynchronous optimization tricks \cite{liu2022fedbcd}. We leave such explorations to future research.
    \section{Omitted algorithm descriptions}\label{sec: algo}
    In this section we present detailed descriptions of two algorithmic procedures, the first one is the truncated message passing algorithm for PMP-GCN, which for simplicity we present in a non-sampling fashion in algorithm \ref{alg: truncated_mp}. 
    \begin{algorithm}
        \caption{PMP-GCN with truncated message passing}
        \label{alg: truncated_mp}
        \begin{algorithmic}[1]
            \Require Graph $G = (V, E)$, input encodings $ \{h^{(0)}_v\}_{v \in V}$, number of message passing rounds $L$, minimum degree $\mind$, GNN parameter $\mathbf{W}$, MLP parameter $\{W^{(l)}_{tr}, b^{(l)}_{tr}\}_{1 \le l \le L}$
            \State Normalize each $h^{(0)}_v$ into unit $\ell_2$ norm.
            \For{$l \in \{1, \ldots, L\}$}
            \For{$v \in V$}
            \State Compute the linear update
            \begin{align}
                \widetilde{h}^{(l)}_v = \dfrac{W^{(l)} h^{(l-1)}_v}{d_v + 1} + \sum_{u \in N(v)}\frac{W^{(l)} h^{(l-1)}_u}{\sqrt{d_v + 1}\sqrt{d_u + 1}}.
            \end{align}
            \If{$d_v \ge \mind$}
            let $\widehat{h}^{(l)}_v = \widetilde{h}^{(l)}_v$.
            \Else
            \ let $\widehat{h}^{(l)}_v = W^{(l)}_{tr}h^{(l-1)}_v + b^{(l)}_{tr}$.
            \EndIf
            \State Add Gaussian noise and apply nonlinearity
            \begin{align}
                h^{(l)}_v = \sigma(\widetilde{h}^{(l)}_v + z_v),\ z_v \sim N(0, \theta^2I_d).
            \end{align} 
            \State Normalize $h^{(l)}_v = \frac{h^{(l)}_v}{\left\|h^{(l)}_v\right\|_2}$.
            \EndFor
            \EndFor
            \Return A list of all layers' embedding matrices $\mathbf{H}_L = (H^{(1)}, \ldots, H^{(L)})$, with $H^{(l)} = \{h^{(l)}_v\}_{v \in V}, 1 \le l \le L$.
        \end{algorithmic}
    \end{algorithm}
    The second one is a fully-detailed description of the training procedure of the \vesper framework, illustrated pictorially in figure \ref{fig: vesper}. We use different colors to differentiate (local) computations that are performed by different parties. Additionally, we abbreviate the forward computation of three algorithmic components of \vesper by Encode, PRE and Decode respectively. 
    \begin{algorithm}
        \caption{Algorithmic description of \vesper}
        \label{alg: vesper}
        \begin{algorithmic}[1]
            \Require Graph $G = (V, E)$, node features $X$, node label $Y$, batch size $B$, number of training steps $T$. Architectural specifics \{Encoder, PRE, Decoder\}, with the parameters of Encoder and PRE grouped together with notation $\mathbf{W}_A$ and the parameters of Decoder denoted as $\mathbf{W}_B$. Number of message passing layers $L$, max degree $D$.
            \State Initialize parameters $\mathbf{W}_A^{(0)}, \mathbf{W}_B^{(0)}$
            \For{$t = 1, \ldots, T$}
            \State Sample a random batch of root nodes $\mathcal{B}_t = \{v^t_1, \ldots, v^t_B\}$. 
            {\color{blue}
            \State \texttt{/*Computations by party A*/}
            \State Use neighborhood sampler as stated in algorithm \ref{alg: priv_gnn_sample} to obtain 
            \State the combined subgraph $G_{\mathcal{B}_t}^{(L)} = (V_{\mathcal{B}_t}^{(L)}, E_{\mathcal{B}_t}^{(L)})$.
            \State Encode node features
            \begin{align}
                h_v = \text{Encode}(X_v), v \in V_{\mathcal{B}_t}^{(L)}
            \end{align}
            \State Do message passing using the selected PRE mechanism
            \begin{align}
                \mathbf{H} = \text{PRE}(G_{\mathcal{B}_t}^{(L)}, H)
            \end{align}
            \State Pick the node embeddings to transmit $\mathbf{H}_{\mathcal{B}_t} = \{\mathbf{H}_v\}_{v \in \mathcal{B}_t}$ 
            \State and send to party B
            }
            {\color{fandango}
            \State \texttt{/*Computations by party B*/}
            \State Decode node embeddings and compute learning objective
            \begin{align}
                \mathcal{L} = \frac{1}{B}\sum_{v \in \mathcal{B}_t}\ell\left(y_v, \text{Decode}(\mathbf{H}_v)\right)
            \end{align}
            \State Compute gradients w.r.t. decoder $\dfrac{\partial \mathcal{L}}{\partial \mathbf{W}^{(t)}_B}$ and 
            \State update into $\mathbf{W}^{(t+1)}_B$ using selected optimizer.
            \State Send individual gradients w.r.t. node embedding collections
            \begin{align}
                \dfrac{\partial \mathcal{L}}{\partial \mathbf{H}_v}, v \in \mathcal{B}_t
            \end{align}
            \State back to party A.
            }
            {\color{blue}
            \State \texttt{/*Computations by party A*/}
            \State Party A compute gradients w.r.t. all its local parameters 
            \State (encoder and PRE)
            \begin{align}
                \dfrac{\partial \mathcal{L}}{\partial \mathbf{W}^{(t)}_A} = \frac{1}{B}\sum_{v \in \mathcal{B}_t}\dfrac{\partial \mathcal{L}}{\partial \mathbf{H}_{v}}\dfrac{\partial \mathbf{H}_{v}}{\partial \mathbf{W}^{(t)}_A}
            \end{align}
            \State and use the selected optimizer to update $\mathbf{W}^{(t)}_A$ into $\mathbf{W}^{(t+1)}_A$
            }
            \EndFor
            \Return The parameters at the final iteration $\mathbf{W}_A^{(T)}, \mathbf{W}_B^{(T)}$
        \end{algorithmic}
    \end{algorithm}

\end{document}